\documentclass[11pt]{article}

\newif\ifmarkup
\markupfalse

\newcommand{\markupadd}[1]{\ifmarkup
\textcolor{red}{#1}\else
#1\fi
}

\usepackage[english]{babel}
\usepackage{arxiv_sty}
\usepackage[letterpaper,margin=1in]{geometry}
\usepackage[utf8]{inputenc}
\usepackage{amsmath}
\usepackage{graphicx}
\usepackage{xspace}
\usepackage{caption}

\newcommand{\citep}{\cite}

\newif\ifcomments
\commentstrue

\ifcomments
\newcommand{\jd}[1]{{\color{purple}{\textbf{JD:} #1}}}
\newcommand{\pq}[1]{{\color{brown}{\textbf{PQ:} #1}}}
\newcommand{\nz}[1]{{\color{violet}{\textbf{Zn:} #1}}}
\else
\newcommand{\jd}[1]{}
\newcommand{\pq}[1]{}
\newcommand{\nz}[1]{}
\fi
\newcommand{\Ltwo}{\mathcal{L}_{2}^{\D,\sigma}}
\newcommand{\Lsur}{\mathcal{L}_{\mathrm{sur}}^{\D,\sigma}}
\newcommand{\bLsur}{\Bar{\mathcal{L}}_{\mathrm{sur}}^{\D,\sigma}}

\newcommand{\g}{\vec g}
\newcommand{\ball}{\mathcal{B}}

\newcommand{\cWstar}{\mathcal{W}^*}

\makeatletter

\newcommand{\CS}{Cauchy-Schwarz\xspace}

\usepackage{algorithmic}

\usepackage{algorithm}

\title{Robustly Learning a Single Neuron via Sharpness}

\author{
Puqian Wang\thanks{Supported in part by NSF Award CCF-2007757.} ‖\\
UW Madison\\
{\tt pwang333@wisc.edu}\\
\and 
Nikos Zarifis\thanks{Supported in part by NSF award 2023239,  NSF Medium Award CCF-2107079, and a DARPA Learning with Less Labels (LwLL) grant.} ‖\\
UW Madison\\
{\tt zarifis@wisc.edu}\\
	\and Ilias Diakonikolas\thanks{Supported by NSF Medium Award CCF-2107079,
		NSF Award CCF-1652862 (CAREER), a Sloan Research Fellowship, and
		a DARPA Learning with Less Labels (LwLL) grant.}\\
		UW Madison\\
		{\tt ilias@cs.wisc.edu}\\
\and Jelena Diakonikolas\thanks{Supported by NSF Award CCF-2007757 and by the U.\ S.\ Office of
Naval Research under award number N00014-22-1-2348.}\\
UW Madison\\
{\tt jelena@cs.wisc.edu}
}

\begin{document}
\maketitle
\def\thefootnote{‖}\footnotetext{Equal contribution.}
\def\thefootnote{*}
\begin{abstract}
 We study the problem of learning a single neuron with respect to the $L_2^2$-loss in the presence of adversarial label noise. 
We give an efficient algorithm that, 
for a broad family of activations including ReLUs,
approximates the optimal $L_2^2$-error within a constant factor.
Our algorithm applies under much milder distributional assumptions
compared to prior work.
The key ingredient enabling our results
is a novel connection to local error bounds 
from optimization theory.
\end{abstract}

\setcounter{page}{0}
\thispagestyle{empty}
\newpage

\section{Introduction} \label{sec:intro}

We study the following learning task: Given 
labeled examples $(\x,y) \in \R^d \times \R$ from an unknown distribution $\D$, output the best-fitting 
ReLU (or other nonlinear function) with respect to square loss. This is a fundamental problem in machine learning that has been extensively studied in a number of interrelated contexts over the past two decades, including learning GLMs and neural networks.  
More specifically, letting $\sigma:\R\mapsto\R$ denote a nonlinear activation, e.g., $\sigma(t) = \mathrm{ReLU}(t) = \max\{0, t \}$,  the (population) square loss of a vector $\vec w$ is defined as the $L_2^2$ loss of the hypothesis $\sigma(\vec w \cdot \vec x)$, i.e.,  
$\Ltwo (\vec w) \triangleq \Ex_{(\x, y) \sim \D} [(\sigma(\vec w \cdot \vec x) - y)^2]$. Our learning problem is then formally defined as follows.

\begin{problem}[Robustly Learning a Single Neuron] \label{def:agnostic-learning}
Fix $\eps>0, W > 0$, and a class of distributions $\mathcal{G}$ on $\R^d$. 
Let $\sigma : \R \mapsto \R$ be an activation and $\D$ a distribution on labeled examples
$(\x,y) \in \R^d \times \R$ such that its $\x$-marginal $\D_\x$ belongs to $\mathcal{G}$.
For some $C \geq 1$, a $C$-approximate proper learner is given 
$\eps$,  $W$, and i.i.d.\ samples from $\D$ and outputs $\hat{\w}\in \R^d$ such that 
with high probability it holds 
\( \Ltwo(\hat{\w}) \leq C \, \opt  +\eps \), where  
$\opt \triangleq \min_{\|\vec w\|_2 \leq W} \Ltwo (\vec w)$ is the minimum attainable square loss.  
We use $\cWstar \triangleq \argmin_{\norm{\w}_2\leq W} \Ltwo(\vec w)$ to denote 
the set of square loss minimizers. 
\end{problem}

Problem~\ref{def:agnostic-learning} does not make realizability assumptions 
on the distribution $\D$. The labels are allowed to be arbitrary and we are interested 
in the best-fit function with respect to the $L_2^2$ error. 
This corresponds to the (distribution-specific) agnostic PAC learning 
model~\cite{Haussler:92, KSS:94}. In this paper,  
we focus on developing {\em constant factor} approximate learners, 
corresponding to the case that $C$ is a universal constant greater than one. 

The special case of Problem~\ref{def:agnostic-learning} where the labels 
are consistent with a function in 
$\mathcal{H} = \{ \sigma(\vec w \cdot \x) :\; \|\vec w \|_2 \leq W\}$ 
was studied in early work~\cite{KalaiS09,kakade2011efficient}. 
These papers gave efficient methods that succeed for any distribution 
on the unit ball and any monotone Lipschitz activation\footnote{The results in these works 
can tolerate zero mean random noise, but do not apply to the adversarial noise setting.}. 
More recently,~\cite{YS20} showed that gradient descent on the nonconvex $L_2^2$ 
loss succeeds under a natural class of distributions (again in the \emph{realizable} case) 
but fails in general. In other related work, \cite{Sol17} analyzed the case of ReLUs 
in the {realizable} setting under the Gaussian distribution 
and showed that gradient descent efficiently achieves exact recovery.

The agnostic setting is computationally challenging. 
First, even for the case that the marginal distribution on the examples is Gaussian, 
there is strong evidence that any algorithm achieving error $\opt+\eps$ 
($C=1$ in Problem~\ref{def:agnostic-learning}) 
requires $d^{\poly(1/\eps)}$ time~\cite{GoelKK19, DKZ20, GGK20, DKPZ21, DKR23}. 
Second, even if we relax our goal to constant factor approximations, 
some distributional assumptions are required: 
known NP-hardness results rule out proper learners achieving {\em any} constant factor~\cite{Sima02, MR18}. More recent work~\cite{DKMR22} has shown 
that no polynomial time constant factor {\em improper} learner exists 
(under cryptographic assumptions), 
even if the distribution is bounded. These intractability results 
motivate the design of constant factor approximate learners --- corresponding 
to $C>1$ and $C = O(1)$ --- that succeed under as mild distributional assumptions as possible. 

Prior algorithmic work in the robust setting can be classified in two categories: 
A line of work~\cite{FCG20, DKTZ22, ATV22} analyzes gradient descent-based 
algorithms on the natural nonconvex $L_2^2$ objective (possibly with regularization). 
These works show that {\em under certain distributional assumptions} 
gradient descent avoids poor local minima and converges to a good solution. 
Specifically,~\cite{DKTZ22} established that gradient descent efficiently converges 
to a constant factor approximation for a family of well-behaved continuous distributions 
(including logconcave distributions). 
The second line of work~\cite{DGKKS20} proceeds by convexifying the problem, 
namely constructing a {\em convex surrogate} whose optimal solution 
gives a good solution to the initial nonconvex problem. 
This convex surrogate was analyzed by~\cite{DGKKS20} for the 
case of ReLUs who showed that it yields a constant factor 
approximation for logconcave distributions. 

The starting point of our investigation is the observation that 
{\em all previous algorithmic works for Problem~\ref{def:agnostic-learning} 
impose fairly stringent distributional assumptions.}
These works require all of the following properties 
from the marginal distribution on examples: 
(i) anti-concentration, (ii) concentration, and (iii) anti-anti-concentration. 
Assumption (i) posits that that 
{\em every} one-dimensional (or, in some cases, constant-dimensional) projection 
of the points should not put too much mass in any interval (or ``rectangle''). 
Property (ii) means that every one-dimensional projection should be strongly concentrated 
around its mean; specifically, prior work required at least exponential concentration. 
Finally, (iii) requires that the density of every low-dimensional projection 
 is bounded below by a positive constant.

While some concentration  appears necessary, 
prior work required sub-exponential concentration, 
which rules out the important case of heavy-tailed data. 
The anticoncentration assumption (i) from prior work rules out possibly 
lower-dimensional data, while the anti-anti-concentration 
rules out discrete distributions, which naturally occur in practice.

The preceding discussion raises the following question:
\begin{center}
{\em Under what distributional assumptions can we obtain\\ 
efficient {\em constant factor} learners for Problem~\ref{def:agnostic-learning}?}
\end{center}
In this paper, we give such an algorithm that succeeds 
under minimal distributional assumptions. 
Roughly speaking, our novel assumptions 
require anti-concentration {\em only} in the direction 
of the optimal solution (aka a margin assumption) 
and allow for heavy-tailed data. Moreover, 
by removing assumption (iii) altogether, we obtain the first positive 
results for structured discrete 
distributions (including, e.g., discrete Gaussians 
and the uniform distribution over the cube).
 
In addition to its generality, our algorithm is simple --- a mini-batch SGD ---
and achieves significantly better sample complexity 
for distributions covered in prior work.

\subsection{Overview of Results} \label{ssec:results-intro}

We provide a simplified version of our distributional assumptions 
followed by our main result for ReLU activations. 

\vspace{0.4cm}
\noindent {\bf Distributional Assumptions.} 
We make only the following two distributional assumptions.
\vspace{-0.2cm}
\begin{itemize}[leftmargin=*]
\item[] {\bf Margin-like Condition:} 
There exists $\w^* \in \cWstar$ and constants 
$\gamma, \lambda >0$ such that 
\begin{equation} \label{eqn:margin}
\EE_{\x\sim\D_\x}\left[ \x\x^\top \1\lp\{\wstar \cdot\x\geq \gamma \|\wstar\|_2 \rp\} \right] \succeq \lambda \vec I \;.
\end{equation}
\item[] {\bf Concentration:} There exists non-increasing $h: \R_+\to\R_+$ satisfying $h(r) = O(r^{-5})$ such that 
for any unit vector $\bu$ and any $r\geq 1$, it holds 
$\prpr_{\x\sim\D_\x}[|\bu\cdot\x|\geq r]\leq h(r)$. 
\end{itemize}

Before we state our algorithmic result, some comments are in order. 
Condition~\eqref{eqn:margin} is an anti-concentration condition, 
reminiscent of the classical margin condition for 
halfspaces.
In comparison with prior work, our condition requires anti-concentration
{\em only} in the direction of an optimal solution --- as opposed to every direction. 
Our second condition requires that every univariate projection exhibits some concentration.
Our concentration function $h$ can even be inverse polynomial, allowing
for heavy-tailed data. In contrast, prior work only considered sub-exponential tails.
As we will see, the function $h$ affects the sample complexity of our algorithm.

As we show in Appendix~\ref{sec:distri}, our distributional assumptions subsume all previous such assumptions
considered in the literature and additionally include a range of distributions
(including heavy-tailed and discrete distributions) not handled in prior work.

A simplified version of our main result for the
special case of ReLU activations is as follows 
(see Theorem~\ref{thm:l2-fast-rate-thm-m} for a detailed more general statement):
 
\begin{theorem}[Main Algorithmic Result, Informal]\label{thm:main-informal}
Let $W = O(1)$, $\mathcal{G}$ be a class of marginal 
distributions satisfying the above distributinal assumptions, 
and $\sigma$ be the ReLU activation.
There exists a sample-efficient and sample-linear time algorithm
that outputs a hypothesis $\hat{\w}$ such that, with high 
probability, $\Ltwo(\hat{\w}) = O(\opt) + \eps$. 
In particular, if the tail function $h$ is subexponential, 
namely $h(r) = e^{-\Omega(r)}$,
then the algorithm has sample complexity 
$n = \tilde{O}(d \, \polylog(1/\eps))$. 
For heavy-tailed distributions, namely for $h(r) = O(r^{-k})$ 
for some $k > 4$, the algorithm has sample complexity 
$n = \tilde{O}(d \, (1/\eps)^{2/(k-4)})$. 
The algorithm's runtime is always $O(n d)$.
\end{theorem}

Our algorithm is extremely simple: it amounts to mini-batch SGD on a natural
convex surrogate of the problem. As we will explain subsequently,
this convex surrogate has been studied before 
in closely related --- yet more restricted --- contexts.
Our main technical contribution lies in the analysis, which hinges on a new connection
to local error bounds from the theory of optimization. This connection is crucial for us 
in two ways: First, we leverage it to obtain the first 
constant-factor approximate learners 
under much weaker distributional assumptions. 
Second, even for distributions 
covered by prior work, the connection allows us to obtain significantly 
more efficient algorithms.

Finally, we note that our algorithmic result 
applies to a broad family of monotone activations 
(Definition~\ref{def:ub} and Assumption~\ref{assmpt:activation}), 
and can be adapted to handle non-monotone activations --- 
including GeLU~\cite{HG16} and Swish~\cite{RZL17} --- 
see Appendix~\ref{sec:non-monotone-activation-results}.

\subsection{Technical Contributions} \label{ssec:techniques}

The main algorithmic difficulty in solving Problem~\ref{def:agnostic-learning} is its non-convexity.
Indeed,  the $L_2^2$ loss is non-convex for nonlinear activations, even without noise. Of course, the presence of adversarial label noise only makes the problem even more challenging. 
At a high-level, our approach is to convexify the problem 
via an appropriate {\em convex surrogate} function 
(see, e.g.,~\cite{bartlett2006convexity}). 
In more detail, given a distribution $\D$ on labeled examples $(\x,y)$ 
and an activation $\sigma$, the surrogate $\Lsur(\w)$ is defined by
$$\Lsur(\w) = \mathbf{E}_{(\x,y)\sim\D}\bigg[\int_0^{\w\cdot\x} (\sigma(r)-y)\diff{r} \bigg] \;.$$

This function is not new. It was first defined in~\cite{AuerHW95} and subsequently
(implicitly) used in~\cite{KalaiS09, kakade2011efficient} for learning GLMs with zero mean noise.
More recently,~\cite{DGKKS20} used this convex surrogate for robustly learning ReLUs
under logconcave distributions. Roughly speaking, they showed that -- under the logconcavity
assumption -- a near-optimal solution to the (convex) optimization 
problem of minimizing $\Lsur(\w)$ 
yields a constant factor approximate learner 
for Problem~\ref{def:agnostic-learning} (for the special case of ReLU activations).

Very roughly speaking, our high-level approach is similar to that of~\cite{DGKKS20}. 
The main novelty of our contributions lies in two aspects: 
(1) The {\em generality} of the distributional assumptions under which
we obtain a constant-factor approximation, 
and (2) the sample and computational complexities 
of the associated algorithm.  
Specifically, our analysis yields a constant-factor approximate learner
under a vastly more general class of distributions\footnote{Recall that
without distributional assumptions obtaining {\em any} 
constant-factor approximate learner is NP-hard.} 
as compared to prior work, and extends to a much broader
family of activations beyond ReLUs. Moreover, even if restrict
ourselves to, e.g., logconcave distributions, 
the complexity of our algorithm is exponentially
smaller as a function of $\eps$ 
--- namely, $\polylog(1/\eps)$ as opposed to $\Omega(1/\eps^2)$.
For a more detailed comparison, see Appendix~\ref{ssec:related}.

The key technical ingredient enabling our results 
is the notion of {\em sharpness} (local error bound) 
from optimization theory, which we prove holds for our stochastic surrogate problem.
Before explaining how this comes up in our setting, we provide an overview 
from an optimization perspective.

\vspace{0.4cm}
\noindent {\bf Local Error Bounds and Sharpness.} 
Broadly speaking, given an optimization problem (P) and a ``residual'' function $r$ 
that is a measure of error of a candidate solution $\w$ to (P), 
an error bound certifies that a small residual translates into closeness 
between the candidate solution and the set of ``test'' (typically optimal) solutions $\mathcal{W}^*$ to (P). 
In particular, an error bound certifies an inequality of the form
\vspace{-0.2cm}
\begin{equation}\notag
    r(\w) \geq (\mu/\nu) \, \mathrm{dist}(\w, \cWstar)^{\nu}
\end{equation}
for some parameters $\mu, \nu >0$,  
where $\mathrm{dist}(\w, \cWstar) = \min_{\w^* \in \cWstar}\|\w - \w^*\|_2$ 
(see, e.g., the survey~\cite{pang1997error}). 
When this bound holds {\em only locally} in some neighborhood of $\cWstar,$ 
it is referred to as a \emph{local error bound}.

Local error bounds are well-studied within optimization theory, 
with the earliest result in this area being attributed to~\cite{hoffman2003approximate}, 
which provided local error bounds for systems of linear inequalities. 
The work of~\cite{hoffman2003approximate} was extended to many 
other optimization problems; see, e.g., Chapter 6 in \cite{facchinei2003finite} for an overview of classical results and \cite{bolte2017error,karimi2016linear,roulet2017sharpness,liu2022solving} and references therein for a more cotemporary overview.  
One of the most surprising early results in this area states that 
for  minimizing a convex function $f$, an inequality of the form
\begin{equation}\label{eq:sharpness-def}
    f(\w) - \min_{\mathbf{u}} f(\mathbf{u}) \geq (\mu/\nu) \, \mathrm{dist}(\w, \cWstar)^{\nu}
\end{equation}
holds generically whenever $f$ is a real analytic 
or subanalytic function~\citep{lojasiewicz1963propriete,lojasiewicz1993geometrie}. 
The main downside of this result is that the parameters $\mu, \nu$ 
are usually impossible to evaluate and, moreover, 
even when it is known that, e.g., $\nu = 2$, 
the parameter $\mu$ can be exponentially small in the dimension. 
Furthermore, local error bounds have primarily been studied in the context of \emph{deterministic} 
optimization problems, with results for stochastic problems being very rare~\cite{chen2005expected, liu2018fast}.

Perhaps the most surprising aspect of our results is that we show that the (stochastic) 
convex surrogate minimization problem not only satisfies a local error bound 
(a relaxation of \eqref{eq:sharpness-def} and a much weaker property than strong convexity; see Appendix~\ref{app:sharp}) with $\nu= 2$, 
but we are also able to characterize the parameter $\mu$ based on the assumptions 
about the activation function and the probability distribution over the data. 
More importantly, for standard activation functions such as ReLU, Swish, and GeLU 
and for broad classes of distributions (including heavy-tailed and discrete ones), 
we prove that $\mu$ is an \emph{absolute constant}. 
This is precisely what leads to the error and complexity results achieved in our work.

\vspace{0.4cm}
\noindent {\bf Robustly Learning a Single Neuron via Sharpness.} 
Our technical approach can be broken down into the following main ideas. 
As a surrogate for minimizing the square loss, 
we first consider the \emph{noise-free} convex surrogate, 
defined by 
$$\bLsur(\w; \w^*) = \EE_{\x\sim\D_\x}\left[\int_0^{\w\cdot\x} (\sigma(r)-\sigma(\wstar\cdot\x))\diff{r}\right],$$
where $\w^* \in \cWstar$ is a square-loss minimizer 
that satisfies our margin assumption. 
We keep this $\w^*$ fixed throughout the analysis and simply write $\bLsur(\w)$ instead of $\bLsur(\w; \w^*)$. 
Compared to the convex surrogate $\Lsur(\w)$ introduced earlier in the introduction, 
the noise-free convex surrogate $\bLsur(\w; \w^*)$ 
replaces the noisy labels $y$ with $\sigma(\w^* \cdot \x)$. 
Clearly, $\bLsur(\w; \w^*)$ is a function that cannot be directly optimized, 
as we lack the knowledge of $\w^*$. On the other hand, the noise-free surrogate 
relates more directly to the square loss minimization: 
we prove (Lemma~\ref{lem:sharpness-well-behaved}) 
that our distributional assumptions suffice 
for the noise-free surrogate to be sharp 
on a ball of radius $2\|\w^*\|_2$, $\ball(2\|\w^*\|_2)$; this structural result 
in turn leads to the conclusion that $\w^*$ is its unique minimizer. Hence, we can conclude that minimizing the noise-free surrogate $\bLsur(\w; \w^*)$ leads to minimizing the $L_2^2$ loss. Of course, we cannot directly minimize $\bLsur(\w; \w^*)$, as we do not know $\w^*.$ 

Had there been no adversarial label noise, 
we could stop at this conclusion, 
as there would be no difference between $\Lsur(\w)$ and $\bLsur(\w; \w^*)$ 
and we could minimize the $L_2^2$ error to any desired accuracy by minimizing $\Lsur(\w)$. 
This difference between $\Lsur(\w)$ and $\bLsur(\w; \w^*)$ is precisely 
what causes the $L_2^2$ error to only be brought down to $O(\opt) + \epsilon,$ 
where the constant in the big-Oh notation depends on the sharpness parameter $\mu.$ 
On the technical side, we prove (\Cref{cor:true-sharpness-m}) 
that $\Lsur(\w)$ is also sharp w.r.t.~the same $\w^*$ as $\bLsur(\w; \w^*)$ 
and with the sharpness parameter $\mu$ of the same order, 
but \emph{only on a nonconvex subset} of the ball $\ball(2\|\w^*\|_2)$, 
which excludes a neighborhood of $\w^*.$ This turns out to be sufficient to relate minimizing $\Lsur(\w)$ 
to minimizing the $L_2^2$ loss (\Cref{cor:landscape-assuptions-m}). 

What we argued so far is sufficient for 
ensuring that  minimizing the surrogate loss $\Lsur(\w)$ 
leads to the claimed bound on the $L_2^2$ loss. 
However, it is not sufficient for obtaining the claimed sample and computational complexities, 
and there are additional technical hurdles that can only be handled 
using the specific structural properties of our resulting optimization problem. 
In particular, using solely smoothness and sharpness of the objective 
(even if the sharpness held on the entire region 
over which we are optimizing), 
would only lead to complexities scaling with $\frac{1}{\epsilon},$ 
using standard results from stochastic convex optimization. 
However, the complexity that we get is \emph{exponentially better}, 
scaling with $\mathrm{polylog}(\frac{1}{\epsilon})$. 
This is enabled by the refined variance bound 
for the stochastic gradient estimate 
(see \Cref{app:cor:empirical-grad-variance}), 
which, unlike in standard stochastic optimization settings 
(where we get a fixed upper bound), 
scales with ${\opt} + \|\w - \w^*\|_2^2$.\footnote{Similar variance bound assumptions have been made in the more recent literature on stochastic optimization; see, e.g., Assumption 4.3(c) in~\cite{bottou2018optimization}. We note, however, that our guarantees hold with high probability (compared to the more common expectation guarantees) and that the bulk of of our technical contribution lies in proving that such a variance bound holds, rather than in analyzing SGD under such an assumpton.} 
This property enables us to construct high-accuracy gradient estimates 
using mini-batching, which further leads to the improved linear 
rates within the (nonconvex) region where the surrogate loss is sharp. 
To complete the argument, we further show that the excluded region 
on which the sharpness does not hold does not negatively 
impact the overall complexity, as within it  
the target approximation guarantee for the $L_2^2$ loss holds.

\subsection{Notation} \label{sec:prelims}

For $n \in \Z_+$, we denote by $[n]$ the set $\{1, \ldots, n\}$.  
We use lowercase boldface letters for vectors
and uppercase bold letters for matrices. 
For $\bx \in \R^d$ and $i \in [d]$, $\bx_i$ denotes the
$i^\mathrm{th}$ coordinate of $\bx$, and $\|\bx\|_2 \eqdef (\littlesum_{i=1}^d {\bx_i}^2)^{1/2}$ denotes the
$\ell_2$-norm of $\bx$. 
We  use $\bx \cdot \by $ for the standard inner product of $\bx, \by \in \R^d$
and $ \theta(\bx, \by)$ 
for the angle between $\bx, \by$.  
We use $\1_\event$ for the
characteristic function of the set/event $\event$, 
i.e., $\1_\event(\x)= 1$ if $\x\in \event$ 
and $\1_\event(\x)= 0$ if $\x\notin \event$. 
We denote by $\ball(r) = \{\bu : \norm{\bu}_2\leq r\}$ 
the $\ell_2$-ball of radius $r$.
We use the standard asymptotic notation 
$\wt{O}(\cdot)$ and $\wt{\Omega}(\cdot)$
to omit polylogarithmic factors in the argument.  
We write $E \gtrsim F$ for two nonnegative
expressions $E$ and $F$ to denote that there exists some universal constant $c > 0$
(independent of the variables or parameters on which $E$ and $F$ depend) such that $E \geq c \, F$.
We use $\Ex_{X\sim \D}[X]$ for the expectation of random variable $X$ 
according to the distribution $\D$ and $\pr{\mathcal{E}}$ 
for the probability of event $\mathcal{E}$. For simplicity
of exposition, we may omit the distribution 
when it is clear from the context.  For $(\x,y)$
distributed according to $\D$, 
we denote by $\D_\x$  the marginal distribution of $\x$.

\section{Landscape of Noise-Free Surrogate}\label{sec:sharpness}

We start by defining the class of activations 
and the distributional assumptions under which our results apply. 
We then establish our first structural result, 
showing that these conditions suffice for sharpness 
of the noise-free surrogate.

\subsection{Activations and Distributional Assumptions}

The main assumptions used throughout this paper to prove sharpness results are summarized below.

\begin{definition}[Monotonic Unbounded Activations,~\cite{DKTZ22}] \label{def:ub}
Let $\sigma: \R \mapsto \R$ be non-decreasing  
and let $\alpha, \beta > 0$. We say that $\sigma$ is (monotonic) $(\alpha, \beta)$-unbounded if 
(i) $\sigma$ is $\alpha$-Lipschitz; and (ii) $\sigma'(t)\geq \beta$ for all $t>0$. 
\end{definition}

The above class contains a range of popular activations, 
including the ReLU (which is $(1, 1)$-unbounded), 
and the Leaky ReLU with parameter $0\leq \lambda\leq \frac{1}{2}$, i.e., 
$\sigma(t) = \max\{\lambda t, (1-\lambda)t\}$ (which is is $(1-\lambda, 1-\lambda)$-unbounded).

Our results apply for the following class of activations.

\begin{assumption}[Controlled Activation]\label{assmpt:activation}
The activation function $\sigma:\R \to \R$ is $(\alpha, \beta)$-unbounded, 
for some positive parameters $\alpha\geq 1, \beta\in(0,1)$, 
and it holds that $\sigma(0)=0$.
\end{assumption}

The assumption on the activation is important both for the convergence analysis 
of our algorithm and for proving the sharpness property of the surrogate loss.

We can now state our distributional assumptions.

\begin{assumption}[Margin]\label{assmpt:margin}
There exists $\w^* \in \cWstar$ and parameters $\gamma, \lambda \in(0,1]$ such that 
$$\EE_{\x\sim\D_\x}\lp[\x\x^\top \1\lp\{ \wstar \cdot \x \geq \gamma \|\wstar\|_2 \rp\}\rp]\succeq \lambda \vec I. $$
\end{assumption}

We note that in order to obtain a constant-factor approximate learner,
the parameters $\gamma$ and $\lambda$ in \Cref{assmpt:margin} should be 
dimension-independent constants. 

\begin{assumption}[Concentration]\label{assmpt:concentration}
There exists a non-increasing $h: \R_+\to\R_+$ satisfying 
$h(r)\leq  Br^{-(4+\rho)}$ for some parameters $B\geq 1$ and $1\geq \rho>0$, 
such that for any $\bu\in \ball(1)$ and any $r\geq 1$, 
it holds $\prpr[|\bu\cdot\x|\geq r]\leq h(r)$.
\end{assumption}

The concentration property enables us to control the moments of $|\bu\cdot\x|$, 
playing an important role when we bound the variance 
of the gradient of the empirical surrogate loss. 

\vspace{-0.5cm}

\

\subsection{Key Assumptions Suffice for Sharpness}
We now prove that Assumptions~\ref{assmpt:activation}--\ref{assmpt:concentration} 
suffice to guarantee that the noise-free surrogate loss is sharp. 
We provide a proof sketch under the simplifying assumption that $\|\wstar\|_2=1$. 
The full proof can be found in \Cref{app:proof-of-lem:sharpness-well-behaved}.

\begin{lemma}\label{lem:sharpness-well-behaved}
Suppose that Assumptions~\ref{assmpt:activation}--\ref{assmpt:concentration} hold.
Then the noise-free surrogate loss $\bLsur$ is $\Omega(\lambda^2\gamma\beta\rho/B)$-sharp in the ball $\ball(2\|\w^*\|_2)$, i.e., $\forall \w \in \ball(2\|\w^*\|_2)$,
\begin{equation*}
        \nabla \bLsur(\w)\cdot(\w-\wstar)\gtrsim \lambda^2\gamma\beta\rho/B\norm{\w - \wstar}_2^2\; .
    \end{equation*}
\end{lemma}
\begin{proof}[Proof Sketch of \Cref{lem:sharpness-well-behaved}]
Observe that $\nabla\bLsur(\w) = \E{(\sigma(\w\cdot\x) - \sigma(\wstar\cdot\x))\x}$. Using the fact that $\sigma$ is non-decreasing, it holds that $\nabla \bLsur(\w)\cdot(\w-\wstar)=\EE_{\x\sim\D_\x}[|\sigma(\w\cdot\x)-\sigma(\wstar\cdot\x)|
|\w\cdot\x-\wstar\cdot\x|]$. Denote $\mathcal E_m=\{\wstar\cdot\x\geq \gamma\}$. Using the fact that every term inside the expectation is nonnegative, we can further bound $\nabla \bLsur(\w)\cdot(\w-\wstar)$ from below by
\begin{align*}
    \nabla \bLsur(\w)\cdot(\w-\wstar) 
\geq \EE_{\x\sim\D_\x}[|\sigma(\w\cdot\x)-\sigma(\wstar\cdot\x)|
|\w\cdot\x-\wstar\cdot\x|\1_{\mathcal E_m}(\x)] 
\;.
\end{align*}
Since $\sigma$ is $(\alpha,\beta)$-unbounded, we have that $\sigma'(t)\geq \beta$ for all $t\in(0,\infty)$. By the mean value theorem, we can show that for $t_2\geq t_1\geq 0$, we have $|\sigma(t_1)-\sigma(t_2)| \geq \beta|t_1-t_2|$. Additionally, if $t_1\geq 0$ and $t_2\leq 0$, then $|\sigma(t_1)-\sigma(t_2)| \geq \beta t_1$.
Therefore, by combining the above, and denoting the event $\{\x: \w\cdot\x\leq 0,\,\wstar\cdot\x\geq \gamma\}$ as  $\event_0$, we get
\begin{equation}\label{eq:sharp-1}
\begin{aligned}
    \nabla \bLsur(\w)\cdot(\w-\wstar)
& \geq \beta\EE_{\x\sim\D_\x}[(\w\cdot\x-\wstar\cdot\x)^2\1{\{\w\cdot\x>0,\,\mathcal E_m(\x)\}}]  \\
    & + \beta\underbrace{\E{|\wstar\cdot\x||\w\cdot\x-\wstar\cdot\x|\1_{\event_0}(\x)}}_{(Q)}.
\end{aligned}
\end{equation}
We show the term $(Q)$ can be bounded below by a quantity proportional to: $\E{(\w\cdot\x-\wstar\cdot\x)^2\1_{\mathcal E_0}(\x)}$.
To this end, we establish the following claim. 
\begin{claim}\label{clm:expectation-bound}
   For $r_0\geq 1$, define the event $\event_1 = \event_1(r_0)=\{\x: -2r_0<\vec w\cdot \x\leq0,\mathcal E_m(\x)\}$. It holds 
    $(Q)\geq (\gamma/(3r_0))\E{(\w\cdot\x - \wstar\cdot\x)^2 \1_{\event_1}(\x)}\;.$ 
\end{claim}
\begin{proof}[Proof of \Cref{clm:expectation-bound}]
Since $\event_1\subseteq\event_0$, it holds that
    $(Q)\geq \E{|\wstar\cdot\x||\w\cdot\x-\wstar\cdot\x|\1_{\event_1}(\x)}$.
Restricting $\x$ on the event $\event_1$, it holds that $|\w\cdot \x|\leq 2(r_0/\gamma)|\wstar\cdot\x|$. Thus,\begin{equation*}
        \w^*\cdot \x - \w\cdot \x = |\w^*\cdot \x| + |\w \cdot \x| \leq (1+ 2r_0/\gamma)|\w^* \cdot \x|.
    \end{equation*}
By \Cref{assmpt:margin} we have that $\gamma\in(0,1]$, therefore we get that $|\w^* \cdot \x|\geq \gamma/(\gamma + 2r_0)\geq \gamma/(3r_0)$, since $r_0\geq 1$.
Taking the expectation of $|\wstar\cdot \x||\w\cdot\x - \wstar\cdot\x|$ with $\x$ restricted on event $\event_1$, we obtain
    \begin{align*}
 (Q) \geq \E{|\wstar\cdot \x||\w\cdot\x - \wstar\cdot\x| \1_{\event_1}(\x)} \geq  \gamma/(3r_0)\E{(\w\cdot\x - \wstar\cdot\x)^2 \1_{\event_1}(\x)}\;,
    \end{align*}
as desired. \end{proof}

Combining \Cref{eq:sharp-1} and \Cref{clm:expectation-bound}, we get that 
\begin{align}\label{eq:final-bound}
    \nabla \bLsur(\w)\cdot(\w-\wstar)  \geq\frac{\beta\gamma}{3r_0}\EE_{\x\sim\D_\x}[(\w\cdot\x-\wstar\cdot\x)^2 \1{\{\w\cdot\x > -2r_0,\,\mathcal E_m(\x)\}}],
\end{align}
where in the last inequality we used the fact that $1 \geq \gamma/(3r_0)$ (since  $\gamma\in(0,1]$ and $r_0\geq 1$). To complete the proof, we need to show that, for an appropriate choice of $r_0$, the probability of the event $\{\x: \w\cdot\x>-2r_0,\,\wstar\cdot\x>\gamma\}$ 
is close to the probability of the event $\{\x:\wstar 
 \cdot\x\geq \gamma\}$. Given such a statement, the lemma follows from \Cref{assmpt:margin}.
Formally, we show the following claim.

\begin{claim}\label{clm:approx-error}
Let $r_0 \geq 1$
such that $h(r_0)\leq \lambda^2\rho/(20B)$.
Then, for all $\w\in\ball(2\norm{\w^*}_2)$, we have that 
\begin{align*}
    \EE_{\x\sim\D_\x}[(\w\cdot\x-\wstar\cdot\x)^2\1{\{\w\cdot\x > -2r_0 ,\,\mathcal E_m(\x)\}}] \geq (\lambda/2) \|\wstar-\vec w\|_2^2\;.
\end{align*}
\end{claim}
\noindent Since $h(r) \leq B/r^{4+\rho}$ and $h(r)$ is decreasing, 
such an $r_0$ exists and we can always take $r_0\geq 1$.

Combining \Cref{eq:final-bound} and \Cref{clm:approx-error}, we get:
    \begin{equation*}
        \nabla \bLsur(\w)\cdot(\w - \w^*)\gtrsim \frac{\gamma\lambda\beta}{r_0}\|\vec w-\wstar\|_2^2.
    \end{equation*}
 To complete the proof of \Cref{lem:sharpness-well-behaved}, it remains to choose $r_0$ appropriately. By \Cref{clm:approx-error}, we need to select $r_0$ to be sufficiently large so that $h(r_0)\leq \lambda^2\rho/(20B)$. By \Cref{assmpt:concentration}, we have that $h(r)\leq B/r^{4 + \rho}$. Thus, we can choose $r_0=5B/(\lambda\rho)$, by our assumptions. \end{proof}

\section{Efficient Constant-Factor Approximation }\label{sec:fast-rate-for-sharp-objective}

We now outline our main technical approach, including the algorithm, its analysis, connections between the $L_2^2$ loss and the two (noisy and noise-free) surrogates, and the role of sharpness. 
For space constraints, this section contains simplified proofs and proof sketches, while the full technical details are deferred to \Cref{app:section3}.

\subsection{The Landscape of Surrogate Loss}
We start this section by showing that the landscape of surrogate loss connects with the error of the true loss.

\begin{theorem}\label{cor:landscape-assuptions-m}
Let $\D$ be a distribution supported on $\R^d\times \R$ and 
let $\sigma:\R\mapsto\R$ be an $(\alpha,\beta)$-unbounded activation.
Fix $\wstar \in \mathcal W^*$ and suppose that $\D_\x$ satisfies \Cref{assmpt:margin,assmpt:concentration} with respect to $\wstar$. 
Furthermore, let $C>0$ be a sufficiently small absolute constant 
and let $\bar{\mu}=C \lambda^2\gamma \beta \rho/B$. 
Then, for any $\eps>0$ and $\hat{\w}\in \ball(2\|\w^*\|_2)$, 
so that $\Lsur(\hat{\w})-\inf_{\w\in\ball(2\|\w^*\|_2)}\Lsur(\w) \leq \eps$, it holds
$\Ltwo(\hat{\w})\leq O((\alpha B/(\rho\Bar{\mu}))^2 )(\Ltwo(\wstar) + \alpha\eps).$
\end{theorem}
\begin{proof}
For this proof, we assume for ease of presentation that $\Exx{\vec x \vec x^\top}\	\preceq  \vec I $ and $B,\rho,\alpha=1$.
Denote $\mathcal{K}$ as the set of $\hat{\w}$ such that $\hat{\w}\in \ball(2\|\w^*\|_2)$ and $\Lsur(\hat{\w})-\inf_{\w\in\ball(2\|\w^*\|_2)}\Lsur(\w) \leq \eps$.

Next observe that the set of minimizers of the loss $\Lsur$ 
inside the ball $\ball(2\|\w^*\|_2)$ is convex. 
Furthermore, the set $\ball(2\|\w^*\|_2)$ is compact. 
Thus, for any point $\w'\in \ball(2\|\w^*\|_2)$ that minimizes $\Lsur$ 
it will either hold that $\|\nabla \Lsur(\w')\|_2=0$ 
or $\w'\in \partial \ball(2\|\w^*\|_2)$. 
Let $\mathcal{W}_\mathrm{sur}^*$ be the set of minimizers of $\Lsur$.

We first show that if there exists a minimizer 
$\w'\in \mathcal{W}_\mathrm{sur}^*$ such that 
$\w'\in \partial \ball(2\|\w^*\|_2)$, 
then any point $\w$ inside the set $\ball(2\|\w^*\|_2)$ 
gets error proportional to $\Ltwo(\wstar)$. 
Observe for such point $\w'$, by the necessary condition of optimality, 
\begin{equation}\label{eq:contradiction-m}
    \nabla \Lsur(\w')\cdot(\w'-\w)\leq 0\;,
\end{equation}
for any $\w \in \ball(2\|\w^*\|_2)$.
Using \Cref{cor:true-sharpness-m}, we get that either 
$\nabla \Lsur(\w')\cdot(\w'-\wstar)\geq (\Bar{\mu}/2)\|\w'-\wstar\|_2^2$
or $\w'\in \{\w:\|\w-\wstar\|_2^2\leq (20/\bar{\mu}^2)\Ltwo(\w^*)\}$. 
But \Cref{eq:contradiction-m} contradicts with 
$\nabla \Lsur(\w')\cdot(\w'-\wstar)\geq (\Bar{\mu}/2)\|\w'-\wstar\|_2^2>0$, 
since $\w'\in \partial \ball(2\norm{\w^*}_2)$, $\norm{\w'}_2 = 2\norm{\w^*}_2$; 
hence $\w'\neq \w^*$. So it must be the case that 
$\w'\in \{\w:\|\w-\wstar\|_2^2\leq (20/\bar{\mu}^2)\Ltwo(\w^*)\}$. 
Again, we have that $\w'\in \partial \ball(2\|\w^*\|_2)$, 
therefore $\|\w'-\wstar\|_2\geq \|\wstar\|_2$. 
Hence, $(20/\bar{\mu}^2)\Ltwo(\w^*)\geq \|\wstar\|_2^2$. 
Therefore, for any $\w\in \ball(2\|\w^*\|_2)$, we have
\begin{align*}
\Ltwo(\w) = \Ey{(\sigma(\w\cdot\x) - y)^2} \leq 2\Ltwo(\wstar) + \norm{\w - \w^*}_2^2= O(1/\Bar{\mu}^2)\Ltwo(\wstar)\;,
\end{align*}
where  we used the fact that $\E{\x\x^\top} \preceq \vec I$ 
and that $\sigma$ is $1$-Lipschitz. Since the inequality above holds 
for any $\w\in\ball(2\norm{\w^*}_2)$, it will also be true 
for $\hat{\w}\in\mathcal{K}\subseteq \ball(2\norm{\w^*}_2)$.
It remains to consider the case where the minimizers 
$\mathcal{W}_\mathrm{sur}^*$ are strictly inside 
the $\ball(2\|\w^*\|_2)$. Note that $\Lsur(\w)$ is $1$-smooth.
Therefore, for any $\hat{\w}\in \mathcal K$
it holds $\|\nabla\Lsur(\hat{\w})\|_2^2\leq 2\eps$. 
By \Cref{cor:true-sharpness-m} (stated and proved below), we get that 
either $\|\hat{\w}-\wstar\|_2^2\leq (1/\Bar{\mu}^2) \Ltwo(\wstar)$ 
or that $\sqrt{2\eps}\geq (\bar{\mu}/2)\|\hat{\w}-\wstar\|_2$. 
Therefore, we obtain that 
$\|\hat{\w}-\wstar\|_2^2\leq (1/\Bar{\mu}^2) (\Ltwo(\wstar)+\eps)$.
\end{proof}

The proof of \Cref{cor:landscape-assuptions-m} required the following proposition 
which shows that if the current vector $\vec w$ is sufficiently 
far away from the true vector $\wstar$, 
then the gradient of the surrogate loss has a large component 
in the direction of $\vec w-\wstar$; in other words, the surrogate loss is sharp.

\begin{proposition}
\label{cor:true-sharpness-m}
Let $\D$ be a distribution supported on $\R^d\times \R$ and let $\sigma:\R\mapsto\R$ be an $(\alpha,\beta)$-unbounded activation.
    Suppose that $\D_\x$ satisfies \Cref{assmpt:margin,assmpt:concentration} and let $C>0$ be a sufficiently small absolute constant and let $\bar{\mu}=C \lambda^2\gamma \beta \rho/B$. Fix $\wstar \in \mathcal W^*$ and let $S=\ball(2\|\w^*\|_2)-\{\w:\|\w-\wstar\|_2^2\leq (20B/(\rho\bar{\mu}^2))\Ltwo(\w^*)\}$. Then, the surrogate loss $\Lsur$ is $(\bar{\mu}/2)$-sharp in  $S$, i.e.,
\begin{equation*}
        \nabla \Lsur(\w)\cdot(\w-\wstar)\geq (\bar{\mu}/2)\norm{\w - \wstar}_2^2,\;\; \forall \w \in S.
    \end{equation*}
\end{proposition}
\begin{proof}
For this proof, we 
assume for ease of presentation that $\Exx{\vec x \vec x^\top}\	\preceq  \vec I $ and $\kappa, B,\rho,\alpha=1$.
 We show that $\nabla \Lsur(\w)\cdot (\w-\wstar)$ is bounded away from zero.
    We decompose the gradient into two parts, i.e., $\nabla \Lsur(\w)=(\nabla \Lsur(\w)-\nabla \Lsur(\wstar))+\nabla \Lsur(\wstar)$.
  First, we bound $\nabla \Lsur(\wstar)$ in the direction $\w-\wstar$, which yields
    \begin{align*}
\nabla \Lsur(\wstar)\cdot(\w-\wstar)
\geq- 
\sqrt{\Ltwo(\wstar)}\|\w-\wstar\|_2 \;,   
\end{align*}
where we used the \CS inequality and that $\EE_{\x\sim\D_\x}[\x\x^\top]\preceq \vec I$. It remains to bound the remaining term. Note that $(\nabla \Lsur(\w)-\nabla \Lsur(\wstar))=\nabla \bLsur(\w)$.
Using the fact that $\bLsur(\w)$ is $\bar{\mu}$-sharp for any $\vec w\in S$ from \Cref{lem:sharpness-well-behaved}, it holds that 
    $\nabla\bLsur(\w)\cdot(\w-\wstar)\geq \bar{\mu}\|\w-\wstar\|_2^2\;.$
    Combining everything together, we get the claimed result. \end{proof}

\subsection{Fast Rates
for $L_2^2$ Loss Minimization}\label{sec:fast-rates}

In this section, we proceed to show that when the surrogate loss is sharp, applying batch Stochastic Gradient Descent (SGD) on the empirical surrogate loss obtains a $C$-approximate parameter $\hat{\w}$ of the $L_2^2$ loss in linear time. To be specific, consider the following iteration update
\begin{equation}\label{alg:sgd}
    \w^{(t+1)} = \argmin_{\w\in\ball(W)} \big\{\w\cdot\g^{(t)} + (1/(2\eta))\norm{\w - \w^{(t)}}^2_2\big\},
\end{equation}
where $\eta$ is the step size and $\g^{(t)}$ is the empirical gradient of the surrogate loss, i.e.,
$
    \g^{(t)} = \frac{1}{N}\sum_{j=1}^N (\sigma(\w^{(t)}\cdot\x(j))-y(j))\x(j).
$
The algorithm is summarized in \Cref{alg:gd-3}. 

\begin{algorithm}[tb]
   \caption{Stochastic Gradient Descent on Surrogate Loss}
   \label{alg:gd-3}
\begin{algorithmic}
   \STATE {\bfseries Input:} Iterations: $T$, sample access from $\D$, batch size $N$, step size $\eta$, bound $M$. Initialize: $\vec w^{(0)} \gets \vec 0$.
\FOR{$t=1$ {\bfseries to} $T$}
\STATE Draw $N$ samples $\{(\x(j), y(j))\}_{j=1}^N\sim\D$.
   \STATE For each $j\in[N]$, $y(j)\gets\sgn(y(j))\min(|y(j)|,M)$. \STATE $\vec g^{(t)} \gets\frac{1}{N}\sum_{j=1}^N (\sigma(\w^{(t)}\cdot\x(j))-y(j))\x(j).$ \STATE $\vec w^{(t+1)} \gets \vec w^{(t)}-\eta \vec{g}^{(t)}$.
\ENDFOR
   \STATE {\bfseries Output:} The weight vector $\w^{(T)}$.
\end{algorithmic}
\end{algorithm}

We define the helper functions $H_2$ and $H_4$ as follows:
\begin{gather*}
    H_2(r) \triangleq \max_{\bu\in\ball(1)} \E{(\bu\cdot\x)^2\1\{|\bu\cdot\x|\geq r\}},\\
    H_4(r) \triangleq \max_{\bu\in\ball(1)} \E{(\bu\cdot\x)^4\1\{|\bu\cdot\x|\geq r\}}.
\end{gather*}

Now we state our main theorem.

\begin{theorem}[Main Algorithmic Result]\label{thm:l2-fast-rate-thm-m}
    Fix $\eps, W>0$ and suppose \Cref{assmpt:activation,assmpt:margin,assmpt:concentration} hold.
Let $\mu:=\mu(\lambda,\gamma,\beta,\rho,B)$ be a sufficiently small constant multiple of $\lambda^2\gamma\beta\rho/B$, and let $M = \alpha W H_2^{-1}(\eps/(4\alpha^2 W^2))$.
Further, choose parameter $r_\eps$ large enough so that $H_4(r_\eps)$ is a sufficiently small constant multiple of $\eps$.
Then after
       $ T = \wt{\Theta}\left((B^2\alpha^2/(\rho^2\mu^2))\log\lp(W/ \eps\rp)\right)$
iterations with batch size $N = \Omega(dT(r_\eps^2 + \alpha^2M^2))$, \Cref{alg:gd-3} converges to a point $\w^{(T)}$ such that 
    $\Ltwo(\vec w^{(T)}) = O\lp((B^2\alpha^2/(\rho^2 \mu^2))\rp)\opt +\eps\;,$ 
    with probability at least $2/3$. 
    
\end{theorem}

As shown in \Cref{cor:landscape-assuptions-m}, when we find a vector $\hat{\w}$ 
that minimizes the surrogate loss, then this $\hat{\w}$ is itself a $C$-approximate 
solution of \Cref{def:agnostic-learning}. However, minimizing the surrogate loss 
can be expensive in sample and computational complexity. 
\Cref{cor:true-sharpness-m} says that we can achieve strong-convexity-like rates, 
as long as we are far away from a minimizer of the $L_2^2$ loss. 
Roughly speaking, we show that at each iteration $t$, 
it holds $\|\w^{(t+1)} - \w^*\|_2^2\leq C\|\w^{(t)} - \w^*\|_2^2 + \opt$, 
where $0<C<1$ is some constant depending on the parameters 
$\alpha, \beta, \mu$, $\rho$, and $B$. 
Then $\|\w^{(t)} - \w^*\|_2$ contracts fast 
as long as $\|\w^{(t)} - \w^*\|_2^2> (1/(1-C))\opt$. When this condition fails, 
we have converged to a point that achieves $O(\opt)$ $L_2^2$ error.

The following lemma states that we can truncate 
the labels $y$ to $y' \leq M$, where $M$ is a parameter depending on $\D_\x$. 
The proof can be found in \Cref{app:full-version-sec-4.2}.

\begin{lemma}\label{lem:y-bounded-by-M-m}
Let  $M=\alpha W H_2^{-1}(\eps/(4\alpha^2 W^2))$ and $y' = \sgn(y)\min(|y|, M)$.  
Then we have that $\Ey{(\sigma(\wstar\cdot\x) - y')^2} = \opt+\eps$\;.    
\end{lemma}

\Cref{lem:y-bounded-by-M-m} allows us to assume that 
$|y|\leq M$.

\begin{proof}[Proof Sketch of \Cref{thm:l2-fast-rate-thm-m}]
For this sketch, we will assume for ease of notation 
that $B,\rho,\alpha=1$ and that $\Exx{\vec x \vec x^\top}\	\preceq  \vec I   $.
The blueprint of the proof is to show that \Cref{alg:gd-3} minimizes $\|\w - \w^*\|_2$ efficiently, in terms of both the sample complexity and the iteration complexity. To be specific, we show that at each iteration, $\|\w^{(t+1)} - \w^*\|_2^2 \leq (1 - C)\|\w^{(t)} - \w^*\|_2^2 + (\mathrm{small\; error})$, where $0<C<1$. The key technique is to exploit the sharpness property of the surrogate loss, which we have already proved in \Cref{cor:true-sharpness-m}. 

     To this aim, we study the difference of $\|\w^{(t+1)}-\wstar\|_2^2$ and $\|\w^{(t)}-\wstar\|_2^2$. 
     We remind the reader that for convenience of notation, 
     we denote the empirical gradients as the following
$
         \g^{(t)} = \frac{1}{N}\sum_{j=1}^N (\sigma(\w^{(t)} \cdot \x(j)) - y(j))\x(j),
         \g^* = \frac{1}{N}\sum_{j=1}^N (\sigma(\w^* \cdot \x(j)) - y(j))\x(j).$
     Moreover, we denote the noise-free empirical gradient by $\Bar{\g}^{(t)}$, i.e.,
    $
         \Bar{\g}^{(t)} = \g^{(t)} - \g^* $.
     Plugging in the iteration scheme $\w^{(t+1)} = \w^{(t)} - \eta \g^{(t)}$ 
     while expanding the squared norm, we get
\begin{align*}
        \norm{\w^{(t+1)}-\wstar}^2_2
        & =\underbrace{\norm{\w^{(t)}-\wstar}^2_2 - 2\eta\nabla \Lsur(\w\tth)\cdot(\w^{(t)}-\wstar)}_{Q_1}\\
        &\quad \underbrace{-2\eta(\g^{(t)}-\nabla \Lsur(\w\tth))\cdot(\w^{(t)}-\wstar)+\eta^2\norm{\g^{(t)}}^2_2}_{Q_2}\;.
    \end{align*}
Observe that we decomposed the right hand side into two parts, 
the true contribution of the gradient $(Q_1)$ and the estimation error $(Q_2)$.  
In order to utilize the sharpness property of surrogate loss 
at the point $\w^{(t)}$, the conditions
    \begin{gather}
\w^{(t)}\in\ball(2\|\w^*\|_2)\,\text{ and }\nonumber\\
        \,\w^{(t)}\in\{\w:\|\w^{(t)} - \w^*\|_2^2\geq 20/\Bar{\mu}^2\opt\}\label{eq:condition-w^t-in-ball(2||w*||_2)-m}
    \end{gather}
    need to be satisfied. For the first condition, recall that 
    we initialized $\vec w^{(0)}=\vec 0$; hence, \Cref{eq:condition-w^t-in-ball(2||w*||_2)-m} is valid for $t=0$. By induction, it suffices to show that 
    assuming $\w^{(t)}\in\ball(2\|\w^*\|_2)$ holds, 
    we have $\|\w^{(t+1)} - \w^*\|_2\leq (1-C)\|\w^{(t)} - \w^*\|_2$ for some constant $0<C<1$. Thus, we assume temporarily that \Cref{eq:condition-w^t-in-ball(2||w*||_2)-m} 
    is true at iteration $t$, and we will show in the remainder of the proof 
    that $\|\w^{(t+1)} - \w^*\|_2\leq  (1-C)\|\w^{(t)} - \w^*\|_2$ 
    until we arrived at some final iteration $T$. Then, by induction, 
    the first part of \Cref{eq:condition-w^t-in-ball(2||w*||_2)-m} 
    is satisfied at each step $t\leq T$. For the second condition, 
    note that if it is violated at some iteration $T$, 
    then $\|\w^{(T)} - \w^*\|_2^2= O(\opt)$ 
    implying that this would be the solution we are looking for 
    and the algorithm could be terminated at $T$. 
    Therefore, whenever $\|\w^{(t)}-\w^*\|_2^2$ is far away from $\opt$, 
    the prerequisites of \Cref{cor:true-sharpness-m} are satisfied 
    and  $\Lsur$ is sharp.

    For the first term $(Q_1)$, using  that $\nabla \Lsur(\w^{(t)})$ 
    is $\mu$-sharp by \Cref{cor:true-sharpness-m},
we immediately get a sufficient decrease at each iteration, i.e., 
    $\|\w^{(t+1)} - \w^*\|_2^2\leq (1 - C)\|\w^{(t)} - \w^*\|_2^2$. 
    Namely, applying \Cref{cor:true-sharpness-m}, we get
    \begin{align*}
        (Q_1) = \norm{\w^{(t)}-\wstar}^2_2 - 2\eta\nabla\Lsur(\w^{(t)})\cdot(\w^{(t)}-\wstar) \leq (1 - 2\eta \mu)\norm{\w^{(t)} - \w^*}_2^2\;,
    \end{align*}
   where $\mu = C\lambda^2\gamma\beta$ for some sufficiently small constant $C$.

    Now it suffices to show that $(Q_2)$ can be bounded above by $C'\|\w^{(t)} - \w^*\|_2^2$, where $C'$ is a parameter depending on $\eta$ and $\mu$ that can be made comparatively small. Formally, we show the following claim.
    \begin{claim}\label{claim:bound-Q2}
      Suppose $\eta\leq 1$. Fix $r_\epsilon\geq 1$ such that $H_4(r_\epsilon)$ 
      is a sufficiently small constant multiple of $\eps$. 
      Choosing $N$ to be a sufficiently large constant multiple of $(d/\delta)(r_\eps^2 + M^2)$, then we have with probability at least $1-\delta$
        \begin{equation*}
            (Q_2) \leq ((3/2)\eta \mu + 8\eta^2)\norm{\w^{(t)} - \wstar}_2^2 + ({8\eta}/{\mu})(\opt + \epsilon)\;.
        \end{equation*}
    \end{claim}
        
    \begin{proof}
        Observe that by the inequality $\bx\cdot\by\leq (\mu/2)\norm{\bx}_2^2 + (1/(2\mu))\norm{\by}_2^2$ applied to the inner product $(\g^{(t)}-\nabla \Lsur(\w\tth))\cdot(\w^{(t)}-\wstar)$, we get
\begin{align}
            (Q_2)\leq \frac{\eta}{\mu}\norm{\g^{(t)}-\nabla \Lsur(\w\tth)}_2^2 + \eta\mu\norm{\w^{(t)} - \wstar}_2^2 + 2\eta^2\norm{\bar{\g}^{(t)}}^2_2 + 2\eta^2\norm{\g^*}^2_2 \;,\nonumber  
        \end{align}
        where $\mu$ is the sharpness parameter and we used the definition that $\Bar{\g}\tth = \g\tth - \g^*$ in the first inequality. 

        Note that 
 $\|\g^{(t)}-\nabla \Lsur(\w\tth)\|_2^2\leq
            2\|\bar{\g}^{(t)} - \nabla\bLsur(\w^{(t)})\|_2^2+2\|\g^{*} - \nabla\Lsur(\wstar)\|_2^2$,
        since we have $\bLsur(\w\tth) = \Lsur(\w\tth) - \Lsur(\w^*)$. Thus, it holds
        \begin{align}\label{eq:main-thm-claim-I2-m}
        (Q_2)&\leq \eta\mu\norm{\w^{(t)} - \wstar}_2^2              +2\eta^2\norm{\bar{\g}^{(t)}}^2_2 + 2\eta^2\norm{\g^*}^2_2 \nonumber\\
        &\quad +(2\eta/\mu)\left(\norm{\bar{\g}^{(t)} - \nabla\bLsur(\w^{(t)})}_2^2 + \norm{\g^{*} - \nabla\Lsur(\wstar)}_2^2\right)  \;.
        \end{align}  
        Furthermore, using standard concentration tools, 
        it can be shown that when $N\geq C d(r_\eps^2+M^2)/\delta$
        where $C$ is a sufficiently large absolute constant, 
        with probability at least $1-\delta$, it holds       
        \begin{align*}
            \|\bar{\g}^{(t)} - \nabla\bLsur(\w^{(t)})\|_2^2&
            \leq (\mu^2/4) \|\w^{(t)} - \w^*\|_2^2, \\ \|\bar{\g}^{(t)}\|_2^2 &\leq 4\|\w^{(t)} - \wstar\|_2^2,
        \end{align*}
        and 
       $
        \norm{\g^*-\nabla\Lsur(\wstar)}_2^2\leq\opt +\epsilon $, $   \norm{\g^*}_2^2\leq 2\opt + \eps$ (see \Cref{app:lem:empirical-grad-approx-population} and \Cref{app:cor:bound-norm-g*-gt} for details).
        It remains to plug these bounds back into \Cref{eq:main-thm-claim-I2-m}.
\end{proof}

Combining the upper bounds on $(Q_1)$ and $(Q_2)$ and choosing $\eta = {\mu}/{32}$, we have:
    \begin{equation}\label{eq:final-eq}
       \norm{\w^{(t+1)} - \wstar}_2^2
         \leq \;\lp(1 - \mu^2/128\rp)\norm{\w^{(t)} - \w^*}_2^2+ (1/4)\big(\opt + \eps\big)\;. 
    \end{equation}
    When $\|\w^{(t)} - \w^*\|_2^2\geq (64/\mu^2)(\opt + \epsilon)$,
in other words when $\w^{(t)}$ is still away from the minimizer $\w^*$, it further holds with probability $1-\delta$:
\begin{equation}\label{eq:||w^(t+1)-w*|| <= (128/mu^2 + 1/alpha^2)opt-when-t<T-m}
       \norm{\w^{(t+1)}-\wstar}^2_2 \leq (1 - \mu^2/256)\norm{\w^{(t)} - \wstar}_2^2,
   \end{equation}
 which proves the sufficient decrease of $\|\w^{(t)} - \w^*\|_2^2$ that we proposed at the beginning. 

Let $T$ be the first iteration such that $\w^{(T)}$ satisfies $\|\w^{(T)} - \w^*\|_2^2\leq (64/\mu^2)(\opt + \epsilon)$. Recall that we need \Cref{eq:condition-w^t-in-ball(2||w*||_2)-m} for every $t\leq T$ to be satisfied to implement sharpness. The first condition is satisfied naturally for $\|\w^{(t+1)}-\wstar\|^2_2\leq \norm{\wstar}_2^2$ as a consequence of \Cref{eq:||w^(t+1)-w*|| <= (128/mu^2 + 1/alpha^2)opt-when-t<T-m} (recall that $\w^{(0)} = 0$). For the second condition, when $t+1\leq T$, we have
$    \norm{\w^{(t+1)} - \w^*}_2^2\geq ({64}/{\mu^2})(\opt + \epsilon)$,
hence the second condition also holds.

When $t\leq T$, the contraction of $\|\w^{(t)} - \w^*\|_2^2$ indicates a linear convergence of SGD. Since $\w^{(0)}=0$, $\|\w^*\|_2\leq W$, it holds $\|\w^{(t)} - \w^*\|_2^2\leq (1 - \mu^2/256)^t\|\w^{(0)} - \w^*\|_2^2\leq \exp(-t\mu^2/256) W^2$. Thus, to generate  $\w^{(T)}$ such that $\|\w^{(T)} - \w^*\|_2^2\leq (64/\mu^2)(\opt + \epsilon)$, it suffices to run \Cref{alg:gd-3} for
$   T= \wt{\Theta}((1/\mu^2)\log\lp(W/\eps\rp)\big)$
iterations. Recall that at each step $t$ the contraction $\|\w^{(t+1)}-\wstar\|^2_2 \leq (1 - \mu^2/256)\|\w^{(t)} - \wstar\|_2^2$ holds with probability $1 - \delta$, thus the union bound implies $\|\w^{(T)} - \w^*\|_2^2\leq (64/\mu^2)(\opt + \epsilon)$ holds with probability $1 - T\delta$.
   Moreover, as $ \Ltwo(\vec w^{(T)}) \lesssim \|\vec w^{(T)}-\wstar\|_2^2,$ if $\|\w^{(T)} - \w^*\|_2^2\leq (64/\mu^2)(\opt + \epsilon)$, then 
$\Ltwo(\vec w^{(T)})  = O(1/\mu^2)(\opt+ \eps).$
   Letting $\delta = 1/(3T)$ completes the proof.
\end{proof}

\vspace{-0.3cm} \section{Conclusion}
\vspace{-0.1cm} We provided an efficient constant-factor approximate learner 
for the problem of agnostically learning a single neuron over structured classes of distributions. Notably, our algorithmic result applies under much milder distributional assumptions as compared to prior work.
Our results are obtained by leveraging a sharpness property 
(a local error bound) from optimization theory that we prove 
holds for the considered problems. This property is crucial both to establishing a constant factor approximation and to obtaining  improved sample complexity and runtime.  An interesting direction for future work is to explore whether sharpness can be leveraged to obtain positive results 
for other related learning problems.
\vspace{-0.25cm} \paragraph{Acknowledgement. }
JD thanks Alexandre 
D’Aspremont and Jérôme Bolte for a useful 
discussion on local error bounds.

\newpage

\bibliographystyle{alpha}
\bibliography{mydb}
\appendix
\newpage

\section*{Appendix} \label{sec:app}

\paragraph{Organization. }
The appendix is organized as follows: 
In \Cref{app:sharp}, we provide some remarks on the sharpness property we have been using throughout the paper. In \Cref{ssec:related}, we provide additional detailed comparison 
with prior work. In \Cref{app:section3} and \Cref{app:section4}, we present 
the full contents of \Cref{sec:sharpness} and \Cref{sec:fast-rate-for-sharp-objective} respectively, providing supplementary lemmas and completing the omitted proofs in the main body. \Cref{sec:distri} shows that there are many natural distributions 
satisfying \Cref{assmpt:margin} and \Cref{assmpt:concentration}. 
Finally, in \Cref{sec:non-monotone-activation-results}, we show that our results 
extend to certain non-monotonic distributions, 
including GeLUs~\cite{HG16} and Swish~\cite{RZL17}.

\paragraph{Additional Notation. }
Some additional notation  we use here is listed below.
Given a distribution $\D$ on $\R^d\times\R$, we use $\{(\x(j),y(j))\}_{j=1}^N$ to denote $N$ i.i.d.\ samples from $\D$.
We slightly abuse the notation and denote by $\vec e_i$ the $i^{\mathrm{th}}$ standard basis vector in $\R^d$.  
The notation $[\cdot]_+$ is used for the positive part of the argument, i.e.,  $[\cdot]_+ = \max\{\cdot, 0\}.$ For a vector $\x=(\x_1,\cdots,\x_n),$ $[\cdot]_+$ is applied element-wise: $[\x]_+ := ([\x_1]_+,\cdots,[\x_n]_+).$ 
For nonnegative expressions $E, F$ 
we write $E \gg F$ to denote $E \geq C \, F$, where
$C>0$ is a \emph{sufficiently large} universal constant 
(independent of the parameters of $E$ and $F$). 
The notation $\ll$ is defined similarly.

\section{Remarks about Sharpness} \label{app:sharp}

We recall the formal definition of sharpness, already mentioned
in the introduction.

\begin{definition}[Sharpness]\label{def:sharpness-formal}
Given a function $f: \mathcal{C}\mapsto \R$ 
where $\mathcal{C}\subseteq \R^d$, 
suppose the set of its minimizers $\mathcal{Z^*} = \argmin_{\bz\in\mathcal{C}}f(\bz)$ is closed and not empty. Let $f^* = \min_{\bz\in\mathcal{C}} f(\bz)$. 
We say that $f$ is $\mu$-sharp, for some $\mu>0$, 
if the following inequality holds:
\begin{equation*}
f(\bz) - f^*\geq \frac{\mu}{2} \mathrm{dist}(\bz, \mathcal{Z}^*)^2,\,\forall \bz\in\R^d,
\end{equation*}
where $\mathrm{dist}(\bz, \mathcal{Z}^*) = 
\min_{\bz^*\in\mathcal{Z}^*} \norm{\bz - \bz^*}_2$.
\end{definition}

\begin{remark} \label{rem:sharp}
{\em We will slightly abuse the name of sharpness 
to refer to sharpness-like properties. For example, if a function satisfies
\begin{equation}\label{eq:sharpness-like}
    \nabla f(\bz)\cdot(\bz - \bz^*)\geq \mu\norm{\bz - \bz^*}_2^2,
\end{equation}
for some $\bz^*\in\mathcal{Z}^*$, then we say $f$ is $\mu$-sharp. 
This is due to the fact that when $f$ is a convex function, 
it holds $f(\bz) - f^*\leq \nabla f(\bz)\cdot(\bz - \bz^*)$, 
hence \Cref{def:sharpness-formal} implies \Cref{eq:sharpness-like}. 
Thus, \Cref{eq:sharpness-like} can be viewed as a milder property of sharpness.}
\end{remark}

Compared to strong convexity, sharpness is a milder condition. 
Indeed, for any $\mu$-strongly-convex function $f$, 
if $\bz^*\in\argmin_{\bz\in\R^d} f(\bz)$ then 
$f(\bz) - f^* \geq \nabla f(\bz^*)\cdot(\bz - \bz^*) 
+ \frac{\mu}{2}\norm{\bz - \bz^*}_2^2 \geq \frac{\mu}{2}\norm{\bz - \bz^*}_2^2$; therefore, $f$ is $\mu$-sharp. However, the opposite does not hold in general. 
For example, consider $f:\R\mapsto\R$ defined by $f(z) = z^2$ 
if $z\geq 0$ and $f(z) = 0$ otherwise, whose set of minimizers on $\R$ 
is $\mathcal{Z}^* = (-\infty, 0]$ and $f^* = 0$. 
Thus, if $z\geq0$, then $f(z) - f^*\geq \mathrm{dist}(z,\mathcal{Z}^*)^2 = z^2$ 
and if $z<0$, we have $f(z) - f^* = 0 = \mathrm{dist}(z,\mathcal{Z}^*)^2$. 
Therefore, $f$ is $2$-sharp but it is not strongly convex.

\section{Additional Comparison to Prior Work}\label{ssec:related}

Here we provide additional technical comparison to prior work that did not appear in the main body, due to space limitations.

\paragraph{Comparison with \cite{FCG20}.} The work of \cite{FCG20} studies the problem of learning ReLU (and other nonlinear) activations and shows that gradient descent on the $L_2^2$ loss converges to a point achieving error $K\sqrt{\opt}$. The parameter $K$ depends on the maximum norm of the points $\x$, and can depend on the dimension $d$. Specifically, even for the basic case that the marginal distribution on examples is the standard normal distribution, the parameter $K$ scales (polynomially) with $d$. That is, \cite{FCG20} does not provide constant factor approximate learners in this setting.

\paragraph{Comparison with \cite{DGKKS20}.} The work of \cite{DGKKS20} studies the problem of learning ReLU activations using the same surrogate loss we consider in this work. Our work differs from \cite{DGKKS20} in two key aspects. The first aspect concerns the generality and strengh of results; the second aspect concerns the techniques.

In terms of the results themselves, the algorithm given in 
\cite{DGKKS20} is restricted to the case of ReLUs (while we handle a broader family of activations). More importantly,
the distributional assumptions of \cite{DGKKS20} are much stricter than ours,
--- focusing on logconcave distributions --- 
whereas we handle broader classes of distributions, 
including heavy tailed and discrete distributions, 
not covered by any prior work (see also \Cref{sec:distri}). 
Informally, what allows us to handle broader classes of distributions 
is our focus on proving the sharpness property (as opposed to strong convexity), which is a much milder property. 
Further, we show that it suffices for this property to hold 
only in a small region (ball of radius $2\|\w^*\|_2$) 
and for the (impossible to evaluate) noise-free surrogate loss. Another remark is that \cite{DGKKS20} assume that the (corrupted) labels 
are bounded, not fully capturing the agnostic setting. By contrast, our analysis can handle unbounded labels, 
i.e., we do not make further assumptions about the noise. 
Finally, even if we restrict our focus to the class of logconcave 
distributions, our algorithm has sample complexity scaling with 
$\polylog(1/\eps)$, as opposed to $1/\eps^2$ in \cite{DGKKS20}.

The second and more important difference lies in the techniques 
that are used in each work.
\cite{DGKKS20} optimizes the surrogate loss directly and 
shows that finding a point with a small gradient of the surrogate loss 
leads to the small $L_2^2$ error. More specifically, the requirement in 
\cite{DGKKS20} is that the gradient is sufficiently small so that the 
optimality gap of the surrogate loss is of the order $\epsilon$. This statement is similar to the result we show in \Cref{cor:landscape-assuptions-m}. 
Crucially, while we utilize the gradients of the surrogate loss 
in the algorithm and in the analysis, we never impose 
a requirement that the optimality gap of the surrogate loss 
is of the order $\epsilon.$ Instead, we show that as long as 
the gradient is larger than order-$\sqrt{\opt+\eps}$, 
sharpness holds and linear convergence rate applies. 
On the other hand, when the gradient is of the order $\sqrt{\opt+\eps}$ 
or smaller, we argue that the candidate solution that 
the algorithm maintains is already an $(O(\opt) + \epsilon)$-approximate solution in terms of the $L_2^2$ error. This approach further 
enables us to be agnostic in the value of $\opt.$ 
Notably, if $\eps \ll \opt$ and we were to require 
that the algorithm finds a solution with either the gradient 
of the order $\sqrt{\eps}$ or the optimality gap $\eps,$ 
we would need to optimize the surrogate loss within a region where the 
sharpness does not necessarily hold. Without sharpness, 
only sublinear rates of convergence apply, and the number 
of iterations increases to order-$\frac{1}{\eps}$. 
Thus, leveraging the structural properties that 
we prove in this work is crucial to obtaining 
the exponential improvements in sample 
and computational complexities.   

Finally, \cite{DGKKS20} requires the 
surrogate loss to be strongly convex to 
connect the small gradient condition with 
the small $L_2^2$ error. This makes the argument 
rather straightforward, compared with what is used in our work. 
For the sake of discussion, assume that $\Lsur$ is $1$-
strongly convex and the distribution is 
isotropic. Furthermore, denote by $\wstar$ 
the minimizer of the $L_2^2$ loss and by $\w'$ the 
minimizer of $\Lsur$. The property that 
$\Lsur$ is strongly convex implies that 
$\|\nabla \Lsur(\wstar)-\nabla \Lsur(\vec w')\|_2^2\geq \|\wstar-\vec w'\|_2^2$; 
furthermore, it can be shown that $\|\nabla 
\Lsur(\wstar)\|_2^2\leq {\Ltwo(\wstar)} $. 
Therefore, because $\nabla \Lsur(\vec 
w')=\vec 0$, it immediately follows that 
$\Ltwo(\w')\lesssim \Ltwo(\wstar)$. 
Our work leverages a much weaker property than strong convexity --- sharpness --- as summarized in \Cref{cor:true-sharpness-m}. This weaker property turns out to be sufficient to ensure that the noise  cannot make the gradient field guide 
us far away from the optimal 
solution.

\paragraph{Comparison with \cite{DKTZ22}.} The work of \cite{DKTZ22} studies the problem of ReLU (and other unbounded activations) regression with agnostic noise. They show that for a class of well-behaved distributions (see Definition~\ref{def:well-behaved}) gradient descent on the $L_2^2$ loss converges to a point achieving $O(\opt)+\eps$ error. 
Moreover, the sample and computational complexities of their algorithm are similar to those achieved in our work (for the class of well-behaved distributions). On the other hand, the distributional assumptions used in \cite{DKTZ22} are quite strong. Specifically, the ``well-behaved'' assumption requires that the marginal distribution 
have sub-exponential concentration and anti-anti-concentration 
in every lower dimensional subspace; that is, the probability density function 
is lower bounded by a positive constant at every point. The latter assumption 
does not allow for several discrete distributions, like discrete Gaussians or 
uniform on the cube, that is handled in our work. Moreover, our work can 
additionally handle distributions with much weaker concentration properties.

\paragraph{Comparison with \cite{karmakar2022provable}.}
In a weaker noise model, the work of \cite{karmakar2022provable} considered a similar-looking --- though crucially different  --- condition for robust ReLU regression, 
namely that:
\begin{equation}\label{eq:margin-assmpt-KM22}
    \lambda_{\mathrm{min}} \lp(\E{\x\x^\top\1\{\w^*\cdot\x\geq 2\theta^*\}}\rp) = \lambda_1>0,
\end{equation}
where $\theta^*$ is the largest possible absolute value of the noise; 
in other words, $\theta^* = \sup_{(\x,y)\sim\D} |y-\sigma(\w^*\cdot\x)|$. It is worth noting that \Cref{eq:margin-assmpt-KM22} 
cannot be easily satisfied, as the noise in the agnostic model is not bounded. 
But even if the noise was bounded, this condition would give slack for a small number of distributions. For instance in the uniform on the hypercube, if $\theta^*> 1/2$, then the minimum eigenvalue is zero. Furthermore, the algorithm in that work converges to a point that achieves $O(\theta^*)$ error, instead of $O(\opt)$ error.
In contrast, we make no assumptions about the boundedness of the noise, and obtain near-optimal error in more general settings.

\paragraph{Additional Related Work.}
As mentioned in the introduction, the convex surrogate leveraged in our 
work was first defined in~\cite{AuerHW95} and then implicitly used 
in~\cite{KalaiS09, kakade2011efficient} for learning GLMs. In addition to 
these and the aforementioned works, it is worth mentioning that the same 
convex surrogate has been useful in the context of learning linear 
separators from limited information~\cite{DeDFS14} and in related game-
theoretic settings~\cite{DeDS17, DiakonikolasPPS22}.

\section{Full Version of \Cref{sec:sharpness}}\label{app:section3}

\paragraph{Discussion about the Parameters in \Cref{assmpt:activation,assmpt:concentration,assmpt:margin}}
If an activation $\sigma$ is $(\alpha',\beta')$-bounded, 
then it is also $(\alpha,\beta)$-bounded 
for $\alpha\geq \alpha'$ and $\beta\leq \beta'$. 
This justifies the convention $\alpha\geq 1$ and $\beta\leq 1$ in \Cref{assmpt:activation}.
If $\sigma(0)\neq 0$, we can generate new labels $y'$ by subtracting 
$\sigma(0)$ from $y$ and consider the activation $\sigma_0(t) = \sigma(t) - \sigma(0)$.
Similar reasoning justifies the conventions 
$\lambda,\gamma \in(0,1]$ in \Cref{assmpt:margin} 
and $B\geq 1$, $\rho\leq 1$ in \Cref{assmpt:concentration}.

\subsection{Proof of \Cref{lem:sharpness-well-behaved}}\label{app:proof-of-lem:sharpness-well-behaved}

For convenience, we restate the lemma followed by its detailed proof.

\begin{lemma}
Suppose that Assumptions~\ref{assmpt:activation}--\ref{assmpt:concentration} hold.
Then the noise-free surrogate loss $\bLsur$ is $\Omega(\lambda^2\gamma\beta\rho/B)$-sharp in the ball $\ball(2\|\w^*\|_2)$, i.e., $\forall \w \in \ball(2\|\w^*\|_2)$ we have
\begin{equation*}
        \nabla \bLsur(\w)\cdot(\w-\wstar)\gtrsim \lambda^2\gamma\beta\rho/B\norm{\w - \wstar}_2^2\; .
    \end{equation*}
\end{lemma}

\begin{proof}
By definition, we can write $\nabla\bLsur(\w) = \E{(\sigma(\w\cdot\x) - \sigma(\wstar\cdot\x))\x}$. Therefore, the inner product $\nabla \bLsur(\w)\cdot(\w-\wstar)$ can be written as 
\begin{align*}
    &\quad \nabla \bLsur(\w)\cdot(\w-\wstar) \\
    &= \E{(\sigma(\w\cdot\x)-\sigma(\wstar\cdot\x))(\w\cdot\x-\wstar\cdot\x)} \\
    &= \E{|\sigma(\w\cdot\x)-\sigma(\wstar\cdot\x)|
|\w\cdot\x-\wstar\cdot\x|} 
\\&\geq \E{|\sigma(\w\cdot\x)-\sigma(\wstar\cdot\x)|
|\w\cdot\x-\wstar\cdot\x|\1{\{\wstar\cdot\x\geq \gamma\norm{\wstar}_2\}}} 
\;,
\end{align*}
where the second equality is due to the non-decreasing property of $\sigma$, and the inequality is due to the fact that every term inside the expectation is nonnegative. Since $\sigma$ is $(\alpha,\beta)$-unbounded, we have that $\sigma'(t)\geq \beta$ for all $t\in[0,\infty)$. By the mean value theorem, for $t_2\geq t_1\geq 0$, we have $\sigma(t_1)-\sigma(t_2) = \sigma'(\xi)(t_1-t_2)$ for some $\xi\in(t_1,t_2)$. Thus, we obtain that $|\sigma(t_1)-\sigma(t_2)|\geq \beta|t_1-t_2|$. Additionally, if $t_1\geq 0$ and $t_2\leq 0$, then $|\sigma(t_1)-\sigma(t_2)|=|\sigma(t_1)-\sigma(0)|+|\sigma(0)-\sigma(t_2)|\geq |\sigma(t_1)-\sigma(0)| \geq \beta t_1$.
Therefore, by combining the above, we get
\begin{equation}\label{app:eq:sharp-1}
\begin{aligned}
    \nabla \bLsur(\w)\cdot(\w-\wstar) 
&\geq \beta\E{(\w\cdot\x-\wstar\cdot\x)^2\1{\{\w\cdot\x>0,\,\wstar\cdot\x>\gamma\norm{\wstar}_2\}}} \\
    &\quad + \beta\underbrace{\E{|\wstar\cdot\x||\w\cdot\x-\wstar\cdot\x|\1{\{\w\cdot\x\leq 0,\,\wstar\cdot\x\geq \gamma\norm{\wstar}_2\}}}}_{(Q)}.
\end{aligned}
\end{equation}

Denote $\event_0 = \{\x: \w\cdot\x\leq 0,\,\wstar\cdot\x\geq \gamma\norm{\wstar}_2\}$. We show that the term $(Q)$ can be bounded below by a quantity that is proportional to $\E{(\w\cdot\x-\wstar\cdot\x)^2\1{\{\w\cdot\x\leq 0,\,\wstar\cdot\x>\gamma\norm{\wstar}_2\}}}$.
To this end, we establish the following claim. 
\begin{claim}\label{app:clm:expectation-bound}
   For $r_0\geq 1$, define the event $\event_1 = \event_1(r_0)=\{\x: -2r_0\|\wstar\|_2<\vec w\cdot \x\leq0,\wstar\cdot \x\geq\gamma\|\wstar\|_2\}$. It holds 
    $(Q)\geq (\gamma/(3r_0))\E{(\w\cdot\x - \wstar\cdot\x)^2 \1_{\event_1}(\x)}\;.$ 
\end{claim}
\begin{proof}[Proof of \Cref{app:clm:expectation-bound}]
Since $\event_1\subseteq\event_0$, it holds that
    $(Q)\geq \E{|\wstar\cdot\x||\w\cdot\x-\wstar\cdot\x|\1_{\event_1}(\x)}$.
Restricting $\x$ on the event $\event_1$, it holds that $|\w\cdot \x|\leq 2(r_0/\gamma)|\wstar\cdot\x|$. Therefore, we get
    \begin{equation*}
        \w^*\cdot \x - \w\cdot \x = |\w^*\cdot \x| + |\w \cdot \x| \leq (1+ 2r_0/\gamma)|\w^* \cdot \x|.
    \end{equation*}
By \Cref{assmpt:margin} we have that $\gamma\in(0,1]$, therefore we get that $|\w^* \cdot \x|\geq \gamma/(\gamma + 2r_0)\geq \gamma/(3r_0)$, since $r_0\geq 1$.
Taking the expectation of $|\wstar\cdot \x||\w\cdot\x - \wstar\cdot\x|$ with $\x$ restricted on event $\event_1$, we obtain
    \begin{align*}
 (Q) &\geq \E{|\wstar\cdot \x||\w\cdot\x - \wstar\cdot\x| \1_{\event_1}(\x)}\geq  \gamma/(3r_0)\E{(\w\cdot\x - \wstar\cdot\x)^2 \1_{\event_1}(\x)}\;,
    \end{align*}
as desired. \end{proof}

Combining \Cref{app:eq:sharp-1} and \Cref{app:clm:expectation-bound}, we get that 
\begin{align}\label{app:eq:final-bound}
    &\quad \nabla \bLsur(\w)\cdot(\w-\wstar) \nonumber\\
    &\geq \beta\E{(\w\cdot\x - \w^*\cdot\x)^2\1\{\w\cdot\x>0, \w^*\cdot\x \geq \gamma\norm{\w^*}_2\}} \nonumber\\
    &\quad + \frac{\beta\gamma}{3r_0}\E{(\w\cdot\x - \w^*\cdot\x)^2\1{\{-2r_0\norm{\w^*}_2 < \w\cdot\x\leq 0,\,\wstar\cdot\x\geq \gamma\norm{\wstar}_2\}}} \nonumber\\
    &\geq \frac{\beta\gamma}{3r_0}\E{(\w\cdot\x-\wstar\cdot\x)^2\1{\{\w\cdot\x > -2r_0\|\wstar\|_2,\,\wstar\cdot\x>\gamma\norm{\wstar}_2\}}},
\end{align}
where in the last inequality we used the fact that $1 \geq \gamma/(3r_0)$ (since  $\gamma\in(0,1]$ and $r_0\geq 1$). To complete the proof, we need to show that, for an appropriate choice of $r_0$, the probability of the event $\{\x: \w\cdot\x>-2r_0\|\wstar\|_2,\,\wstar\cdot\x>\gamma\norm{\wstar}_2\}$ 
is close to the probability of the event $\{\x:\wstar 
 \cdot\x\geq \gamma\|\wstar\|_2\}$. Given such a statement, the lemma follows from \Cref{assmpt:margin}.
 
Formally, we show the following claim.

\begin{claim}\label{app:clm:approx-error}
Let $r_0 \geq 1$ such that $h(r_0)\leq \lambda^2\rho/(20B)$.
Then, for all $\w\in\ball(2\norm{\w^*}_2)$, we have  
that 
$$
\E{(\w\cdot\x-\wstar\cdot\x)^2\1{\{\w\cdot\x > -2r_0\|\wstar\|_2 ,\,\wstar\cdot\x>\gamma\norm{\wstar}\}}}\geq\frac{\lambda}{2}\|\wstar-\vec w\|_2^2\;.
$$
\end{claim}

\noindent Since $h(r) \leq B/r^{4+\rho}$ and $h(r)$ is decreasing, 
such an $r_0$ exists and we can always make $r_0\geq 1$.

\begin{proof}[Proof of \Cref{app:clm:approx-error}]
    By \Cref{assmpt:margin}, we have that  $\EE_{\x\sim\D_\x}\lp[(\wstar\cdot\x)^2\1\{\wstar\cdot \x\geq \gamma\|\wstar\|_2\}\rp]\geq \lambda \|\wstar\|_2^2$.
Let $\event_2=\{\vec w\cdot \x\leq -2r_0\|\wstar\|_2,\wstar\cdot \x\geq\gamma\|\wstar\|_2\}$. We have that
\begin{align*}
     &\E{(\w\cdot\x-\wstar\cdot\x)^2\1{\{\w\cdot\x > -2r_0\|\wstar\|_2,\,\wstar\cdot\x>\gamma\norm{\wstar}\}}}\\
     =&\; \EE_{\x\sim\D_\x}\lp[(\w\cdot\x-\wstar\cdot\x)^2\1\{\wstar\cdot \x\geq \gamma\|\wstar\|_2\}\rp] - \E{(\w\cdot\x-\wstar\cdot\x)^2\1_{\event_2}(\x)}\\
     \geq&\;\lambda\|\wstar-\vec w\|_2^2-\E{(\w\cdot\x-\wstar\cdot\x)^2\1_{\event_2}(\x)}.
     \end{align*}
    By the Cauchy-Schwarz inequality, we get
\begin{align*}
    \E{(\w\cdot\x-\wstar\cdot\x)^2\1_{\event_2}(\x)}&\leq  \E{(\w\cdot\x - \wstar\cdot\x)^2 \1\{\vec w\cdot \x\leq -2r_0\|\wstar\|_2\}}\\
    &\leq \|\w-\wstar\|_2^2\max_{\vec u\in \ball(1)}\sqrt{\E{(\vec u\cdot\x)^4}}\sqrt{\pr{ \vec w\cdot \x\leq -2r_0\|\wstar\|_2\}}]}\;.
\end{align*}
     Since $\w\in\ball(2\norm{\w^*}_2)$, it holds that $\w/(2\norm{\w^*}_2)\in\ball(1)$. Thus, from the concentration properties of $\D_\x$, it follows that $\pr{ \vec w\cdot \x\leq -2r_0\|\wstar\|_2\}}\leq h(r_0)$. 
     It remains to bound $\max_{\bu\in\ball(1)} \E{(\bu\cdot\x)^4}$. It is not hard to see that for distributions satisfying the concentration property of \Cref{assmpt:concentration}, $\max_{\bu\in\ball(1)} \E{(\bu\cdot\x)^4}$ as well as $\max_{\bu\in\ball(1)} \E{(\bu\cdot\x)^2}$ are at most $5B/\rho$. The proof of the following simple fact can be found in \Cref{app:proof-of-claim:bound-H(0)}.
     \begin{fact}\label{claim:bound-H(0)}
    Let $\D_\x$ be a distribution satisfying \Cref{assmpt:concentration}. Then  
    $\max_{\bu\in\ball(1)} \E{(\bu\cdot\x)^i}\leq 5B/\rho$ for $i = 2,4$.
    \end{fact}
  
Although only the bound on the $4^\mathrm{th}$ order moment is needed here, the upper bound on $\max_{\bu\in\ball(1)} \E{(\bu\cdot\x)^2}$ will also be  used in later sections.

     Therefor, by our choice of $r_0$, we have $h(r_0)\leq \frac{\lambda^2\rho}{20B}$, hence $\max_{\bu\in\ball(1)} \E{(\bu\cdot\x)^4}h(r_0)\leq \lambda^2/4$. 
Therefore, 
\[ \E{(\w\cdot\x-\wstar\cdot\x)^2\1_{\event_2}(\x)}\leq (\lambda/2)\|\w-\wstar\|_2^2\;,
\]
completing the proof of \Cref{app:clm:approx-error}.
    \end{proof}
Combining \Cref{app:eq:final-bound} and \Cref{app:clm:approx-error}, we get:
    \begin{equation*}
        \nabla \bLsur(\w)\cdot(\w - \w^*)\gtrsim \frac{\gamma\lambda\beta}{r_0}\|\vec w-\wstar\|_2^2.
    \end{equation*}
 To complete the proof, it remains to choose $r_0$ appropriately. By \Cref{app:clm:approx-error}, we need to select $r_0$ to be sufficiently large so that $h(r_0)\leq \lambda^2\rho/(20B)$. By \Cref{assmpt:concentration}, we have that $h(r)\leq B/r^{4 + \rho}$. Thus, we can choose $r_0=5B/(\lambda\rho)$, which is at least $1$ by our assumptions. This completes the proof of the lemma.

\end{proof}

\subsection{Proof of \Cref{claim:bound-H(0)}}\label{app:proof-of-claim:bound-H(0)}
We restate and prove the following fact.

\begin{fact}
    Let $\D_\x$ be a distribution satisfying \Cref{assmpt:concentration}. Then  
    $\max_{\bu\in\ball(1)} \E{(\bu\cdot\x)^i}\leq 5B/\rho$ for $i = 2,4$.
    \end{fact}
\begin{proof}
    Let $i=2$ or $4$. By \Cref{assmpt:concentration}, for any unit vector $\bu$, we have
    \begin{align*}
        \E{(\bu\cdot\x)^i} &= \int_0^\infty \pr{(\bu\cdot\x)^i\geq t}\diff{t}\\
        &=\int_0^\infty is^{i-1}\pr{|\bu\cdot\x|\geq s}\diff{s}\\
        &\leq \int_{0}^\infty is^{i-1}\min\{1, h(s)\}\diff{s}.
    \end{align*}
    By \Cref{assmpt:concentration} we have $h(s)\leq B/s^{4+\rho}$ for some $1\geq\rho>0$ and $B\geq 1$, thus it further holds
    \begin{align*}
        \E{(\bu\cdot\x)^i} &\leq \int_{0}^1 is^{i-1}\diff{s} + \int_1^\infty is^{i-1}h(s)\diff{s}\leq 1 + B\int_1^\infty is^{i-1}\frac{1}{s^{4 + \rho}}\diff{s}\leq \frac{5B}{\rho}.
    \end{align*}
\end{proof}

\section{Full Version of \Cref{sec:fast-rate-for-sharp-objective}}\label{app:section4}

\subsection{The Landscape of Surrogate Loss}\label{app:full-version-sec-4.1}

\begin{theorem}[Landscape of Surrogate Loss]\label{app:thm:landscape-for-sur-loss}
Let $\Bar{\mu}\in(0,1]$ and $\alpha,\kappa\geq 1$. Let $\D$ be a distribution supported on $\R^d\times \R$ and let $\sigma:\R\mapsto\R$ be an $(\alpha,\beta)$-unbounded activation for some $\beta>0$. 
Furthermore, assume that the maximum eigenvalue of the matrix $\EE_{\x\sim\D_\x}[\x\x^\top]$ is $\kappa$. Further, fix $\wstar \in\mathcal{W}^*$ and suppose $\bLsur$ is $\Bar{\mu}$-sharp with respect to $\wstar$ in a subset $S_1\subseteq \R^d$. Let $S_2=\{\w:\Ltwo(\w)\geq (4\alpha\kappa/\Bar{\mu})^2\Ltwo(\wstar) \}$. Then
for any $\vec w\in S_1\cap S_2$, we have 
$$\| \nabla \Lsur(\w)\|_2\leq \alpha\sqrt\kappa\norm{\w-\w^*}_2 + \sqrt{\kappa\Ltwo(\wstar)}\;,$$
and
$$\| \nabla \Lsur(\w)\|_2\geq \frac{\Bar{\mu}}{4\alpha\sqrt{\kappa}}\sqrt{\Ltwo(\w) }\;.$$
\end{theorem}

If we can assume that the set $S_1$ of \Cref{app:thm:landscape-for-sur-loss} is convex and that there is no local minima in the boundary of $S_1$, then by running any convex-optimization algorithm in the feasible set $S_1$, we guarantee that we converge either to a local minimum which has zero gradient or to a point inside the set $(S_2)^c$ where the true loss is sufficiently small. The next corollary shows that this is indeed the case for a distribution that satisfies \Cref{assmpt:margin,assmpt:concentration}.

\begin{corollary}
\label{app:cor:landscape-assuptions}
Let $\D$ be a distribution supported on $\R^d\times \R$ and let $\sigma:\R\mapsto\R$ be an $(\alpha,\beta)$-unbounded activation.
     Fix $\wstar \in \mathcal W^*$ and suppose that $\D_\x$ satisfies \Cref{assmpt:margin,assmpt:concentration} with respect to $\wstar$. Furthermore, let $C>0$ be a sufficiently small absolute constant and let $\bar{\mu}=C \lambda^2\gamma \beta \rho/B$. 
     Then, for any $\eps>0$ and $\hat{\w}\in \ball(2\|\w^*\|_2)$, so that $\Lsur(\hat{\w})-\inf_{\w\in\ball(2\|\w^*\|_2)}\Lsur(\w) \leq \eps$, it holds
     \[
     \Ltwo(\hat{\w})\leq O((\alpha B/(\rho\Bar{\mu}))^2 )(\Ltwo(\wstar)+\alpha\eps)\;.
     \]
\end{corollary}
\begin{proof}[Proof of \Cref{app:cor:landscape-assuptions}]
Denote $\mathcal{K}$ as the set of $\hat{\w}$ such that $\hat{\w}\in \ball(2\|\w^*\|_2)$ and $\Lsur(\hat{\w})-\inf_{\w\in\ball(2\|\w^*\|_2)}\Lsur(\w) \leq \eps$. First, note that as claimed in \Cref{claim:bound-H(0)}, $\EE_{\x\sim\D_\x}[\x\x^\top]\preceq (5B/\rho)\vec I$ for any unit vector $\bu$ when \Cref{assmpt:concentration} holds.

Next, observe that the set of minimizers of the loss $\Lsur$ inside the ball $\ball(2\|\w^*\|_2)$ is convex. Furthermore, the set $\ball(2\|\w^*\|_2)$ is compact. Thus, for any point $\w'\in \ball(2\|\w^*\|_2)$ that minimizes $\Lsur$ it will either hold that $\|\nabla \Lsur(\w')\|_2=0$ or $\w'\in \partial \ball(2\|\w^*\|_2)$. Let $\mathcal{W}_\mathrm{sur}^*$ be the set of minimizers of $\Lsur$.

We first show that if there exists a minimizer $\w'\in \mathcal{W}_\mathrm{sur}^*$ such that $\w'\in \partial \ball(2\|\w^*\|_2)$, then any point $\w$ inside the set $\ball(2\|\w^*\|_2)$ gets error proportional to $\Ltwo(\wstar)$. Observe for such point $\hat{\w}$, by the necessary condition of optimality, it should hold
\begin{equation}\label{eq:contradiction}
    \nabla \Lsur(\w')\cdot(\w'-\w)\leq 0\;,
\end{equation}
for any $\w \in \ball(2\|\w^*\|_2)$.
Using \Cref{app:cor:true-sharpness}, we get that either $ \nabla \Lsur(\w')\cdot(\w'-\wstar)\geq (\Bar{\mu}/2)\|\w'-\wstar\|_2^2$
or $\w'\in \{\w:\|\w-\wstar\|_2^2\leq (20B/(\bar{\mu}^2\rho))\Ltwo(\w^*)\}$. But \Cref{eq:contradiction} contradicts with $ \nabla \Lsur(\w')\cdot(\w'-\wstar)\geq (\Bar{\mu}/2)\|\w'-\wstar\|_2^2>0$ since $\w'\in \partial \ball(2\norm{\w^*}_2)$, $\norm{\w'}_2 = 2\norm{\w^*}_2$ hence $\w'\neq \w^*$. So it must be the case that $\w'\in \{\w:\|\w-\wstar\|_2^2\leq (20B/(\bar{\mu}^2\rho))\Ltwo(\w^*)\}$. Again, we have that $\w'\in \partial \ball(2\|\w^*\|_2)$, therefore $\|\w'-\wstar\|_2\geq \|\wstar\|_2$. Hence, $(20B/(\bar{\mu}^2\rho))\Ltwo(\w^*)\geq \|\wstar\|_2^2\geq (1/9)\|\w - \w^*\|_2^2$ for any $\w\in\ball(2\norm{\w^*}_2)$. Therefore, for any $\w\in \ball(2\|\w^*\|_2)$, we have
\begin{align}\label{eq:app:basic}
        \Ltwo(\w) &= \Ey{(\sigma(\w\cdot\x) - y)^2}\nonumber\\
        &\leq 2\Ltwo(\wstar) +2\E{(\sigma(\w\cdot\x) - \sigma(\w^*\cdot\x))^2} \nonumber\\
        &\leq 2\Ltwo(\wstar) + 10B\alpha^2/\rho\norm{\w - \w^*}_2^2  \\
        &\leq O(B^2\alpha^2/(\Bar{\mu}^2\rho^2))\Ltwo(\wstar)\;,\nonumber
    \end{align}
where in the second inequality we used the fact that $\E{\x\x^\top} \preceq (5B/\rho)\vec I$ and $\sigma$ is $\alpha$-Lipschitz. Since the inequality above holds for any $\w\in\ball(2\norm{\w^*}_2)$, it will also be true for $\hat{\w}\in\mathcal{K}\subseteq \ball(2\norm{\w^*}_2)$.
    
It remains to consider the case where the minimizers $\mathcal{W}_\mathrm{sur}^*$ are strictly inside the $\ball(2\|\w^*\|_2)$. Note that $\Lsur(\w)$ is $\alpha$-smooth.
Therefore, we get that for any $\hat{\w}\in \mathcal K$, it holds $\|\nabla\Lsur(\hat{\w})\|_2^2\leq 2\alpha\eps$. By applying \Cref{app:cor:true-sharpness}, we get that either $\|\hat{\w}-\wstar\|_2^2\leq (20B/(\Bar{\mu}^2\rho)) \Ltwo(\wstar)$ or that $\sqrt{2\alpha\eps}\geq (\bar{\mu}/2)\|\hat{\w}-\wstar\|_2$. Therefore we get that, $\|\hat{\w}-\wstar\|_2^2\leq (20B/(\Bar{\mu}^2\rho)) (\Ltwo(\wstar)+\alpha\eps)$. Then the result follows from \Cref{eq:app:basic}.

\end{proof}

To prove \Cref{app:thm:landscape-for-sur-loss}, we need the following proposition which shows that if the current vector $\vec w$ is sufficiently far away from the true vector $\wstar$, then the gradient of the surrogate loss has a large component in the direction of $\vec w-\wstar$; in other words, the surrogate loss is sharp.

\begin{proposition} \label{app:prop:sharpness-to-true-loss} 
Let $\D$ be a distribution supported on $\R^d\times \R$ 
and let $\sigma:\R\mapsto\R$ be an $(\alpha,\beta)$-unbounded activation.
Furthermore, assume that the maximum eigenvalue of the matrix 
$\EE_{\x\sim\D_\x}[\x\x^\top]$ is $\kappa>0$.  
Fix $\wstar \in\mathcal W^*$ and suppose $\bLsur$ is $\Bar{\mu}$-sharp for some $\Bar{\mu}>0$
 with respect to $\wstar$ in a nonempty subset $S_1\subseteq \R^d$. 
Further, let $S_2=\{\w:\|\w-\wstar\|_2^2\geq 4(\kappa/\bar{\mu}^2)\Ltwo(\wstar)\}$. Then
for any $\vec w\in S_1\cap S_2$, we have \[ \nabla \Lsur(\w)\cdot (\w-\wstar )\geq (\bar{\mu}/2)\|\w-\wstar\|_2^2 \;.\]
\end{proposition}
\begin{proof}[Proof of \Cref{app:prop:sharpness-to-true-loss}]
 We show that $\nabla \Lsur(\w)\cdot (\w-\wstar)$ is bounded sufficiently far away from zero.
    We decompose the gradient into the noise-free part and the noisy, i.e., $\nabla \Lsur(\w)=\nabla \bLsur(\w) +\nabla \Lsur(\wstar)$. 
  First, we bound the noisy term in the direction $\w-\wstar$, which yields
    \begin{align*}
\nabla \Lsur(\wstar)\cdot(\w-\wstar)&\geq-\mathbf{E}_{(\x,y)\sim \D}[|\sigma(\wstar\cdot\x)-y||\w\cdot\x-\wstar\cdot\x|]
\\
&\geq- 
\sqrt{\Ltwo(\wstar)}\|\w-\wstar\|_2\sqrt{\kappa} \;,   
\end{align*}
where we used the \CS inequality and that $\EE_{\x\sim\D_\x}[\x\x^\top]\preceq \kappa \vec I$. 
Next, we bound the contribution of $\nabla \bLsur(\w)$ in the direction $\w-\wstar$. Using the fact that $\bLsur(\w)$ is $\bar{\mu}$-sharp for any $\vec w\in S_1$, it holds that 
    \[\nabla\bLsur(\w)\cdot(\w-\wstar)\geq \bar{\mu}\|\w-\wstar\|_2^2\;.\]
    Combining everything together we have that 
    \begin{equation*}
        \quad \nabla\Lsur(\w)\cdot(\w-\wstar)\geq \bar{\mu}\|\w-\wstar\|_2\left(\|\w-\wstar\|_2-(\sqrt \kappa/\bar{\mu})\sqrt{\Ltwo(\wstar)}\right)\;.
    \end{equation*}
    The proof is completed by taking any $\w\in S_1\cap S_2$, where $\|\w-\wstar\|_2\geq (2\sqrt\kappa/\bar{\mu})\sqrt{\Ltwo(\wstar)}$, and therefore 
    \[
       \nabla\Lsur(\w)\cdot(\w-\wstar)\geq (\bar{\mu}/2)\|\w-\wstar\|_2^2\;.
    \]
\end{proof}

\begin{corollary}
\label{app:cor:true-sharpness}
Let $\D$ be a distribution supported on $\R^d\times \R$ and let $\sigma:\R\mapsto\R$ be an $(\alpha,\beta)$-unbounded activation.
    Suppose that $\D_\x$ satisfies \Cref{assmpt:margin,assmpt:concentration} and let $C>0$ be a sufficiently small absolute constant and let $\bar{\mu}=C \lambda^2\gamma \beta \rho/B$. Fix $\wstar \in \mathcal W^*$ and let $S=\ball(2\|\w^*\|_2)-\{\w:\|\w-\wstar\|_2^2\leq \frac{20B}{\bar{\mu}^2\rho}\Ltwo(\w^*)\}$. Then, the surrogate loss $\Lsur$ is $\bar{\mu}$-sharp in  $S$, i.e.,
\begin{equation*}
        \nabla \Lsur(\w)\cdot(\w-\wstar)\geq (\bar{\mu}/2)\norm{\w - \wstar}_2^2,\;\; \forall \w \in S.
    \end{equation*}
\end{corollary}
\begin{proof}[Proof of \Cref{app:cor:true-sharpness}]
    Note that $\max_{\bu\in\ball(1)}\E{(\bu\cdot\x)^2} = \kappa\leq 5B/\rho$ as proven in \Cref{claim:bound-H(0)}. Then combining \Cref{app:prop:sharpness-to-true-loss} and \Cref{lem:sharpness-well-behaved} we get the desired result.
\end{proof}

\begin{proof}[Proof of \Cref{app:thm:landscape-for-sur-loss}]
     Using \Cref{app:prop:sharpness-to-true-loss}, we get that for any $\vec w\in S'\cap S_1$, where $S'=\{\w:\|\w-\wstar\|_2^2\geq 4(\kappa/\Bar{\mu}^2)\Ltwo(\wstar) \}$, we have that $\nabla \Lsur(\w)\cdot (\w-\wstar )\geq (\Bar{\mu}/2)\|\w-\wstar\|_2^2$. Note that 
    \begin{align*}
        \Ltwo(\w) &= \Ey{(\sigma(\w\cdot\x) - y)^2}\\
        &\leq 2\E{(\sigma(\w\cdot\x) - \sigma(\w^*\cdot\x))^2} + 2\Ey{(\sigma(\w^*\cdot\x) - y)^2}\\
        &\leq 2\alpha^2\kappa\norm{\w - \w^*}_2^2 + 2\Ltwo(\wstar) \leq 2\alpha^2\kappa\norm{\w - \w^*}_2^2 + (1/2)\Ltwo(\w)\;,
    \end{align*}
    where we used that $\Ltwo(\w)\geq 4\Ltwo(\wstar)$. Hence, it holds $4\alpha^2\kappa\norm{\w - \w^*}_2^2 \geq \Ltwo(\w)$.
    Therefore, when $\w\in S_2$, it holds that $\norm{\w - \w^*}_2^2\geq (4\alpha\kappa/\Bar{\mu})^2\Ltwo(\w)$, hence $S_2\subseteq S'$. 

    Now observe that for any unit vector $\vec v\in \R^d$, it holds $\|\nabla \Lsur(\w)\|_2\geq \vec v\cdot \nabla \Lsur(\w)$. Therefore, for any $\w\in S_1\cap S_2\subseteq S_1\cap S'$, we have
    \begin{align*}
        \|\nabla \Lsur(\w)\|_2\geq \nabla \Lsur(\w)\cdot\bigg(\frac{\w-\wstar}{\|\w-\wstar\|_2}\bigg) \geq(\Bar{\mu}/2)\norm{\w - \w^*}_2\geq \frac{\Bar{\mu}}{4\alpha\sqrt{\kappa}}\sqrt{\Ltwo(\w)}\;.
    \end{align*}
We now show that the gradient is also bounded from above.
        By definition, we have
    \begin{align*}
        \norm{\nabla \Lsur(\w)}_2 &= \bigg\|\Ey{(\sigma(\w\cdot\x) - y)\x}\bigg\|_2\\
        &\leq \bigg\|\E{(\sigma(\w\cdot\x) - \sigma(\w^*\cdot\x))\x}\bigg\|_2 + \bigg\|\Ey{(\sigma(\w^*\cdot\x) - y)\x}\bigg\|_2\\
        &\leq \max_{\norm{\bu}_2\leq 1}\E{|\sigma(\w\cdot\x) - \sigma(\w^*\cdot\x)||\bu\cdot\x|} + \max_{\norm{\bv}_2\leq 1}\Ey{|\sigma(\w^*\cdot\x) - y||\bv\cdot\x|}
    \end{align*}
    Applying \CS to the inequality above, we further get
    \begin{align*}
        \norm{\nabla \Lsur(\w)}_2 &\leq\max_{\norm{\bu}_2\leq 1}\sqrt{\E{|\sigma(\w\cdot\x) - \sigma(\w^*\cdot\x)|^2}\E{|\bu\cdot\x|^2}}\\
        &\quad + \max_{\norm{\bv}_2\leq 1}\sqrt{\Ey{|\sigma(\w^*\cdot\x) - y|^2}\E{|\bv\cdot\x|^2}}\\
        &\leq \alpha\sqrt\kappa\norm{\w-\w^*}_2 + \sqrt{\kappa\Ltwo(\wstar)}\;,
    \end{align*}
    where in the last inequality we used the fact that $\sigma$ is $\alpha$-Lipschitz and that the maximum eigenvalue of $\EE_{\x\sim\D_\x}[\x\x^\top]$ is $\kappa$.
\end{proof}

\subsection{Fast Rates
for Surrogate Loss}\label{app:full-version-sec-4.2}

In this section, we proceed to show that when the surrogate loss is sharp, then applying batch Stochastic Gradient Descent (SGD) on the empirical surrogate loss obtains a $C$-approximate parameter $\hat{\w}$ of the $L_2^2$ loss in linear time. To be specific, consider the following iteration update
\begin{equation}\label{app:alg:sgd}
    \w^{(t+1)} = \argmin_{\w\in\ball(W)} \bigg\{\w\cdot\g^{(t)} + \frac{1}{2\eta}\norm{\w - \w^{(t)}}^2_2\bigg\},
\end{equation}
where $\eta$ is the step size and $\g^{(t)}$ is the empirical gradient of the surrogate loss:
\begin{equation}
    \g^{(t)} = \frac{1}{N}\sum_{j=1}^N (\sigma(\w^{(t)}\cdot\x(j))-y(j))\x(j).
\end{equation}

The algorithm is summarized in \Cref{app:alg:gd-3}. 

\begin{algorithm}\caption{Stochastic Gradient Descent on Surrogate Loss}
   \label{app:alg:gd-3}
\begin{algorithmic}
   \STATE {\bfseries Input:} Iterations: $T$, sample access from $\D$, batch size $N$, step size $\eta$, bound $M$.
   \STATE Initialize $\vec w^{(0)} \gets \vec 0$.
\FOR{$t=1$ {\bfseries to} $T$}
\STATE Draw $N$ samples $\{(\x(j), y(j))\}_{j=1}^N\sim\D$.
   \STATE For each $j\in[N]$, $y(j)\gets\sgn(y(j))\min(|y(j)|,M)$. \STATE Calculate 
   $$\vec g^{(t)} \gets\frac{1}{N}\sum_{j=1}^N (\sigma(\w^{(t)}\cdot\x(j))-y(j))\x(j).$$ \STATE $\vec w^{(t+1)} \gets \vec w^{(t)}-\eta \vec{g}^{(t)}$.
\ENDFOR
   \STATE {\bfseries Output:} The weight vector $\w^{(T)}$.
\end{algorithmic}
\end{algorithm}

Further, for simplicity of notation, we use $\Bar{\g}^{(t)}$ to denote the empirical gradient of the noise-free surrogate loss:
\begin{equation}
    \Bar{\g}^{(t)} = \frac{1}{N}\sum_{j=1}^N (\sigma(\w^{(t)}\cdot\x(j))-\sigma(\wstar\cdot\x(j)))\x(j).
\end{equation}

In addition, we define the following helper functions $H_2$ and $H_4$.
\begin{definition}\label{app:def:H-2&H-4}
    Let $\D_\x$ be a distribution on supported on $\R^d$ that satisfies \Cref{assmpt:concentration}  we define non-negative non-increasing functions $H_2$ and $H_4$ as follows:
    \begin{gather*}
        H_2(r) \triangleq \max_{\bu\in\ball(1)} \E{(\bu\cdot\x)^2\1\{|\bu\cdot\x|\geq r\}},\\
        H_4(r) \triangleq \max_{\bu\in\ball(1)} \E{(\bu\cdot\x)^4\1\{|\bu\cdot\x|\geq r\}}.
    \end{gather*}
\end{definition}

\begin{remark}
    In particular, when $r=0$, $H_2(0)$ and $H_4(0)$ bounds from above the second and fourth moments. Recall that in \Cref{claim:bound-H(0)}, it is proved that $H_2(0),H_4(0)\leq 5B/\rho$.
\end{remark}

Now we state our main theorem.

\begin{theorem}[Main Algorithmic Result]\label{app:thm:l2-fast-rate-thm}
    Fix $\eps>0$ and $W>0$ and suppose \Cref{assmpt:activation,assmpt:margin,assmpt:concentration} hold.
Let $\mu:=\mu(\lambda,\gamma,\beta,\rho,B)$ be a sufficiently small constant multiple of $\lambda^2\gamma\beta\rho/B$, and let $M = \alpha W H_2^{-1}\big(\frac{\eps}{4\alpha^2 W^2}\big)$.
Further, choose parameter $r_\eps$ large enough so that $H_4(r_\eps)$ is a sufficiently small constant multiple of $\eps$.
Then after
    \begin{equation*}
        T = \wt{\Theta}\left(\frac{B^2\alpha^2}{\rho^2\mu^2}\log\lp(\frac{W}{ \eps}\rp)\right)
\end{equation*}
    iterations with batch size $N = \Omega(dT(r_\eps^2 + \alpha^2M^2))$, \Cref{app:alg:gd-3} converges to a point $\w^{(T)}$ such that 
    $$\Ltwo(\vec w^{(T)}) = O\lp(\frac{B^2\alpha^2}{\rho^2 \mu^2}\opt \rp)+\eps\;,$$ 
    with probability at least $2/3$. 
    
\end{theorem}

We now provide a brief overview of the proof. As follows from \Cref{app:cor:landscape-assuptions}, when we find a vector $\hat{\w}$ that minimizes the surrogate loss, then this $\hat{\w}$ is itself a $C$-approximate solution of \Cref{def:agnostic-learning}. However, minimizing the surrogate loss can be expensive in computational and sample complexity. \Cref{app:cor:true-sharpness} says that we can achieve strong-convexity-like rates as long as we are far away from a minimizer of the $L_2^2$ loss, i.e., when $\|\w - \w^*\|_2^2\geq O(\opt)$. Roughly speaking, we would like to show that at each iteration $t$, it holds $||\w^{(t+1)} - \w^*||_2^2\leq C||\w^{(t)} - \w^*||_2^2$ where $0<C<1$ is some constant depending on the parameters $\alpha, \beta, \mu$, $\rho$ and $B$. Then since the distance from $\w^{(t)}$ to $\w^*$ contracts fast, we are able to get the linear convergence rate of the algorithm. To this end, we prove that under a sufficiently large batch size, the empirical gradient of the surrogate loss $\g^{(t)}$ approximates $\nabla \Lsur(\w^{(t)})$ with a small error. Thus, $||\w^{(t+1)} - \w^*||_2^2$ can be written as
\begin{align*}
    ||\w^{(t+1)} - \w^*||_2^2 = ||\w^{(t)} - \w^*||_2^2 - 2\eta \nabla\Lsur(\w^{(t)})\cdot(\w^{(t)} - \w^*) + (\mathrm{error}).
\end{align*}

We then apply the sharpness property of the surrogate (\Cref{app:prop:sharpness-to-true-loss}) to the inner product $\nabla\Lsur(\w^{(t)})\cdot(\w^{(t)} - \w^*)$, which as a result leads to $||\w^{(t+1)} - \w^*||_2^2 \leq (1-2\eta\mu)||\w^{(t)} - \w^*||_2^2 + (\mathrm{error})$. By choosing the parameters $\eta$ and the batch size $N$ carefully, one can show that 
$$||\w^{(t+1)} - \w^*||_2^2 \leq (1-C)||\w^{(t)} - \w^*||_2^2 + C'(\opt + \eps),$$
indicating a fast contraction $||\w^{(t+1)} - \w^*||_2^2 \leq (1-C/2)||\w^{(t)} - \w^*||_2^2$ whenever $C'(\opt + \eps)\leq (C/2)||\w^{(t)} - \w^*||_2^2$.

To prove the theorem, we provide some supplementary lemmata. The following lemma states that we can truncate the labels $y$ to $y' \leq M$, where $M$ is a parameter determined by distribution $\D_\x$.

\begin{lemma}\label{app:lem:y-bounded-by-M}
    Define $y' = \sgn(y)\min(|y|, M)$ for $M=\alpha W H_2^{-1}(\frac{\eps}{4\alpha^2 W^2})$, then:
    \begin{equation*}
        \Ey{(\sigma(\wstar\cdot\x) - y')^2} = \opt+\eps,
    \end{equation*}
    meaning that we can consider $y'$ instead of $y$ and assume $|y|\leq M$ without loss of generality, where $H_2$ was defined in \Cref{app:def:H-2&H-4}.
\end{lemma}

\begin{proof}[Proof of \Cref{app:lem:y-bounded-by-M}]
    Fix $M>0$, and denote $P:\R\to\R$ the operator that projects the points of $\R$ onto the interval $[-M,M]$, i.e., $P(t)=\sgn(t)\min(|t|, M)$. To prove the aforementioned claim, we split the expectation into two events: the first event is when $|\wstar\cdot \x|\leq  (M/\alpha)$ and the second when the latter is not true. Observe that in the first case, $P(\sigma(\wstar\cdot \x))=\sigma(\wstar\cdot \x)$,
hence, using the fact that $P$ is non-expansive, we get
      \begin{align*}
        &\quad \Ey{(\sigma(\wstar\cdot\x) - P(y))^2\1\{|\wstar\cdot \x|\leq (M/\alpha)\}} \\
        &= \Ey{(P(\sigma(\wstar\cdot\x)) - P(y))^2\1\{|\wstar\cdot \x|\leq (M/\alpha)\}}\\
        &\leq  \Ey{(\sigma(\wstar\cdot\x - y)^2\1\{|\wstar\cdot \x|\leq (M/\alpha)\}}
        \\&\leq \opt\;.
    \end{align*}
    It remains to bound the error in the event that $|\wstar\cdot\x |> (M/\alpha)$. In this event $\alpha |\wstar\cdot \x|\geq |P(y)|$, and so we have
         \begin{align*}
        \Ey{(\sigma(\wstar\cdot\x) - P(y))^2\1\{|\wstar\cdot \x|> (M/\alpha)\}} &\leq  
        4\alpha^2\Ey{(\wstar\cdot\x)^2\1\{|\wstar\cdot \x|> (M/\alpha)\}}\\
        &\leq  4 \alpha^2\|\wstar\|_2^2 H_2(M/(\alpha W))\leq\eps\;,
    \end{align*}
    where in the first inequality we used the standard inequality $(a+b)^2\leq 2(a^2+b^2)$ and that $\sigma$ is $\alpha$-Lipschitz hence $|\sigma(\w^*\cdot\x)| = |\sigma(\w^*\cdot\x) - \sigma(0)|\leq \alpha|\w^*\cdot\x|$.
\end{proof}

Next, we show that the difference between the empirical gradients and the population gradients of the surrogate loss can be made small by choosing a large batch size $N$. Specifically, we have:

\begin{lemma}\label{app:lem:empirical-grad-approx-population}
Suppose $N$ samples $\{(\x{(j)}, y{(j)})\}_{j=1}^N$ are drawn from $\D$ independently and suppose \Cref{assmpt:activation,assmpt:margin,assmpt:concentration} hold. Let $\g^*$ be the empirical gradient of $\Lsur$ at $\w^*$ and let $\Bar{\g}^t$ be the empirical gradient of $\bLsur(\w^t)$, i.e.,
    \begin{gather*}
        \g^* = \frac{1}{N}\sum_{j=1}^N (\sigma(\w^*\cdot\x(j))-y(j))\x(j),\\
        \Bar{\g}^{(t)} = \frac{1}{N}\sum_{j=1}^N (\sigma(\w^{(t)}\cdot\x(j))-\sigma(\wstar\cdot\x(j)))\x(j). 
    \end{gather*}
     Moreover, let $H_4(r)$ be defined as in \Cref{app:def:H-2&H-4}. Then for a fixed positive real number $r_\eps$ satisfying $H_4(r_\eps)\lesssim \epsilon$ and $r_\eps\geq 1$, we have the following bounds holds with probability at least $1 - \delta$:
    \begin{equation}\label{app:eq:||g*-nabla L*||_2}
        \norm{\g^* - \nabla\Lsur(\wstar)}_2 \lesssim\sqrt{\frac{d(r_\eps^2\opt +  \alpha^2 M^2\epsilon)}{\delta N}},
    \end{equation}
    and similarly:
    \begin{equation}\label{app:eq:||tilde g*-nabla tilde L*||_2}
        \norm{\Bar{\g}^{(t)}-\nabla\bLsur(\w^{(t)})}_2\lesssim \sqrt{\frac{\alpha^2 d B}{\delta \rho N}} \norm{\w^{(t)} - \wstar}_2.
    \end{equation}
\end{lemma}

\begin{proof}[Proof of \Cref{app:lem:empirical-grad-approx-population}]
    The proof follows from a direct application of Markov's inequality and a careful bound on the variance term using the tail-bound assumptions. To be specific, by Markov's Inequality, for any $\xi>0$ it holds:
    \begin{equation*}
        \pr{\norm{\g^*-\nabla\Lsur(\wstar)}_2\geq \xi} = \pr{\norm{\g^*-\nabla\Lsur(\wstar)}_2^2 \geq \xi^2}\leq \frac{1}{\xi^2}\Ex_{(\x,y)\sim\D}\bigg[\norm{\g^*-\nabla\Lsur(\wstar)}_2^2\bigg].
    \end{equation*}
    Now for the variance term $\Ey{||\g^*-\nabla\Lsur(\wstar)||_2^2}$, recall that each sample $\x(j)$ and $y(j)$ are i.i.d., therefore, we can bound it in the following way
    \begin{align}
        &\quad\Ey{\norm{\g^*-\nabla\Lsur(\wstar)}_2^2} \nonumber\\
        &= \Ex_{(\x,y)\sim\D}\bigg[\frac{1}{N^2}\bigg\|\sum_{j=1}^N \bigg((\sigma(\wstar\cdot\x(j)) - y(j))\x(j) - \Ex_{(\x,y)\sim\D}\big[(\sigma(\wstar\cdot\x(j)) - y(j))\x(j)\big]\bigg)\bigg\|_2^2\bigg] \nonumber\\
        &= \frac{1}{N}\Ex_{(\x,y)\sim\D}\bigg[\big\| (\sigma(\wstar\cdot\x) - y)\x - \Ey{(\sigma(\wstar\cdot\x) - y)\x} \big\|_2^2\bigg] \nonumber\\
        &\leq \frac{1}{N}\Ey{\norm{(\sigma(\wstar\cdot\x) - y)\x}_2^2}, \label{app:eq:g^*-approx-nabla-f^*}
    \end{align}
    where in the second equation we used  that for any mean-zero independent random variables $\bz_j$, we have $\EE[||\sum_{j} \bz_j||_2^2] = \sum_{j}\EE[\norm{\bz_j}_2^2]$, and in the final inequality we used  that for any random variable $X$, it holds $\EE[\norm{X - \EE[X]}_2^2]\leq \EE[\norm{X}_2^2]$.

    Next, we show that $\E{\norm{(\sigma(\wstar\cdot\x) - y)\x}_2^2}$ can be bounded  above in terms of $H_2$ and $H_4$.

    \begin{claim}\label{app:claim: bound on opt^2||x||^2}
$\Ey{\norm{(\sigma(\wstar\cdot\x) - y)\x}_2^2}\lesssim d(r_\eps^2\opt + \alpha^2 M^2 H_4(r_\eps)).$
    \end{claim}

    \begin{proof}[Proof of \Cref{app:claim: bound on opt^2||x||^2}]
        To prove the claim, note that $\norm{\x}_2^2 = \sum_{i=1}^d |\x_i|^2$, therefore by linearity of expectation it holds
        \begin{equation*}
            \Ey{\norm{(\sigma(\wstar\cdot\x) - y)\x}_2^2} = \sum_{i=1}^d \Ey{(\sigma(\wstar\cdot\x) - y)^2\x_i^2}.
        \end{equation*}
        
        Thus, the goal is to bound $\Ey{(\sigma(\wstar\cdot\x) - y)^2\x_i^2}$ effectively for each entry $i$. Deploying the intuition that the probability of $|\x_i| = |\e_i\cdot\x|$ being very large is tiny since we have $\pr{|\e_i\cdot\x|>r}\leq h(r)$ and $h(r) \leq Br^{-(4 + \rho)}$ by the \Cref{assmpt:concentration}, we fix some large $r_\eps$ and bound the expectation by looking separately at the events that $|\x_i|\leq r_\eps$ and $|\x_i|> r_\eps$, i.e.,
        \begin{equation}\label{app:eq:(sigma(w^*x)-y)^2x_i^2}
        \begin{split}
            \Ey{(\sigma(\wstar\cdot\x) - y)^2\x_i^2} 
            &= \Ey{(\sigma(\wstar\cdot\x) - y)^2\x_i^2\1\{|\x_i| \leq r_\eps\}} \\&+ \Ey{(\sigma(\wstar\cdot\x) - y)^2\x_i^2\1\{|\x_i| > r_\eps\}}.
        \end{split}
        \end{equation}
        Note when conditioned on the event $|\x_i|\leq r_\eps$ the bound follows easily as:
        \begin{equation}\label{app:eq:(sigma(w^* x)-y)^2x_i^2-x_i<=r_1}
            \Ey{(\sigma(\wstar\cdot\x) - y)^2\x_i^2\1\{|\x_i| \leq r_\eps\}}\leq r_\eps^2 \Ey{(\sigma(\wstar\cdot\x) - y)^2} = r_\eps^2\opt.
        \end{equation}

        When considering $|\x_i|> r_\eps$, notice that $\sigma$ is $\alpha$-Lipschitz and that $\sigma(0) = 0$ 
, as well as that we assumed $|y|\leq M$ due to \Cref{app:lem:y-bounded-by-M}, therefore, denoting $\bu_{\wstar} = \wstar/\norm{\wstar}_2$, it holds:
        \begin{align*}
             \Ey{(\sigma(\wstar\cdot\x) - y)^2\x_i^2\1\{|\x_i| > r_\eps\}}
             &\leq 2\Ey{((\sigma(\wstar\cdot\x))^2 + y^2)\x_i^2\1\{|\x_i| > r_\eps\}}\\
             &\leq 2\E{(\alpha^2 (\wstar\cdot\x)^2+M^2)\x_i^2\1\{|\x_i| > r_\eps\}}\\
             &\leq 2\alpha^2 \norm{\wstar}_2^2\E{(\bu_{\wstar}\cdot\x)^2\x_i^2\1\{|\x_i| > r_\eps\}} + 2M^2 H_2(r_\eps)\;,
        \end{align*}
        where in the last inequality we used \Cref{app:def:H-2&H-4}. For the first term above, note that $\bu_{\wstar}$ is also a unit vector, so by \Cref{assmpt:concentration} the probability mass of $|\bu_{\wstar}\cdot\x|> r_\eps$ is also small, thus, we can show that $\E{(\bu_{\wstar}\cdot\x)^2\x_i^2\1\{|\x_i| > r_\eps\}}$ is dominated by $r_\eps^2\E{\x_i^2\1\{|\x_i| > r_\eps\}}$, which can then be bounded above by $H_2$ and $H_4$. In detail, we split the expectation by conditioning on the events that $|\bu_{\wstar}\cdot\x|> r_\eps$ and $|\bu_{\wstar}\cdot\x|\leq r_\eps$, then noticing that $\1\{|\x_i| > r_\eps, |\bu_{\wstar}\cdot\x|\leq r_\eps\}\leq \1\{|\x_i| \geq r_\eps\}$, we get: \begin{align}\label{app:eq:bound-(u_w^* x)^2 x_i^2}
         \E{(\bu_{\wstar}\cdot\x)^2\x_i^2\1\{|\x_i|> r_\eps\}}
        & \leq \E{r_\eps^2\x_i^2\1\{|\x_i|> r_\eps, |\bu_{\wstar}\cdot\x|\leq r_\eps\}} \nonumber\\
        &\quad + \E{(\bu_{\wstar}\cdot\x)^2\x_i^2\1\{|\x_i| > r_\eps, |\bu_{\wstar}\cdot\x| > r_\eps\}} \nonumber\\
        &\leq \E{r_\eps^2\x_i^2\1\{|\x_i|> r_\eps\}} \nonumber\\
        &\quad + \sqrt{\E{(\bu_{\wstar}\cdot\x)^4 \1\{|\bu_{\wstar}\cdot\x| > r_\eps\}}\E{\x_i^4\1\{|\x_i|> r_\eps\}}} \nonumber\\
        &\leq r_\eps^2 H_2(r_\eps) + H_4(r_\eps),
    \end{align}
    where the second inequality comes from Cauchy-Schwarz and in the last inequality we applied $H_4(r_\eps)\geq \E{(\bu\cdot\x)^4 \1\{|\bu\cdot\x| \geq r_\eps\}}$ for any $\bu\in\ball(1)$ by \Cref{app:def:H-2&H-4}. Now plugging \Cref{app:eq:bound-(u_w^* x)^2 x_i^2} to the bound we get for $\Ey{\sigma(\wstar\cdot\x) - y)^2\x_i^2\1\{|\x_i| \geq r_\eps\}}$, we have:
    \begin{equation*}
        \Ey{(\sigma(\wstar\cdot\x) - y)^2 \x_i^2\1\{|\x_i| \geq r_\eps\}}\leq 2 \alpha^2 \norm{\wstar}_2^2 (r_\eps^2 H_2(r_\eps) + H_4(r_\eps)) + 2M^2H_2(r_\eps).
    \end{equation*}
    Further recall that by definition:
    $$H_4(r) = \max_{\bu\in\ball(1)}\E{(\bu\cdot\x)^4\1\{|\bu\cdot\x|\geq r\}}\geq \max_{\bu\in\ball(1)} r^2\E{(\bu\cdot\x)^2\1\{|\bu\cdot\x|\geq r\}} = r^2 H_2(r),$$ 
    hence $H_4(r)\geq H_2(r)$ when $r\geq 1$. Then applying these facts along with the fact that $\norm{\w^*}_2\leq M$ simplifies the inequality above to the following:
    \begin{equation}\label{app:eq:(sigma(w^* x)-y)^2x_i^2-x_i>r_1}
        \Ey{(\sigma(\wstar\cdot\x) - y)^2 \x_i^2\1\{|\x_i| \geq r_\eps\}}\lesssim \alpha^2 M^2 H_4(r_\eps).
    \end{equation}

        Combining \Cref{app:eq:(sigma(w^* x)-y)^2x_i^2-x_i>r_1} and \Cref{app:eq:(sigma(w^* x)-y)^2x_i^2-x_i<=r_1} with \Cref{app:eq:(sigma(w^*x)-y)^2x_i^2}, we get:
        \begin{equation*}
        \begin{split}
            \Ey{(\sigma(\wstar\cdot\x) - y)^2\norm{\x}_2^2}&\lesssim d(r_\eps^2\opt + \alpha^2 M^2 H_4(r_\eps)),
        \end{split}
        \end{equation*}
        proving the desired claim.
    \end{proof}

    Plugging \Cref{app:claim: bound on opt^2||x||^2} above back to \Cref{app:eq:g^*-approx-nabla-f^*}, we immediately get:
    \begin{equation*}
        \Ey{\norm{\g^*-\nabla\Lsur(\wstar)}_2^2}\lesssim\frac{d}{N}(r_\eps^2\opt + \alpha^2 M^2 \epsilon),
    \end{equation*}
    given that $H_4(r_\eps)\lesssim \epsilon$. Then choosing $\xi \gtrsim \sqrt{\frac{d}{\delta N}(r_\eps^2\opt + \alpha^2 M^2\epsilon)}$, we get \Cref{app:eq:||g*-nabla L*||_2}:
\begin{equation*}
        \pr{\norm{\g^*-\nabla\Lsur(\wstar)}_2\gtrsim \sqrt{\frac{d}{\delta N}(r_\eps^2 + \alpha^2 M^2)\opt}\,}\leq \delta.
    \end{equation*}

    For \Cref{app:eq:||tilde g*-nabla tilde L*||_2}, we repeat the steps when proving \Cref{app:eq:||g*-nabla L*||_2}. Using Markov inequality again, we have
    \begin{equation*}
    \begin{split}
        \pr{\norm{\Bar{\g}^{(t)}-\nabla\bLsur(\w^{(t)})}_2 \geq \zeta}  &= \pr{\norm{\Bar{\g}^{(t)}-\nabla\bLsur(\w^{(t)})}_2^2 \geq \zeta^2} \\
        &\leq \frac{1}{\zeta^2}\Ex_{\x\sim\D_\x}\bigg[\norm{\Bar{\g}^{(t)}-\nabla\bLsur(\w^{(t)})}^2_2\bigg].
    \end{split}
    \end{equation*}
    The goal is to bound the expectation of the squared norm. Notice that $(\x(j), y(j))\sim\D $ are i.i.d.\ samples, therefore, it holds:
    \begin{align*}
        &\quad \E{\norm{\Bar{\g}^{(t)}-\nabla\bLsur(\w^{(t)})}_2^2} \\
        &= \frac{1}{N^2}\Ex_{\x\sim\D_\x}\bigg[\bigg\| \sum_{j=1}^N \bigg((\sigma(\w^{(t)}\cdot\x(j)) - \sigma(\wstar\cdot\x(j)))\x(j) - \E{(\sigma(\w^{(t)}\cdot\x(j)) - \sigma(\wstar\cdot\x(j)))\x(j)} \bigg) \bigg\|_2^2\bigg] \\
        &= \frac{1}{N} \Ex_{\x\sim\D_\x}\bigg[\norm{(\sigma(\w^{(t)}\cdot\x)-\sigma(\wstar\cdot\x))\x - \E{(\sigma(\w^{(t)}\cdot\x)-\sigma(\wstar\cdot\x))\x}}_2^2\bigg],
    \end{align*}
    because for any i.i.d.\ zero-mean random variables $\bz(j)$ it holds $\EE[||\sum_{j}\bz(j)||_2^2] = \sum_{j}\EE[||\bz(j)||_2^2]$. Note that $\EE[||\bz - \EE[\bz]||_2^2]\leq \EE[||\bz||_2^2]$, therefore, we can further bound the variance of $\Bar{\g}^t-\nabla\bLsur(\w^t)$ as:
    \begin{align}\label{app:eq:variance-of-bar-g^t}
         \E{\norm{\Bar{\g}^{(t)}-\nabla\bLsur(\w^{(t)})}_2^2}&\leq \frac{1}{N}\E{(\sigma(\w^{(t)} \cdot \x) - \sigma(\wstar \cdot \x))^2\norm{\x}_2^2} \nonumber\\
         & = \frac{1}{N}\sum_{i=1}^d\E{(\sigma(\w^{(t)} \cdot \x) - \sigma(\wstar \cdot \x))^2 \x_i^2} \nonumber\\
         &\leq \frac{\alpha^2}{N}\sum_{i=1}^d\E{ (\w^{(t)} \cdot \x - \wstar \cdot \x)^2 \x_i^2},
    \end{align}
    where in the last inequality we used $|\sigma(\w^{(t)} \cdot \x) - \sigma(\wstar \cdot \x)|\leq \alpha|\w^{(t)} \cdot \x - \wstar \cdot \x|$, as $\sigma$ is $\alpha$-Lipschitz.

    It remains to bound $\E{(\w^{(t)} \cdot \x - \wstar \cdot \x)^2\x_i^2}$. Denote $\bu_{\w^{(t)}} = (\w^{(t)} - \wstar)/\norm{\w^{(t)} - \wstar}$, which is a unit vector. Abstracting $\norm{\w^t - \wstar}_2$ from the expectation then applying Cauchy-Schwarz, we get:
    \begin{align*}
        \E{(\w^{(t)} \cdot \x - \wstar \cdot \x)^2 \x_i^2} &= \norm{\w^{(t)} - \wstar}_2^2\E{(\bu_{\w^{(t)}}\cdot\x)^2 \x_i^2}\\
        & \leq \norm{\w^{(t)} - \wstar}_2^2\sqrt{\E{(\bu_{\w^{(t)}}\cdot\x)^4}\E{\x_i^4}}\\
        &\leq \norm{\w^{(t)} - \wstar}_2^2 H_4(0),
    \end{align*}
    where the last inequality comes from $H_4(0) = \max_{\bu\in\ball(1)} \E{(\bu\cdot\x)^4}$, which holds by definition. Further from \Cref{claim:bound-H(0)}, $H_4(0)\leq 5B/\rho$, thus, we get
    \begin{equation}\label{app:eq:bound-on-(w^tx - w^*x)^2x_i^2}
        \E{(\w^{(t)} \cdot \x - \wstar \cdot \x)^2\x_i^2}\leq \frac{5B}{\rho}\norm{\w^{(t)} - \wstar}_2^2.
    \end{equation}
    
    To sum up, plugging \Cref{app:eq:bound-on-(w^tx - w^*x)^2x_i^2} back to \Cref{app:eq:variance-of-bar-g^t}, we have:
    \begin{equation*}
        \E{\norm{\Bar{\g}^{(t)}-\nabla\bLsur(\w^{(t)})}_2^2}\lesssim\frac{\alpha^2 d B}{\rho N}\norm{\w^{(t)} - \wstar}_2^2.
    \end{equation*}
    Finally, choosing $\zeta$ to be a sufficiently small multiple of $\sqrt{\frac{\alpha^2 d B}{\delta \rho N}}\norm{\w^{(t)} - \wstar}_2$, \Cref{app:eq:||tilde g*-nabla tilde L*||_2} follows.

\end{proof}

\begin{corollary}\label{app:cor:empirical-grad-variance}
    Suppose $N$ samples $\{(\x{(j)}, y{(j)})\}_{j=1}^N$ are drawn from $\D$ independently and suppose \Cref{assmpt:activation,assmpt:margin,assmpt:concentration} hold. 
   Let $\g\tth$ be the empirical gradient of $\Lsur(\w\tth)$.
     Moreover, let $H_4(r)$ be defined as in \Cref{app:def:H-2&H-4}. Then for a fixed positive real number $r_\eps$ satisfying $H_4(r_\eps)\lesssim \epsilon$ and $r_\eps\geq 1$, with probability at least $1 - \delta$ it holds
    \begin{equation}\label{app:eq-variance-of-the-gradient}
        \norm{\g\tth - \nabla\Lsur(\w\tth)}_2 \lesssim \sqrt{\frac{d\alpha^2 B}{\delta\rho N}}\left(\norm{\w^{(t)} - \wstar}_2+ \sqrt{r_\eps^2\opt +  M^2\epsilon}\right)\;.
    \end{equation}

\end{corollary}

\begin{corollary}\label{app:cor:empirical-grad-variance-2}
    Let $\D$ be a distribution in $\R^d\times \R$ and suppose \Cref{assmpt:activation,assmpt:margin,assmpt:concentration} hold. 
     Moreover, let $H_4(r)$ be defined as in \Cref{app:def:H-2&H-4}. Fix a positive real number $r_\eps$ satisfying $H_4(r_\eps)\lesssim \epsilon$ and $r_\eps\geq 1$. It holds that
    \begin{equation}\label{app:eq-variance-of-the-gradient-2}
        \Exx{\norm{\nabla\Lsur(\w\tth)}_2^2} \lesssim \frac{d\alpha^2 B}{\rho}\left(\norm{\w^{(t)} - \wstar}_2^2+ {r_\eps^2\opt +  M^2\epsilon}\right)\;.
    \end{equation}

\end{corollary}

We further show that the norm of empirical gradients $\g^*$ and $\Bar{\g}^{(t)}$ can be bounded with respect to $\opt$, $\eps$ and $\|\w^{(t)} - \w^*\|_2$.

\begin{corollary}\label{app:cor:bound-norm-g*-gt}
    Suppose the conditions in \Cref{app:lem:empirical-grad-approx-population} are satisfied. Fix $r_\eps\geq 1$ such that $H_4(r_\eps)$ is a sufficiently small multiple of $\epsilon$. Then with probability at least $1-\delta$, we have:
    \begin{equation}
        \norm{\g^*}_2\lesssim \sqrt{(B/\rho)\opt} + \sqrt{\frac{d (r_\eps^2\opt + \alpha^2 M^2 \epsilon)}{\delta N}},
    \end{equation}
    and
    \begin{equation}
        \norm{\bar{\g}^{(t)}}_2 \lesssim  \frac{\alpha B}{\rho}\lp( 1 + \sqrt{\frac{d \rho}{\delta B N}}\rp)\norm{\w^{(t)} - \wstar}_2.
    \end{equation}
\end{corollary}

\begin{proof}
We first estimate the norm of $\nabla\Lsur(\wstar)$ and $\nabla\bLsur(\w^{(t)})$. For the former, applying the Cauchy-Schwarz inequality, we get:
\begin{align*}
    \norm{\nabla\Lsur(\wstar)}_2 &= \norm{\EE[(\sigma(\wstar\cdot\x) - y)\x]}_2\\
    & = \max_{\norm{\bu}_2 = 1} \EE[(\sigma(\wstar\cdot\x) - y)\bu\cdot \x]\\
    & \leq \max_{\norm{\bu}_2 = 1} \sqrt{\EE[(\sigma(\wstar\cdot\x) - y)^2] \EE[(\bu\cdot \x)^2]}\\
    & \leq \sqrt{\frac{5B}{\rho}\opt}\;,
\end{align*}
where we used that $H_2(0)\leq 5B/\rho$ from \Cref{claim:bound-H(0)}.
In addition, by \Cref{app:lem:empirical-grad-approx-population}, with probability at least $1-\delta$, we have:
\begin{equation*}
    \norm{\g^*-\nabla\Lsur(\wstar)}_2\lesssim \sqrt{\frac{d}{\delta N}(r_\eps^2\opt + \alpha^2 M^2\epsilon)},
\end{equation*}
given that $r_\eps$ is chosen large enough so that $H_4(r_\eps)\lesssim \epsilon$. Then combining with the bound of $\norm{\nabla \Lsur(\wstar)}_2$ above, it holds:
\begin{equation*}
    \norm{\g^*}_2\lesssim \sqrt{\frac{B}{\rho}\opt} + \sqrt{\frac{d (r_\eps^2\opt + \alpha^2 M^2 \epsilon)}{\delta N}}.
\end{equation*}

For the second claim, following the exact same approach and utilizing the fact that $\sigma$ is $\alpha$-Lipschitz continuous again, we have:
\begin{align*}
    \norm{\nabla\bLsur(\w^{(t)})}_2 &= \norm{\E{(\sigma(\w^{(t)}\cdot\x) - \sigma(\wstar\cdot\x))\x}}_2\\
    &= \max_{\norm{\bu}_2 = 1}\E{|\sigma(\w^{(t)}\cdot\x) -\sigma(\wstar\cdot\x)|\bu\cdot \x}\\
    &\leq \alpha\max_{\norm{\bu}_2 = 1} \E{|(\w^{(t)} - \wstar)\cdot \x|\bu\cdot \x}.
\end{align*}
Applying Cauchy-Schwarz inequality, we have
\begin{equation*}
    \norm{\nabla\bLsur(\w^{(t)})}_2\leq \alpha\max_{\norm{\bu}_2=1}\sqrt{\E{((\w^{(t)} - \wstar)\cdot \x)^2}\E{(\bu\cdot\x)^2}} \leq \frac{5\alpha B}{\rho}\norm{\w^{(t)} - \wstar}_2.
\end{equation*}

Then combining with \Cref{app:eq:||tilde g*-nabla tilde L*||_2}, we get the desired claim:
\begin{equation*}
    \norm{\bar{\g}^{(t)}}_2 \lesssim \frac{\alpha B}{\rho}\lp( 1 + \sqrt{\frac{d \rho}{\delta B N}}\rp)\norm{\w^{(t)} - \wstar}_2.
\end{equation*}

\end{proof}

Finally, we can turn to the proof of \Cref{app:thm:l2-fast-rate-thm}.
\begin{proof}[Proof of \Cref{app:thm:l2-fast-rate-thm}]

Recall that for a vector $\hat{\w}$, we have
   \begin{align*}
       \Ltwo(\hat{\w}) = \Ey{(\sigma(\hat{\w}\cdot\x) - y)^2}
        &\leq 2\Ltwo(\wstar) +2\E{(\sigma(\hat{\w}\cdot\x) - \sigma(\w^*\cdot\x))^2} \\
        &\leq 2\Ltwo(\wstar) + 10B\alpha^2/\rho\norm{\hat{\w} - \w^*}_2^2,
   \end{align*}
   where in the last inequality we used the fact that $\sigma$ is $\alpha$-Lipschitz and $\EE_{\x\sim\D_\x}[\x\x^\top]\preceq (5B/\rho)I$ according to \Cref{claim:bound-H(0)}. Thus when the algorithm generates some $\hat{\w}$ such that $\|\hat{\w} - \w^*\|_2^2\leq \eps'$, it holds
   \begin{equation}\label{eq:app:basic-O(opt)}
        \Ltwo(\vec w^{(T)})  \leq 2\opt + (10B\alpha^2/\rho)\eps'
\end{equation}
    yielding a $C$-approximate solution to the \Cref{def:agnostic-learning}.
    Therefore, our ultimate goal is to minimize $\|\w - \w^*\|_2$ efficiently.
To this aim, we study the difference of $\|\w^{(t+1)}-\wstar\|_2^2$ and $\|\w^{(t)}-\wstar\|_2^2$. 
     We remind the reader that for convenience of notation, we denote the empirical gradients as the following
     \begin{gather*}
         \g^{(t)} = \frac{1}{N}\sum_{j=1}^N (\sigma(\w^{(t)} \cdot \x(j)) - y(j))\x(j),\\
         \g^* = \frac{1}{N}\sum_{j=1}^N (\sigma(\w^* \cdot \x(j)) - y(j))\x(j).
     \end{gather*}
     Moreover, we denote the ``noise-free'' empirical gradient by $\Bar{\g}^{(t)}$, i.e.,
     \begin{equation*}
         \Bar{\g}^{(t)} = \g^{(t)} - \g^* = \frac{1}{N}\sum_{j=1}^N (\sigma(\w^{(t)} \cdot \x(j)) - \sigma(\w^* \cdot \x(j)))\x(j) .
     \end{equation*}
     Plugging in the iteration scheme $\w^{(t+1)} = \w^{(t)} - \eta \g^{(t)}$ while expanding the squared norm, we get
\begin{align*}
        \norm{\w^{(t+1)}-\wstar}^2_2 &= \norm{\w^{(t)}-\wstar}^2_2 - 2\eta\g^{(t)}\cdot(\w^{(t)}-\wstar) + \eta^2 \norm{\g^{(t)}}_2^2\\
        & \leq \underbrace{\norm{\w^{(t)}-\wstar}^2_2 - 2\eta\nabla \Lsur(\w\tth)\cdot(\w^{(t)}-\wstar)}_{Q_1} \\&\quad \underbrace{-2\eta(\g^{(t)}-\nabla \Lsur(\w\tth))\cdot(\w^{(t)}-\wstar)+\eta^2\norm{\g^{(t)}}^2_2}_{Q_2}\;.
    \end{align*}
Observe that we decomposed the right-hand side into two parts, the true contribution of the gradient $(Q_1)$ and the estimation error $(Q_2)$.

Note that in order to utilize the sharpness property of surrogate loss at the point $\w^{(t)}$, the conditions
    \begin{gather}
        \w^{(t)}\in\ball(2||\w^*||_2)\,\text{ and }\nonumber\\
        \,\w^{(t)}\in\{\w:\|\w^{(t)} - \w^*\|_2^2\geq 20B/(\Bar{\mu}^2\rho)\opt\}\label{app:eq:condition-w^t-in-ball(2||w*||_2)}
    \end{gather}
    need to be satisfied. For the first condition, recall that we initialized $\vec w^{(0)}=\vec 0$, hence \Cref{app:eq:condition-w^t-in-ball(2||w*||_2)} is valid for $t=0$. By induction rule, it suffices to show that assuming $\w^{(t)}\in\ball(2||\w^*||_2)$ holds, we have $||\w^{(t+1)} - \w^*||_2\leq (1-C)||\w^{(t)} - \w^*||_2$ for some constant $0<C<1$. Thus, we assume temporarily \Cref{app:eq:condition-w^t-in-ball(2||w*||_2)} is true at iteration $t$, and we will show in the remainder of the proof that $||\w^{(t+1)} - \w^*||_2\leq  (1-C)||\w^{(t)} - \w^*||_2$ until we arrived at some final iteration $T$. Then by induction, the first part of \Cref{app:eq:condition-w^t-in-ball(2||w*||_2)} is satisfied at each step $t\leq T$. For the second condition, note that if it is violated at some iteration $T$, then $\|\w^{(T)} - \w^*\|_2\leq O(\opt)$ implying that this would be the solution we are looking for and the algorithm could be terminated at $T$. Therefore, whenever $\|\w^{(t)}-\w^*\|_2$ is far away from $\opt$, the prerequisites of \Cref{app:prop:sharpness-to-true-loss} are satisfied and the sharpness property of $\Lsur$ is allowed to use.

    Now for the first term $(Q_1)$, using the fact that $\Lsur(\w^{(t)})$ is $\mu(\gamma,\lambda,\beta,\rho,B)$-sharp according to \Cref{app:cor:true-sharpness},
we immediately get a sufficient decrease at each iteration: $||\w^{(t+1)} - \w^*||_2^2\leq (1 - C)||\w^{(t)} - \w^*||_2^2$. Namely, denote $\mu(\gamma,\lambda,\beta,\rho,B)$ as $\mu$ for simplicity, applying \Cref{app:cor:true-sharpness} we have
    \begin{align*}
        (Q_1) &= \norm{\w^{(t)}-\wstar}^2_2 - 2\eta\nabla\Lsur(\w^{(t)})\cdot(\w^{(t)}-\wstar) \leq (1 - 2\eta \mu)\norm{\w^{(t)} - \w^*}_2^2\;,
    \end{align*}
    where $\mu = 1/2\Bar{\mu}$, and $\Bar{\mu} = C\lambda^2\gamma\beta\rho/B$ for some sufficiently small constant $C$.

    Now it suffices to show that $(Q_2)$ can be bounded above by $C'||\w^{(t)} - \w^*||_2^2$, where $C'$ is a parameter depending on $\eta$ and $\mu$ that can be made comparatively small. Formally, we show the following claim. 
    \begin{claim}\label{app:claim:bound-Q2}
      Suppose $\eta\leq 1$. Fix $r_\epsilon\geq 1$ such that $H_4(r_\epsilon)$ is a sufficiently small multiple of $\eps$. 
Choosing $N$ to be a sufficiently large constant multiple of $\frac{d}{\delta}(r_\eps^2 + \alpha^2 M^2)$, then we have with probability at least $1-\delta$
        \begin{equation*}
            (Q_2) \leq \lp(\frac{3}{2}\eta \mu + \frac{8\eta^2\alpha^2B^2}{\rho^2}\rp)\norm{\w^{(t)} - \wstar}_2^2 + \frac{4\eta}{\mu}\bigg(\frac{2B}{\rho}\opt + \epsilon\bigg)\;.
        \end{equation*}
    \end{claim}

     \begin{proof}[Proof of \Cref{app:claim:bound-Q2}]
        Observe that by applying the Arithmetic-Geometric Mean inequality and Cauchy-Schwarz inequality, we get $\bx\cdot\by\leq (a/2)\norm{\bx}_2^2 + (1/2a)\norm{\by}_2^2$ for any vector $\bx$ and $\by$, 
thus applying this inequality to the inner product $(\g^{(t)}-\nabla \Lsur(\w\tth))\cdot(\w^{(t)}-\wstar)$ with coefficient $a = \mu$, we get
\begin{align}
            (Q_2)&=-2\eta(\g^{(t)}-\nabla \Lsur(\w\tth))\cdot(\w^{(t)}-\wstar)+2\eta^2\norm{{\g}^{(t)}}^2_2\nonumber
            \\&\leq -2\eta(\g^{(t)}-\nabla \Lsur(\w\tth))\cdot(\w^{(t)}-\wstar)+2\eta^2\norm{\bar{\g}^{(t)}}^2_2 + 2\eta^2\norm{\g^*}^2_2\nonumber\\
            &\leq \frac{\eta}{\mu}\norm{\g^{(t)}-\nabla \Lsur(\w\tth)}_2^2 + \eta\mu\norm{\w^{(t)} - \wstar}_2^2+2\eta^2\norm{\bar{\g}^{(t)}}^2_2 + 2\eta^2\norm{\g^*}^2_2 \;,\nonumber  
        \end{align}
        where $\mu$ is the sharpness parameter and we used the definition that $\Bar{\g}\tth = \g\tth - \g^*$ in the first inequality. Note that 
        \begin{align*}
            \norm{\g^{(t)}-\nabla \Lsur(\w\tth)}_2^2& = \norm{\g^{(t)}-\g^* - (\nabla \Lsur(\w\tth) - \nabla \Lsur(\w^*)) + \g^* - \nabla \Lsur(\w^*)}_2^2\\
            &\leq 2\norm{\bar{\g}^{(t)} - \nabla\bLsur(\w^{(t)})}_2^2+2\norm{\g^{*} - \nabla\Lsur(\wstar)}_2^2,
        \end{align*}
        since we have $\bLsur(\w\tth) = \Lsur(\w\tth) - \Lsur(\w^*)$. Thus, it holds
        \begin{equation}\label{app:eq:main-thm-claim-I2}
            \begin{split}
                (Q_2)\leq& \frac{2\eta}{\mu}\norm{\bar{\g}^{(t)} - \nabla\bLsur(\w^{(t)})}_2^2 + \frac{2\eta}{\mu}\norm{\g^{*} - \nabla\Lsur(\wstar)}_2^2 + \eta\mu\norm{\w^{(t)} - \wstar}_2^2\\
                &+2\eta^2\norm{\bar{\g}^{(t)}}^2_2 + 2\eta^2\norm{\g^*}^2_2
            \end{split}
        \end{equation}

        Furthermore, recall that as shown in \Cref{app:lem:empirical-grad-approx-population} and \Cref{app:cor:bound-norm-g*-gt}, $\|\bar{\g}^{(t)} - \nabla\bLsur(\w^{(t)})\|_2^2$, $\|\nabla\bLsur(\w^{(t)})\|_2^2$, $||\g^*||_2^2$ and $\|\bar{\g}^{(t)}\|^2_2$
        can be made small by increasing the batch size $N$. In particular, when $r_\eps$ satisfies $H_4(r_\eps)\lesssim \epsilon$, we have proved that with probability at least $1 - \delta$, it holds 
        $$
            \norm{\bar{\g}^{(t)} - \nabla\bLsur(\w^{(t)})}_2^2\lesssim \frac{\alpha^2 d B}{\delta \rho N}\norm{\w^{(t)} - \w^*}_2^2, \norm{\bar{\g}^{(t)}}_2^2 \lesssim \frac{\alpha^2B^2}{\rho^2} \lp(1 + \sqrt{\frac{d \rho}{\delta B N}}\rp)^2\norm{\w^{(t)} - \wstar}_2^2,$$ 
            $$\norm{\g^*-\nabla\Lsur(\wstar)}_2^2\lesssim \frac{d}{\delta N}(r_\eps^2\opt +  \alpha^2 M^2\epsilon),$$
        and
        \begin{equation*}
            \norm{\g^*}_2^2\lesssim \lp(\sqrt{(B/\rho)\opt} + \sqrt{\frac{d (r_\eps^2\opt + \alpha^2 M^2 \epsilon)}{\delta N}}\rp)^2\lesssim  \frac{B}{\rho}\lp(1 + \frac{r_\eps^2 d}{\delta N}\rp)\opt + \frac{\alpha^2 M^2 d}{\delta N}\epsilon\;.
        \end{equation*}
        
        Therefore, choosing 
        \begin{equation}\label{app:eq:lower-bound-on-batch-size-N}
            N\geq C \max\bigg\{\frac{d r_\eps^2}{\delta}, \frac{\alpha^2 M^2 d}{\delta}, \frac{ B\alpha^2 d}{\rho\mu^2 \delta}\bigg\}, \end{equation}
        where $C$ is a sufficiently large absolute constant, then with probability at least $1-\delta$, it holds
        \begin{equation*}
            \norm{\bar{\g}^{(t)} - \nabla\bLsur(\w^{(t)})}_2^2\leq \frac{\mu^2}{4}\norm{\w^{(t)} - \w^*}_2^2, \;\norm{\bar{\g}^{(t)}}_2^2 \leq \frac{4\alpha^2B^2}{\rho^2} \norm{\w^{(t)} - \wstar}_2^2,
        \end{equation*}
        and 
        \begin{equation*}
        \norm{\g^*-\nabla\Lsur(\wstar)}_2^2\leq\opt +\epsilon\;,\quad    \norm{\g^*}_2^2\leq \frac{2B}{\rho}\opt + \eps.
        \end{equation*}
        Plugging these bounds back to \Cref{app:eq:main-thm-claim-I2}, we get
        \begin{equation*}
        \begin{split}
            (Q_2) &\leq \lp(\frac{3}{2}\eta \mu + 8\frac{\eta^2\alpha^2B^2}{\rho^2}\rp)\norm{\w^{(t)} - \wstar}_2^2 + 2\eta\bigg(\eta + \frac{1}{\mu}\bigg)\bigg(\frac{2B}{\rho}\opt + \epsilon\bigg)\\
            &\leq \lp(\frac{3}{2}\eta \mu + 8\frac{\eta^2\alpha^2B^2}{\rho^2}\rp)\norm{\w^{(t)} - \wstar}_2^2 + \frac{4\eta}{\mu}\bigg(\frac{2B}{\rho}\opt + \epsilon\bigg)\;,
        \end{split}
        \end{equation*}
        where in the last inequality we used the assumption that $\eta\leq 1$ and $\mu\leq 1$. The proof is now complete.
    \end{proof}

    Now combining the upper bounds on $(Q_1)$ and $(Q_2)$ and choosing $\eta = \frac{\mu\rho^2}{32 \alpha^2B^2}$, we have:
    \begin{align}
        \norm{\w^{(t+1)} - \wstar}_2^2 &\leq \lp(1 - \frac{1}{2}\eta \mu + \frac{8\eta^2\alpha^2B^2}{\rho^2} \rp)\norm{\w^{(t)} - \w^*}_2^2 + \frac{4\eta}{\mu}\bigg(\frac{2B}{\rho}\opt + \epsilon\bigg)\nonumber\\
        & \leq \lp(1 - \frac{\mu^2\rho^2}{128\alpha^2B^2}\rp)\norm{\w^{(t)} - \w^*}_2^2 + \frac{\rho}{4\alpha^2 B}\bigg(\opt + \eps\bigg)
        \;. \label{app:eq:final-eq}
    \end{align}
    
    When $\|\w^{(t)} - \w^*\|_2^2\geq (64B/(\rho\mu^2))(\opt + \epsilon)$,
in other words when $\w^{(t)}$ is still away from the minimizer $\w^*$, it further holds with probability $1-\delta$:
   \begin{equation}\label{app:eq:||w^(t+1)-w*|| <= (128/mu^2 + 1/alpha^2)opt-when-t<T}
       \norm{\w^{(t+1)}-\wstar}^2_2 \leq \lp(1 - \frac{\mu^2\rho^2}{256\alpha^2B^2}\rp)\norm{\w^{(t)} - \wstar}_2^2,
   \end{equation}
 which proves the sufficient decrease of $||\w^{(t)} - \w^*||_2^2$ that we proposed at the beginning. 

Let $T$ be the first iteration such that $\w^{(T)}$ satisfies $\|\w^{(T)} - \w^*\|_2^2\leq (64B/(\rho\mu^2))(\opt + \epsilon)$. Recall that we need \Cref{app:eq:condition-w^t-in-ball(2||w*||_2)} for every $t\leq T$ to be satisfied to implement sharpness. The first condition is satisfied naturally for $\|\w^{(t+1)}-\wstar\|^2_2\leq \norm{\wstar}_2^2$ as a consequence of \Cref{app:eq:||w^(t+1)-w*|| <= (128/mu^2 + 1/alpha^2)opt-when-t<T} (recall that $\w^{(0)} = 0$). For the second condition, when $t+1\leq T$, since $\mu = 1/2\Bar{\mu}$, we have
\begin{equation*}
    \norm{\w^{(t+1)} - \w^*}_2^2\geq \frac{64B}{\rho\mu^2}(\opt + \epsilon)\geq \frac{20B}{\rho\Bar{\mu}^2}\opt,
\end{equation*}
hence the second condition is also satisfied.

When $t\leq T$, the contraction of $\|\w^{(t)} - \w^*\|_2^2$ indicates a linear convergence rate of stochastic gradient descent. Since $\w^{(0)}=0$, $\|\w^*\|_2\leq W$, it holds $\|\w^{(t)} - \w^*\|_2^2\leq (1 - \mu^2\rho^2/(256\alpha^2B^2))^t\|\w^{(0)} - \w^*\|_2^2\leq \exp(-t\mu^2\rho^2/(256\alpha^2B^2)) W^2$. Thus, to generate a point $\w^{(T)}$ such that $\|\w^{(T)} - \w^*\|_2^2\leq (64B/(\rho\mu^2))(\opt + \epsilon)$, it suffices to run \Cref{app:alg:gd-3} for
   \begin{equation}\label{app:eq:T-expression}
T= \wt{\Theta}\bigg(\frac{B^2 \alpha^2}{\rho^2 \mu^2}\log\lp(\frac{W}{\epsilon}\rp)\bigg)
\end{equation}
   iterations, where the logarithmic dependence on parameters $\alpha$, $B$, $\rho$ and $\mu$ are hidden in the $\wt{\Theta}(\cdot)$ notation. Further, recall that at each step $t$ the contraction $\|\w^{(t+1)}-\wstar\|^2_2 \leq (1 - \frac{\mu^2\rho^2}{256\alpha^2B^2})\|\w^{(t)} - \wstar\|_2^2$ holds with probability $1 - \delta$, thus the union bound inequality implies $\|\w^{(T)} - \w^*\|_2^2\leq (64B/(\rho\mu^2))(\opt + \epsilon)$ holds with probability $1 - T\delta$. Let $\delta = 1/(3T)$, we get with probability at least $2/3$, $\|\w^{(T)} - \w^*\|_2^2\leq (64B/(\rho\mu^2))(\opt + \epsilon)$, and thus from \Cref{eq:app:basic-O(opt)},
\begin{align*}
        \Ltwo(\vec w^{(T)})  \leq 2\opt + \frac{640\alpha^2 B^2}{\rho^2\mu^2}(\opt + \eps) = O\bigg(\bigg(\frac{B\alpha}{\rho\mu}\bigg)^2\opt \bigg)+ \eps,
\end{align*}
    and the proof is now complete. \end{proof}

In the final part of this section we apply \Cref{app:thm:l2-fast-rate-thm} to sub-exponential and $k$-heavy tail distributions. Before we dig into the details, some upper bounds on $H_2(r)$ and $H_4(r)$ are needed. We provide the following simple fact.

\begin{fact}\label{app:lem:H_i(r)<r^ih(r)}
    Let $H_2(r)$ and $H_4(r)$ be as in \Cref{app:def:H-2&H-4}. Then, we have the following bounds:
    \begin{align*}
        H_2(r) &\leq r^2 \min\{1, h(r)\} + \int_{r}^\infty 2 s \min\{1, h(s)\}\diff{s},\\
        H_4(r)&\leq r^4\min\{1, h(r)\} + \int_{r}^\infty 4 s^3 \min\{1, h(s)\}\diff{s}.
    \end{align*}
    Moreover, if $\D_\x$ is sub-exponential with $h(r) = \exp(-r/B)$ or $k$-heavy tail with $h(r) = B/r^k$, $k>4+\rho$, $\rho>0$, and $r \geq \max\{1, B^{-4-\rho}\}$ then

\begin{equation*}
        H_2(r)\lesssim   r^2 h(r)\quad \text{and} \quad H_4(r)\lesssim r^4 h(r).
    \end{equation*}
\end{fact}

\begin{proof}
    To prove the fact, we bound the expectation $\E{|\bu\cdot\x|^i \1\{|\bu\cdot\x|\geq r\}}$ for any vector $\bu\in\ball(1)$, where $i=2,4$. To calculate the expectation, observe that when $t<r^i$, it holds   
$$\pr{|\bu\cdot\x|^i\1\{|\bu\cdot\x|\geq r\} \geq t} = \pr{|\bu\cdot\x|\geq r}.$$ 
Thus, we have
\begin{align*}
    \E{|\bu\cdot\x|^i \1\{|\bu\cdot\x|\geq r\}}& = \int_{0}^\infty \pr{|\bu\cdot\x|^i\1\{|\bu\cdot\x|\geq r\} \geq t}\diff{t}\\
    & = \pr{|\bu\cdot\x|\geq r}\int_0^{r^i} 1\diff{t} + \int_{r^i}^\infty \pr{|\bu\cdot\x|^i\1\{|\bu\cdot\x|\geq r\} \geq t}\diff{t}\\
    & = r^i\pr{|\bu\cdot\x|\geq r} + \int_{r}^\infty i\pr{|\bu\cdot\x|^i\1\{|\bu\cdot\x|\geq r\} \geq s^i}s^{i-1}\diff{s}.
\end{align*}
Since (by \Cref{assmpt:concentration}) $\pr{|\bu\cdot\x|\geq r}\leq \min\{1, h(r)\}$, and further note that when $s\geq r$ it holds
\begin{equation*}
    \pr{|\bu\cdot\x|^i\1\{|\bu\cdot\x|\geq r\} \geq s^i} = \pr{|\bu\cdot\x|\1\{|\bu\cdot\x|\geq r\} \geq s} = \pr{|\bu\cdot\x| \geq s}\leq \min\{1, h(s)\},
\end{equation*}
then we get
\begin{equation*}
    \E{|\bu\cdot\x|^i \1\{|\bu\cdot\x|\geq r\}}\leq r^i \min\{1, h(r)\} + \int_{r}^\infty i s^{i-1} \min\{1, h(s)\}\diff{s},  
\end{equation*}
which holds for any $\bu\in\ball(1)$. Therefore, we proved the first part of the claim by taking the maximum over $\bu \in \ball(1)$ on both sides of the inequality.

Now, consider $r\geq \max\{1, B^{-4-\rho}\}$. Then for sub-exponential distributions, as $h(s) = \exp(-\frac{s}{B})\leq 1$ when $s\geq r$, we have:
    \begin{equation*}
        \int_{r}^\infty s h(s)\diff{s}= \int_{r}^\infty s\exp(-s/B)\diff{s} = B(r - B)\exp(-r/B)\leq Br^2h(r),
    \end{equation*}
    and
    \begin{equation*}
        \int_{r}^\infty s^3 h(s)\diff{s}= \int_r^\infty s^3\exp(-s/B)\diff{s} = B^4((r/B)^3 + 3(r/B)^2 +6r/B + 6)\exp(-r/B)\leq 16B^4r^4 h(r),
    \end{equation*}
    where we assumed without loss of generality that $c\leq 1$. Hence $H_2(r)\leq (1 + 2B)r^2h(r)$ and $H_4(r)\leq (1 + 64B^4)r^4h(r)$, proving the desired claim.

    Finally, for $k$-Heavy tail distributions with $k>4+\rho$, $\rho > 0$, $h(r) = B/r^{k}$. Since $h(r)\leq 1$ when $r\geq \max\{1, B^{-4-\rho}\}$, we have:
    \begin{equation*}
        H_2(r)\leq r^2h(r) + \int_{r}^\infty \frac{2B}{s^{k-1}}\diff{s} \leq (1 + 2B)r^2h(r),
    \end{equation*}
    and in addition,
    \begin{equation*}
        H_4(r)\leq r^4h(r) + \int_{r}^\infty \frac{4B}{s^{k-3}}\diff{s} \leq \lp(1 + \frac{4B}{\rho}\rp)r^4h(r).
    \end{equation*}
    The claim is now complete.
\end{proof}

Applying \Cref{app:thm:l2-fast-rate-thm} to sub-exponential distributions yields an $L_2^2$ error of order $O(\opt)+\eps$ with $\Tilde{\Theta}(\log(1/\eps))$ convergence rate, using $\Tilde{\Omega}(\polylog(1/\eps))$ samples. Formally, we have the following corollaries.

\begin{corollary}[Sub-Exponential Distributions]\label{app:cor:sub-expo-results}
Fix $\eps>0$ and $W>0$ and suppose \Cref{assmpt:activation,assmpt:margin} hold. Moreover, assume that \Cref{assmpt:concentration} holds for $h(r)=\exp(-r/B)$ for some $B\geq 1$.
    Let $\opt$ denote the minimum value of the $L_2^2$, i.e.,  $
        \opt = \min_{\w\in\ball(W)} \Ey{(\sigma(\w\cdot\x) - y)^2}$.
   Let $\mu:=\mu(\lambda,\gamma,\beta,B)$ be a sufficiently small multiple of $\lambda^2\gamma\beta$, and let $M = O(\alpha W B\log\big(\frac{\alpha W}{\eps}\big))$.
   Then after
    \begin{equation*}
        T = \wt{\Theta}\left(\frac{B^2\alpha^2}{\mu^2}\log\lp(\frac{W}{ \eps}\rp)\right)
    \end{equation*}
    iterations with batch size 
    $$N = \wt{\Omega}\bigg(\frac{dB^4\alpha^6 W^2}{\mu^2}\polylog(1/\eps)\bigg)\;,$$
    \Cref{app:alg:gd-3} converges to a point $\w^{(T)}$ such that $\Ltwo(\vec w^{(T)}) = O\big(\big(\frac{B\alpha}{\mu}\big)^2\opt \big)+\eps$ with probability at least $2/3$. 
\end{corollary}

\begin{proof}[Proof of \Cref{app:cor:sub-expo-results}]
    Since by assumption it holds $h(r) = \exp(-r/B)\leq B/r^{4+\rho}$ for $\rho = 1$, we can set $\rho = 1$ in \Cref{app:thm:l2-fast-rate-thm}. Thus, a direct application of \Cref{app:thm:l2-fast-rate-thm} with parameter $\rho = 1$ gives the desired $L_2^2$ error and the required number of iterations. It remains to determine the batch size with respect to the sub-exponential distributions. To this aim, note that $N = \Omega(dT(r_\eps^2 + \alpha^2 M^2))$, thus we need to find the truncation bound $M$, which is defined in \Cref{app:lem:y-bounded-by-M}, and calculate $r_\eps$ such that $H_4(r_\eps)\lesssim \eps$.

    Denote $\kappa = \frac{\epsilon}{4\alpha^2 W^2}$. Recall that $M=\alpha W H_2^{-1}(\kappa)$. To determine $H_2^{-1}(\kappa)$, note that $H_2(r)$ is a non-increasing function, therefore it suffices to find a $r_\kappa$ such that $H_2(r_\kappa)\leq \kappa$, then it holds $H_2^{-1}(\kappa)\leq r_\kappa$. For sub-exponential distributions where $h(r) = \exp(-r/B)$, choosing $r_\kappa = B\log(1/\kappa^2)$ satisfies
    $$r_\kappa^2h(r_\kappa) = 4B^2\log^2\bigg(\frac{1}{\kappa}\bigg)\kappa^2\leq 4B^2\kappa,$$
    since $\log^2(1/\kappa)\kappa\leq 1$ as $\kappa = \epsilon/(4\alpha^2 W^2)\leq 1$. Further note that in \Cref{app:lem:H_i(r)<r^ih(r)} we showed $H_2(r)\leq (1+2B)r^2h(r)$, thus $H_2(r_\kappa)\lesssim \kappa$ hence $M = O(\alpha W B\log((\alpha W)/\eps))$.

    For $r_\eps$, by the same idea one can show that for $r_\eps = B\log(1/\eps^2)$, it holds
    $$H_4(r_\epsilon)\leq \lp(1 + 64B^4 \rp)r_\epsilon^4 \exp(-r_\epsilon/B) = \lp(1 + 64B^4\rp)16B^4\log^4(1/\eps)\epsilon^2 \leq \lp(1 + 64B^4\rp)80B^4\epsilon,$$
    where the first inequality is due to \Cref{app:lem:H_i(r)<r^ih(r)} and in the last inequality we used the fact that $\log^4(1/\eps)\eps\leq 5$ when $\eps\leq 1$.

    Therefore, combining the bounds on $M$ and $r_\eps$, we get
    \begin{equation*}
        N = \wt{\Omega}(dT(r_\eps^2 + \alpha^2 M^2)) = \wt{\Omega}\bigg(\frac{dB^4\alpha^6W^2}{\rho^2\mu^2}\log^3\bigg(\frac{1}{\epsilon}\bigg)\bigg).
    \end{equation*}
\end{proof}

Next, we apply \Cref{app:thm:l2-fast-rate-thm} to heavy-tail distributions.

\begin{corollary}[Heavy-Tail Distributions]\label{app:cor:heavy-tail-results}
Fix $\eps>0$ and $W>0$ and suppose \Cref{assmpt:activation,assmpt:margin} hold. Moreover, assume that \Cref{assmpt:concentration} holds for $h(r)=B/r^{k}$ for some $k > 4 + \rho$ where $\rho > 0$ and $B \geq 1$.
    Let $\opt$ denote the minimum value of the $L_2^2$, i.e.,  $
        \opt = \min_{\w\in\ball(W)} \Ey{(\sigma(\w\cdot\x) - y)^2}$.
   Let $\mu:=\mu(\lambda,\gamma,\beta,\rho,B)$ be a sufficiently small multiple of $\lambda^2\gamma\beta\rho/B$, and let $M = \Theta(\alpha W \big(\frac{\alpha W B}{\eps}\big)^{1/(k-2)})$.
   Then after
    \begin{equation*}
        T = \wt{\Theta}\left(\frac{B^2\alpha^2}{\rho^2\mu^2}\log\lp(\frac{W}{\eps}\rp)\right)
    \end{equation*}
    iterations with batch size 
    $$N = \widetilde{\Omega}\bigg(\frac{dB^2\alpha^6W^2}{\rho^2\mu^2}\bigg(\frac{B}{\eps}\bigg)^{\frac{2}{k-4}}\bigg)\;,$$
    \Cref{app:alg:gd-3} converges to a point $\w^{(T)}$ such that $\Ltwo(\vec w^{(T)}) = O\big(\big(\frac{B\alpha}{\rho\mu}\big)^2\opt \big)+\eps$ with probability at least $2/3$.
\end{corollary}

\begin{proof}[Proof of \Cref{app:cor:heavy-tail-results}]
    Applying \Cref{app:thm:l2-fast-rate-thm} directly we get the desired convergence rate and $L_2^2$ loss. Now for batch size, we need to determine the truncation bound $M = \alpha W H_2^{-1}(\eps/(4\alpha^2 W^2))$ (see \Cref{app:lem:y-bounded-by-M}) as well as $r_\eps$ such that $H_4(r_\eps)\lesssim \eps$. 

    First, denote $\kappa = \frac{\epsilon}{4\alpha^2 W^2}$. Let $r_\kappa = (\frac{2B}{\kappa})^{1/(k-2)}$. By \Cref{app:lem:H_i(r)<r^ih(r)}, $H_2(r_\kappa)\lesssim r_\kappa^2h(r_\kappa) = \frac{\kappa}{2}$. Since $H_2(r)$ is non-increasing, we know $H_2^{-1}(\kappa)\lesssim r_\kappa$. Thus, $M = O(\alpha W (\frac{\alpha W B}{\eps})^{1/(k-2)})$. Next, choose $r_\eps = (\frac{B}{\eps})^{1/(k-4)}$, then it holds $H_4(r_\eps)\lesssim r_\eps^4h(r_\eps) = \eps$ satisfying the condition.

    Combining the bounds on $r_\eps$ and $M$, we get the batch size 
    \begin{align*}
 N=\Omega(dT(r_\eps^2 + \alpha^2 M^2))
        &= \wt{\Omega}\bigg(\frac{dB^2\alpha^2}{\rho^2\mu^2}\log\bigg(\frac{W}{\eps}\bigg)\bigg(\bigg(\frac{B}{\eps}\bigg)^{\frac{2}{k-4}} + \alpha^4W^2\bigg(\frac{B}{\eps}\bigg)^{\frac{2}{k-2}}\bigg)\bigg)\\
        &= \widetilde{\Omega}\bigg(\frac{dB^2\alpha^6W^2}{\rho^2\mu^2}\bigg(\frac{B}{\eps}\bigg)^{\frac{2}{k-4}}\bigg).
    \end{align*}
\end{proof}

Thus, \Cref{app:alg:gd-3} yields an $L_2^2$ error of $O(\opt)+\eps$ in $\wt{\Theta}(\log(1/\eps))$ iterations with batch size 
$\wt{\Omega}((1/\eps)^{2/(k-4)})$ when applied on $k$-heavy tail distributions.

\section{Distributions Satisfying Our Assumptions}\label{sec:distri}

In this section, we show that many natural distributions satisfy 
\Cref{assmpt:margin,assmpt:concentration}.

\subsection{Well-Behaved Distributions from \cite{DKTZ22}} \label{ssec:wb}

We first consider the class of distributions defined by \cite{DKTZ22} 
and termed ``well-behaved''. 
This distribution class contains many natural distributions like log-concave and $s$-concave distributions.

\begin{definition}[Well-Behaved Distributions] \label{def:well-behaved}
Let $L, R >0$.
An isotropic (i.e., zero mean and identity covariance) distribution $\D_{\bx}$ on $\R^d$ is called
$(L,R)$-well-behaved if for any projection $(\D_{\bx})_V$ of $\D_{\bx}$ onto a subspace $V$ of
dimension at most two, the corresponding pdf $\phi_V$ on $\R^2$ satisfies the following:
\begin{itemize}
\item   For all $\vec x \in V$ such that $\snorm{\infty}{\vec x} \leq R$ it holds $\phi_V(\vec x)
\geq L$  (anti-anti-concentration). 
\item  For all $\x \in V$ it holds that  $\phi_V(\x) \leq (1/L)(e^{-L \| \x \|_2 })$
(anti-concentration and concentration).
\end{itemize}
\end{definition}

The distribution class that is $(L,B)$-well-behaved satisfies \Cref{assmpt:concentration}. Therefore, we need to show that the distributions in this class satisfy \Cref{assmpt:margin}.
\begin{lemma}
    Let $\D_\x$ be a $(L,B)$-well-behaved distribution. Then, $\D_\x$ satisfies \Cref{assmpt:margin}, for $\gamma=R/2$ and $\lambda= L R^4/16$.
\end{lemma}
\begin{proof}
    Let $\vec u, \vec v\in \R^d$ be any two orthonormal vectors and let $V$ be the subspace spanned by $\vec u,\vec v$. We have that 
    \begin{align*}
        \Exx{ (\vec u\cdot\x)^2\1\{\vec v\cdot\x \geq R/2\}}&\geq   \Exx{ (\vec u\cdot\x)^2\1\{R\geq \vec u\cdot \x\geq R/2,R\geq \vec v\cdot \x\geq R/2\}}
        \\&\geq (R^2/4)\Exx{ \1\{R\geq \vec u\cdot \x\geq R/2,R\geq \vec v\cdot \x\geq R/2\}}
        \\& (R^2/4)\int_{R/2}^R\int_{R/2}^R \phi_V(\x) \d \x \geq (R^2/4)L\int_{R/2}^R\int_{R/2}^R  \d \x = L R^4/16\;. 
    \end{align*}
    Furthermore, similarly, we have that $\Exx{ (\vec v\cdot\x)^2\1\{\vec v\cdot\x \geq R/2\}}\geq LR^4/16$. Therefore, $\D_\x$ satisfies \Cref{assmpt:margin} with $\gamma=R/2$ and $\lambda= LR^4/16$.
\end{proof}

\subsection{Symmetric Product Distributions with Strong Concentration}
\label{ssec:sym-prod}

\subsubsection{$k$-Heavy Tailed Symmetric Distributions, $k\geq 7$}

Here we show that symmetric product distributions with sufficiently large polynomial tails satisfy our assumptions.

\begin{proposition}\label{prop:heavy-tail-satisfy-assumptions}
   Let $\D_\x$ be a $k$-Heavy Tailed symmetric distribution with $k\geq 7$ and i.i.d.\ coordinates, i.e., it satisfies $\pr{|\bu\cdot\x|\geq r}\leq B/r^k$ for some absolute constant $B\geq 1$.  Let $\alpha= \E{\x_i^2}$ and $\beta=\E{\x_i^4}$. Suppose $\beta - \alpha^2\geq c\alpha^2$, where $c>0$ is an absolute constant and let $C$ to be suffciently small absolute multiple of $(k-6)c^2\alpha^4/B$. Then, $\D_\x$ satisfies \Cref{assmpt:margin,assmpt:concentration} with $\gamma=\frac{C}{2}(\frac{C^3}{16B})^{1/k}$ and $\lambda= \frac{C^5}{64}(\frac{C^3}{16B})^{2/k}$.
\end{proposition}

\begin{proof}[Proof of \Cref{prop:heavy-tail-satisfy-assumptions}]
    First, observe that by definition, $\D_\x$ satisfies \Cref{assmpt:concentration} with $h(r) = B/r^k$. We show that $\D_\x$ satisfies \Cref{assmpt:margin} with some absolute constants $\gamma$ and $\lambda$. Let $\vec u, \vec v\in \R^d$ be any two orthonormal vectors. We have that 
    \[
    \Exx{|\vec u\cdot\x|\vec v\cdot \x}=0\;,
    \]
    since the distribution is symmetric.
    Therefore, we have that $ \Exx{|\vec u\cdot\x||\vec v\cdot \x|\1\{\vec v\cdot \x\geq 0\}}=\Exx{|\vec u\cdot\x||\vec v\cdot \x|}/2$.
Let $V=|\vec u\cdot\x||\vec v\cdot \x|$. We show that $\E{V}\gtrsim  c^2\alpha^4(k-6)/B$.
\begin{claim}
    \label{claim:sec5-heavy-tail}
    Let $V=|\vec u\cdot\x||\vec v\cdot \x|$. Assume that $\x$ has i.i.d.\ zero mean coordinates and that $\x$ is $k$-Heavy tailed with parameter $k\geq 7$, $B\geq 1$.   Let $\alpha= \E{\x_i^2}$ and $\beta=\E{\x_i^4}$. If $\beta -\alpha^2\geq c\alpha^2$, where $c>0$ is an absolute constant then, $\E{V}\gtrsim  c^2\alpha^4(k-6)/B$.
\end{claim}
\begin{proof}[Proof of \Cref{claim:sec5-heavy-tail}]
    First, note that we can write $V^2$ as $\sqrt{V}V^{3/2}$, because $V\geq 0$. Therefore, by applying the Cauchy-Schwarz inequality, we have that  
    \begin{align*}
        \E{V^2}^2&\leq \E{V}\E{V^{3}}\;.
        \end{align*}
        Therefore, we have that $\E{V}\geq (\E{V^2})^2/\E{V^{3}}$. We first bound $\E{V^2}$ from below. 
        Let $\alpha= \E{\x_i^2}$ and $\beta=\E{\x_i^4}$.
Observe that
\begin{align*}
\E{V^2}&=\sum_{i_1,i_2,i_3,i_4}\E{\vec u_{i_1} \vec u_{i_2} \vec v_{i_3}\vec v_{i_4} \x_{i_1} \x_{i_2} \x_{i_3} \x_{i_4}}
\\&= \sum_{i_1,i_2,i_1\neq i_2}\E{\vec u_{i_1}^2  \vec v_{i_2}^2 \x_{i_1}^2 \x_{i_2}^2 + 2\vec u_{i_1}\vec v_{i_1}  \vec u_{i_2}\vec v_{i_2} \x_{i_1}^2 \x_{i_2}^2 }
+ \sum_{i}\E{\vec u_{i}^2  \vec v_{i}^2 \x_{i}^4 }
\\&= \alpha^2\sum_{i_1,i_2,i_1\neq i_2}\vec u_{i_1}^2  \vec v_{i_2}^2  + 2\alpha^2\sum_{i_1,i_2,i_1\neq i_2}\vec u_{i_1}\vec v_{i_1}  \vec u_{i_2}\vec v_{i_2} 
+ \beta \sum_{i}\vec u_{i}^2  \vec v_{i}^2 
\\&= \alpha^2\sum_{i_1,i_2,i_1\neq i_2}\vec u_{i_1}^2  \vec v_{i_2}^2  - 2\alpha^2\sum_{i}\vec u_{i}^2\vec v_{i}^2  
+ \beta \sum_{i}\vec u_{i}^2  \vec v_{i}^2 
\\&= \alpha^2\sum_{i}\vec u_{i}^2  (1-\vec v_{i}^2)  - 2\alpha^2\sum_{i}\vec u_{i}^2\vec v_{i}^2  
+ \beta \sum_{i}\vec u_{i}^2  \vec v_{i}^2 
\\&=(\beta-3\alpha^2)\sum_{i}\vec u_{i}^2  \vec v_{i}^2 +\alpha^2
\;,
\end{align*}where we used that $\sum_{i=1}^d \vec v_i \vec u_i=0$ and $\|\bu\|_2 = \|\bv\|_2 = 1$, because $\vec v,\vec u$ are orthonormal. Next, we show that $\sum_{i}\vec u_{i}^2  \vec v_{i}^2$ is less than $1/2$.
\begin{claim}\label{app:claim:supplementary-heavy-tail-satisfy-assumption}
    Let $\vec v,\vec u$ be two orthonormal vectors. Then, $\sum_{i}\vec u_{i}^2  \vec v_{i}^2\leq 1/2$.
\end{claim}
\begin{proof}[Proof of \Cref{app:claim:supplementary-heavy-tail-satisfy-assumption}]
Note that since $\sum_i \bu_i^2 = \sum_i \bv_i^2 = 1$ and $\sum_i \bu_i\bv_i = 0$, it holds
    \begin{gather*}
        1 = \bigg(\sum_i \bu_i^2\bigg)\bigg(\sum_i \bv_i^2\bigg) = \sum_i \bu_i^2\bv_i^2 + \sum_{1\leq i < j\leq d} (\bu_i^2\bv_j^2 + \bu_j^2\bv_i^2)\;,\\
        0 = \bigg(\sum_i \bu_i\bv_i\bigg)^2 = \sum_i \bu_i^2\bv_i^2 + 2\sum_{1\leq i < j\leq d} \bu_i\bv_i\bu_j\bv_j\;.
    \end{gather*}
    Thus, summing the equalities above, we get
    \begin{equation*}
        1 = 2\sum_i \bu_i^2\bv_i^2 + \sum_{1\leq i < j\leq d}(\bu_i\bv_j + \bu_j\bv_i)^2\geq 2 \sum_i \bu_i^2\bv_i^2,
    \end{equation*}
    therefore, we have $\sum_{i}\vec u_{i}^2  \vec v_{i}^2\leq 1/2$.
\end{proof}
Therefore, we have that if $c\geq 2$ then $\beta-3\alpha^2\geq 0$ hence $\E{V^2}\geq \alpha^2$. If $c\leq 2$, then it holds $\E{V^2}\geq (c/2)\alpha^2$. In summary we have $\E{V^2}\geq (c/2)\alpha^2$ for any $c>0$.
 Furthermore, we can bound $\E{V^3}$ from above as the following.
Using Cauchy-Schwarz, we have that $\E{V^3}\leq \max_{\vec u\in \mathcal{B}(1)}\E{(\vec u\cdot\x)^6}$. Recall that $\x$ is a $k$-Heavy Tailed random variable with $k\geq 7$, hence $\E{(\vec u\cdot\x)^6} \leq 1 + 6B/(k-6)$. Therefore, we have that 
\begin{equation}\label{eq:strict-lower-bound-of-EV}
    \E{V} = \E{|\vec u\cdot\x||\vec v\cdot \x|}\geq \frac{c^2\alpha^4}{4 + \frac{24B}{k-6}}.
\end{equation}
This completes the proof of \Cref{claim:sec5-heavy-tail}.
\end{proof}

\begin{lemma}\label{lem:heavy-tail-orthonormal-positive}
    Let $Z=|\vec u\cdot\x||\vec v\cdot \x|\1\{\vec v\cdot \x\geq 0\}$. Assume that, there exists a constant $1>C>0$, so that $\E{Z}\geq C$ and $\E{Z^2}\leq 1/C$. Then it holds
    $$\EE_{\x\sim\D_\x}\bigg[(\vec u\cdot \x)^2\1\bigg\{\vec v\cdot \x\geq \frac{C}{2}\bigg(\frac{C^3}{16B}\bigg)^{1/k}\bigg\}\bigg]\geq \frac{C^5}{64}\bigg(\frac{C^3}{16B}\bigg)^{2/k},$$
    and
    $$\EE_{\x\sim\D_\x}\bigg[(\vec v\cdot \x)^2\1\bigg\{\vec v\cdot \x\geq \frac{C}{2}\bigg(\frac{C^3}{16B}\bigg)^{1/k}\bigg\}\bigg]\geq \frac{C^5}{64}\bigg(\frac{C^3}{16B}\bigg)^{2/k}.$$
\end{lemma}

\begin{proof}[Proof of \Cref{lem:heavy-tail-orthonormal-positive}]
    Using Paley-Zigmund inequality, we have that
    \begin{align*}
        \pr{Z\geq \zeta \E{Z}}&\geq (1-\zeta)^2\frac{\E{Z}^2}{\E{Z^2}}\;.
    \end{align*}
    Therefore, we have that $ \pr{Z\geq C/2}\geq C^3/4$, where $Z=|\vec u\cdot\x||\vec v\cdot \x|\1\{\vec v\cdot \x\geq 0\}$. For simplicity of notation, let's denote $|\vec u\cdot\x|$ and $|\vec v\cdot \x|\1\{\vec v\cdot \x\geq 0\}$ as $a$ and $b$ respectively. Then fixing some $t>0$, it holds:
    \begin{align}
        \frac{C^3}{4}&\leq \pr{ab\geq \frac{C}{2}}\nonumber\\
        & = \pr{ab\geq \frac{C}{2}, a\geq t\sqrt{\frac{C}{2}}, b\leq t\sqrt{\frac{C}{2}}} + \pr{ab\geq \frac{C}{2}, a\leq t\sqrt{\frac{C}{2}}, b\geq t\sqrt{\frac{C}{2}}}\nonumber\\
        &+ \pr{ab\geq \frac{C}{2}, a\geq t\sqrt{\frac{C}{2}}, b\geq t\sqrt{\frac{C}{2}}} + \pr{ab\geq \frac{C}{2}, \frac{1}{t}\sqrt{\frac{C}{2}}\leq a\leq t\sqrt{\frac{C}{2}}, \frac{1}{t}\sqrt{\frac{C}{2}}\leq b\leq t\sqrt{\frac{C}{2}}}.\nonumber
    \end{align}

Note that $\D_\x$ is $k$-Heavy Tailed, thus, when $t\geq \sqrt{\frac{2}{C}}\big(\frac{16B}{C^3}\big)^{1/k}$, it holds 
    $$\pr{a\geq t\sqrt{C/2}} = \pr{|\vec u\cdot\x|\geq t\sqrt{C/2}}\leq \frac{B}{(t\sqrt{C/2})^k}\leq\frac{C^3}{16}.$$
    Similarly, for $b=|\bv\cdot\x|\1\{\bv\cdot\x\geq 0\}\leq |\bv\cdot\x|$ it holds $\pr{b\geq t/\sqrt{C/2}}\leq C^3/16$. Therefore, $\pr{ab\geq C/2}$ can be upper-bounded by
    \begin{equation*}
        \frac{C^3}{4}\leq \pr{ab\geq \frac{C}{2}}\leq \frac{3C^3}{16} + \pr{ab\geq \frac{C}{2}, \frac{1}{t}\sqrt{\frac{C}{2}}\leq a\leq t\sqrt{\frac{C}{2}}, \frac{1}{t}\sqrt{\frac{C}{2}}\leq b\leq t\sqrt{\frac{C}{2}} }.
    \end{equation*}
    Hence, we get
    \begin{equation*}
        \pr{a\geq \frac{1}{t}\sqrt{\frac{C}{2}}, b\geq \frac{1}{t}\sqrt{\frac{C}{2}}}\geq \pr{ab\geq \frac{C}{2}, \frac{1}{t}\sqrt{\frac{C}{2}}\leq a\leq t\sqrt{\frac{C}{2}}, \frac{1}{t}\sqrt{\frac{C}{2}}\leq b\leq t\sqrt{\frac{C}{2}} }\geq \frac{C^3}{16},
    \end{equation*}
    where we choose $t = \sqrt{\frac{2}{C}}\big(\frac{16B}{C^3}\big)^{1/k}$. As a result,
    \begin{equation*}
    \begin{split}
        &\quad \E{(\vec u\cdot \x)^2\1\bigg\{\vec v\cdot \x\geq \frac{C}{2}\bigg(\frac{C^3}{16 B}\bigg)^{1/k}\bigg\}}\\
        &\geq \E{(\vec u\cdot \x)^2\1\bigg\{\vec v\cdot \x\geq \frac{C}{2}\bigg(\frac{C^3}{16 B}\bigg)^{1/k}, |\bu\cdot\x|\geq \frac{C}{2}\bigg(\frac{C^3}{16 B}\bigg)^{1/k}\bigg\}}\\
        &\geq \frac{C^5}{64}\bigg(\frac{C^3}{16 B}\bigg)^{2/k}.
    \end{split}
    \end{equation*}
    Similarly, we also have $\EE_{\x\sim\D_\x}[(\vec v\cdot \x)^2\1\{\vec v\cdot \x\geq \frac{C}{2}(\frac{C^3}{16B})^{1/k}\}]\geq \frac{C^5}{64}(\frac{C^3}{16B})^{2/k}$.
\end{proof}

From \Cref{claim:sec5-heavy-tail} we know that $\EE[Z] = \E{|\vec u\cdot\x||\vec v\cdot \x|\1\{\vec v\cdot \x\geq 0\}} = \frac{1}{2}\E{|\vec u\cdot\x||\vec v\cdot \x|}\gtrsim (k-6)c^2\alpha^4/B$. In addition, using \CS  we have $\EE[Z^2]\leq \max_{\bu\in\ball(1)}\E{(\bu\cdot\x)^4}\leq 5B$, where the last inequality comes from \Cref{claim:bound-H(0)} (note that $\rho = 1$ suffices in when $\D_\x$ is $k$-Heavy Tailed, $k\geq7$). Thus, choosing $C$ to be suffciently small absolute multiple of $(k-6)c^2\alpha^4/B$, it holds $\EE[Z]\geq C$ and $\EE[Z^2]\leq 1/C$. Let $\bv = \w^*/\norm{\w^*}_2$ and let $\bu$ be any vector that is orthonormal to $\bv$. Then by the results of \Cref{lem:heavy-tail-orthonormal-positive}, we know that choosing $\gamma = \frac{C}{2}(\frac{C^3}{16B})^{1/k}$ and $\lambda = \frac{C^5}{64}(\frac{C^3}{16B})^{2/k}$, it holds $\E{\x\x^\top\1\{\w^*\cdot\x\geq \gamma\norm{\w^*}_2\}}\succeq \lambda \vec I$.
\end{proof}

\subsubsection{Discrete Gaussians} \label{ssec:disc-Gaus}

Here we show that our assumptions are satisfied for discrete multivariate Gaussians.

We will use the following standard definition of a discrete Gaussian.

\begin{definition}[Discrete Gaussian]
    We define the discrete standard Gaussian distribution as follows: Fix $\theta\in \R_+$ with $
    \theta>0$. Then, the pmf of the discrete Gaussian distribution is given by
    \begin{align*}
        p(z) = \frac{1}{Z}\exp\left(-\frac{z^2}{2}\right)\1\{z\in \theta\Z\}\;,
        \end{align*}
        
        where $Z$ is a normalization constant. Similarly, we define the high dimensional analogous as follows, we say that a random vector $\vec x\in \R^d$ follows the $d$-dimensinal discrete Gaussian distribution if $\vec x$ is a vector of $d$ independent random variables, each of which follows the discrete Gaussian distribution.
\end{definition}

\begin{corollary}\label{cor:discrete-gaussian-astisfy-assum}
    Let $\theta\in (0,1]$ and let $\D_\x$ be a $d$-dimensional discrete Gaussian distribution with parameter $\theta$. Then, there exists an absolute constant $C>0$, so that $\D_\x$ satisfies \Cref{assmpt:margin,assmpt:concentration} with $\rho = 1$, $B = e^9$, $\gamma = C(C^3/B)^{1/7}$ and $\lambda = C^5(C^3/B)^{2/7}$.
\end{corollary}
\begin{proof}[Proof of \Cref{cor:discrete-gaussian-astisfy-assum}]
    We first show that the discrete Gaussian distribution is subgaussian with an appropriate parameter.
\begin{lemma}\label{lem:discrete-gaussian-sub-gaussian-parameter}
    Let $\theta\in (0,1]$ then the discrete Gaussian distribution is subgaussian with parameter $2\sqrt{2}$.
\end{lemma}
\begin{proof}[Proof of \Cref{lem:discrete-gaussian-sub-gaussian-parameter}]
    By definition, a random variable $X$ is $D$-subgaussian if 
    \[
    \pr{|X|\geq t}\leq \exp(-t^2 /D^2)\;.
    \]
    Let $X\sim \D_\x$, where $\D_\x$ is a discrete Gaussian distribution with parameter $\theta$.
    Fix $t\geq \theta$, Then, we have that
    \begin{align*}
        \pr{|X|\geq t}&\leq \frac{1}{Z}\sum_{z\in \theta \Z, |z|\geq t} \exp\left(-\frac{z^2}{2}\right)
        \leq \frac{2}{Z}\int_{t/\theta-1}^{\infty} \exp\left(-\frac{z^2\theta^2}{2}\right)\d z
        \\&=\frac{2}{\theta Z}\int_{t-\theta}^{\infty} \exp\left(-\frac{z^2}{2}\right)\d z \leq \frac{2}{\theta Z}\int_{t/2}^{\infty}\exp\left(-\frac{x^2}{2}\right)\d x\;,
    \end{align*}
    where for the first inequality, we used the integral mean value theorem and the monotonicity, i.e., $\int_{k}^{k+1}f(t) \d t= f(\xi)$ for $\xi\in(k,k+1)$ and that $f(k+1)\leq f(\xi)\leq f(k)$ because $f$ is a decreasing function. Further note that $2\int_{t/2}^\infty \exp(-x^2/2)\diff{x} = \sqrt{2\pi}\mathrm{erfc}\big(\frac{t}{2\sqrt{2}}\big)\leq \sqrt{2\pi}\exp(-t^2/8)$.
It remains to show that $\theta Z$ is lower-bouned. Note that $\exp(-x^2/2)$ is a decreasing function when $x\geq 0$, therefore for any $z\in\Z_+$ it holds $\exp(-(\theta z)^2/2)\geq \int_{z}^{z+1}\exp(-(\theta x)^2/2)\diff{x}$. Thus, by definition,
\begin{align*}
    \theta Z &= \theta + 2\theta\sum_{z\in\Z_+} \exp((\theta z)^2/2)\\
    &\geq \theta\int_0^1\exp(-(\theta x)^2/2)\diff{x} + 2\theta\sum_{z\in\Z_+} \int_{z}^{z+1}\exp(-(\theta x)^2/2)\diff{x}\\
    &\geq \int_0^\infty\exp(-t^2/2)\diff{t} + \int_1^\infty \exp(-t^2/2)\diff{t} = \sqrt{\frac{\pi}{2}}(2-\mathrm{erf}(1/\sqrt{2}))\geq \sqrt{\frac{\pi}{2}}.
\end{align*}
Thus, combining these results, we get $\pr{|X|\geq t}\leq 4\exp(-t^2/8)$, thus discrete Gaussian is sub-Gaussian with parameter $D=2\sqrt{2}$.

\end{proof}

It remains to show that the discrete Gaussian distribution satisfies the requirements of \Cref{prop:heavy-tail-satisfy-assumptions}. To be specific, we show that it holds
\begin{claim}\label{claim:discrete-gaussian-claim}
    Let $X$ be a discrete Gaussian random variable. Denote $\EE[X^2]$ as $\alpha$ and $\EE[X^4]$ as $\beta$. Then it holds $\alpha\leq 1$ and $\beta\geq 1.25$.
\end{claim}

\begin{proof}[Proof of \Cref{claim:discrete-gaussian-claim}]
By Poisson summation formula, we know that it holds $\sum_{z\in\Z} f(z) = \sum_{z\in\Z} \hat{f}(z)$ where $\hat{f}(t)$ is the fourier transform of $f$, i.e., $\hat{f}(z) = \int_{-\infty}^{+\infty} f(x)e^{-2\pi i x t}\d x$. It is easy to calculate that for $f(z) = \theta^2 z^2\exp\big(-\frac{\theta^2 z^2}{2}\big)$ we have 
$$\hat{f}(t) = \frac{\sqrt{2\pi}}{\theta^3}(\theta^2 - 4\pi^2t^2)\exp\bigg(-\frac{2\pi^2 t^2}{\theta^2}\bigg),$$
and for $g(z) = \exp\big(-\frac{\theta^2 z^2}{2}\big)$, we have
$$\hat{g}(t) = \frac{\sqrt{2\pi}}{\theta}\exp\bigg(-\frac{2\pi^2t^2}{\theta^2}\bigg).$$
Thus, by definition,
\begin{align*}
    \alpha = \EE[X^2] = \frac{\sum_{z\in\Z}f(z)}{\sum_{z\in\Z}g(z)} = \frac{\sum_{z\in\Z}\hat{f}(z)}{\sum_{z\in\Z}\hat{g}(z)} = 1-\frac{\sum_{z\in\Z} \frac{4\pi^2 z^2}{\theta^2}\exp\big(-\frac{2\pi^2 z^2}{\theta^2}\big)}{\sum_{z\in\Z} \exp\big(-\frac{2\pi^2 z^2}{\theta^2}\big)}\leq 1.
\end{align*}

For $\EE[X^4]$, note that $t^4\exp(-t^2/2)$ is an increasing function when $t\in(0,2)$ and is decreasing when $t\in(2,\infty)$. Thus, denote $\Delta z = 1$, then by the property of integral, it holds
    \begin{align*}
        \E{X^4}&=\frac{1}{\theta Z}\sum_{z\in\Z} (\theta z)^4\exp(-(\theta z)^2/2)\theta(\Delta z)\\
        & \geq \frac{2}{\theta Z} \bigg(\int_0^{2-\theta} t^4\exp(-t^2/2)\diff{t} + \int_2^\infty t^4\exp(-t^2/2)\diff{t}\bigg)\\
        &\geq \frac{2}{\theta Z}\bigg(\int_0^1 t^4\exp(-t^2/2)\diff{t} + 2.06\bigg)\geq \frac{4.4}{\theta Z},
    \end{align*}
    where the second inequality is due to the fact that $\int_2^\infty t^4\exp(-t^2/2)\diff{t}\geq 2.06$ and $\theta\in(0,1]$. Further, in the last inequality we used the fact that $\int_0^1 t^4\exp(-t^2/2)\diff{t}\geq 0.14$.

    Now, note that for $\theta Z$, it holds
    \begin{align*}
        \theta Z = \theta + 2\sum_{z\in\Z_+}\exp(-(\theta z)^2/2)\theta(\Delta z)\leq \theta + 2\int_0^\infty \exp\bigg(-\frac{t^2}{2}\bigg)\diff{t}\leq 1 + \sqrt{2\pi}.
    \end{align*}

    Therefore, combining with upper-bound on $\theta Z$, it holds $\EE[X^4]\geq 4.4/(1 + \sqrt{2\pi})\geq 1.25$.
\end{proof}

Now since $\alpha\leq 1$, $\beta\geq 1.25\alpha^2$ and discrete Gaussian is $2\sqrt{2}$-sub-Gaussian as proved in \Cref{lem:discrete-gaussian-sub-gaussian-parameter}, we have $\pr{|\bu\cdot\x|\leq r}\leq \exp(-r^2/8)\leq e^{9}/r^7$. Thus, the conditions in \Cref{prop:heavy-tail-satisfy-assumptions} are satisfied with parameters $B=e^{9}$, $c = 0.25$, $k=7$. Thus, choosing $C$ to be a small multiple of $c^2\alpha^4/B$, then since according to \Cref{claim:sec5-heavy-tail}, $\E{Z}\gtrsim c^2\alpha^4/B\geq C$ and $\E{Z^2}\leq 5B\leq 1/C$, thus, by \Cref{prop:heavy-tail-satisfy-assumptions}, we know that \Cref{assmpt:margin} is satisfied with parameters $\gamma = \frac{C}{2}(\frac{C^3}{16B})^{1/7}$ and $\lambda = \frac{C^5}{64}(\frac{C^3}{16B})^{2/7}$
.

\end{proof}

\subsubsection{Uniform Distribution on $\{-1,0,1\}^d$}

Finally, we show that the uniform distribution on a hyper-grid satisfies our assumptions.

\begin{corollary}\label{cor:uniform-discrete-satify-assum}
     Let $\D_\x$ be a $d$-dimensional uniform distribution over the $\{-1,0,1\}^d$. Then, there exists an absolute constant $C>0$, so that $\D_\x$ satisfies \Cref{assmpt:margin,assmpt:concentration} with $B=1$, $\rho =1$, $\gamma = C(C^3/B)^{1/7}$ and $\lambda = C^5(C^3/B)^{2/7}$.
\end{corollary}
\begin{proof}[Proof of \Cref{cor:uniform-discrete-satify-assum}]
    Note that the distribution is 1-sub-Gaussian. Now since $\beta = \E{\x_i^4} = 2/3$ and $\alpha = \E{\x_i^2} = 2/3$, therefore, $\beta = 1.5\alpha^2$. Thus, the conditions in \Cref{prop:heavy-tail-satisfy-assumptions} are satisfied with parameters $B=1$, $\alpha = 2/3$, $c=0.5$, $k=7$. Now choosing $C$ to be a small multiple of $c^2\alpha^4/B$, then since according to \Cref{claim:sec5-heavy-tail}, $\E{Z}\gtrsim c^2\alpha^4/B\geq C$ and $\E{Z^2}\leq 5B\leq 1/C$, thus, by \Cref{prop:heavy-tail-satisfy-assumptions}, we know that \Cref{assmpt:margin} is satisfied with parameters $\gamma = \frac{C}{2}(\frac{C^3}{16B})^{1/7}$ and $\lambda = \frac{C^5}{64}(\frac{C^3}{16B})^{2/7}$.

\end{proof}

\section{Extension to Certain Non-Monotone Activations}\label{sec:non-monotone-activation-results}

In this section, we extend our algorithmic results to certain cases 
where the activation function is not monotone. 
Specifically, we will consider activations 
like GeLU~\cite{HG16}: $\sigma_{GeLU}(t) = t\Phi(t)$, 
where $\Phi(t)$ is the cdf. of the standard normal random variable $\mathcal{N}(0,1)$ and Swish~\cite{RZL17} defined by $\sigma_{Swish}(t) = \frac{t}{1 + \exp(-t)}.$

\begin{definition}[Non-Monotonic $(\alpha, \beta)$-Unbounded Activations]\label{def:non-monotone-activation}
Let $\sigma: \R \mapsto \R$ be an activation function and let $\alpha, \beta > 0$. We say that $\sigma$ is non-monotonic $(\alpha, \beta)$-unbounded if it satisfies the following assumptions:
\begin{enumerate}
    \item $\sigma(t_2) \geq 0\geq\sigma(t_1)$ for any $t_2\geq 0\geq t_1$;
    \item $\sigma$ is $\alpha$-Lipschitz; and
    \item $\sigma'(t)\geq \beta$ for all $t\in(0,\infty)$.
\end{enumerate}
\end{definition}

As mentioned above, \Cref{def:non-monotone-activation} contains GeLU and Swish.  Indeed, one can show that $\sigma_{GeLU}(t)$ is actually non-monotonic $(1.1, 1/2)$-unbounded, and $\sigma_{Swish}(t)$ is non-monotonic $(1.2, 0.4)$-unbounded. We include the following \Cref{fig:Gelu and Swish} of these activations to provide the readers with a better geometric intuition.

\begin{figure}[h]
\centering
\begin{subfigure}[b]{0.4\textwidth}
\centering
  \includegraphics[width=\textwidth]{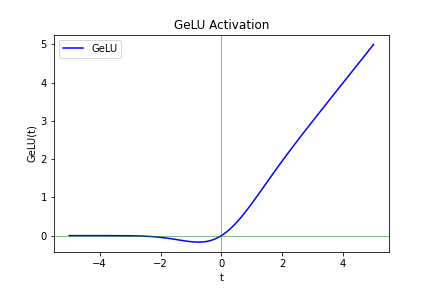}
\end{subfigure}\begin{subfigure}[b]{0.4\textwidth}
\centering
  \includegraphics[width=\textwidth]{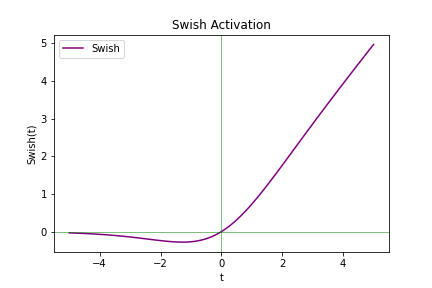}
\end{subfigure}
\caption{Non-Monotonic $(\alpha,\beta)$-Unbounded Activation Examples: GeLU and Swish}
\label{fig:Gelu and Swish}
\end{figure}

Now, we show that truncating a non-monotonic $(\alpha, \beta)$-unbounded activation $\sigma$ to $\hat{\sigma}(t) = [\sigma(t)]_+$ and cutting off the negative part of $y$ induces only a small $L_2^2$ error at point $\w^*$, i.e., $\Ey{(\hat{\sigma}(\wstar\cdot\x) - y\1\{y\geq 0\})^2}\leq \opt$, implying that we can consider running an algorithm similar to \Cref{alg:gd-3} on $\hat{\sigma}(t)$ and truncated $y$.

\begin{lemma}\label{lem:truncate-sigma-and-y}
    Let $\w^* = \argmin_{\w\in\ball(W)} \Ey{(\sigma(\w\cdot\x) - y)^2} = \argmin_{\w\in\ball(W)}\Ltwo(\w)$ and denote $\Ltwo(\w^*)$ as $\opt$. Define $\hat{y} = [y]_+$ and $\hat{\sigma}(t) = [\sigma(t)]_+$. Then:
    \begin{equation*}
        \Ey{(\hat{\sigma}(\wstar\cdot\x) - \hat{y})^2} \leq \opt\;.
    \end{equation*}
\end{lemma}
\begin{proof}
    The proof follows similar ideas in \Cref{app:lem:y-bounded-by-M}. Since $[t]_+$ is a non-expansive projection from $\R$ to $\R^+$, we have $|[t_1]_+ - [t_2]_+|\leq |t_1 - t_2|$ for any $t_1,t_2\in\R$. 
Thus, we get
    \begin{equation*}
        \Ey{(\hat{\sigma}(\wstar\cdot\x) - y')^2} = \Ey{([\sigma(\w^*\cdot\x)]_+ - [y]_+)^2} \leq \Ey{(\sigma(\w^*\cdot\x) - y)^2} = \opt.
    \end{equation*}
\end{proof}

The truncated activation function is a monotonic $(\alpha, \beta)$-unbounded since $\hat{\sigma}(t)$ is increasing when $t\geq 0$ and $\hat{\sigma}(t) = 0$ when $t\leq 0$. Thus, when \Cref{assmpt:margin,assmpt:concentration} hold with respect to distribution $\D_\x$ and $\hat{\sigma}$, we can use a slightly modified algorithm \Cref{alg:gd-4} that works as efficiently as \Cref{alg:gd-3} since \Cref{lem:sharpness-well-behaved} and \Cref{thm:l2-fast-rate-thm-m} can be applied to activation $\hat{\sigma}(t)$ with minor modifications. Formally, we have the following corollaries:

\begin{corollary}\label{cor:sharp-non-mono-activation}
Let $\sigma(t)$ be a non-monotonic $(\alpha, \beta)$-unbounded activation function satisfying \Cref{def:non-monotone-activation}. Suppose that \Cref{assmpt:margin} and \Cref{assmpt:concentration} holds. Further denote $\hat{\sigma}(t) = \sigma(t)\1\{t\geq 0\}$.
    Then the noise-free surrogate loss $\Bar{\mathcal{L}}^{\D,\hat{\sigma}}_{\mathrm{sur}}$ with respect to activation function $\hat{\sigma}(t)$ is $\Omega(\lambda^2\gamma\beta\rho/B)$-sharp in the ball $\ball(2\|\w^*\|_2)$, i.e.,
\begin{equation*}
        \nabla \Bar{\mathcal{L}}^{\D,\hat{\sigma}}_{\mathrm{sur}}(\w)\cdot(\w-\wstar)\gtrsim \lambda^2\gamma\beta\rho/B\norm{\w - \wstar}_2^2,\;\; \forall \w \in \ball(2\|\w^*\|_2).
    \end{equation*}
\end{corollary}
\begin{proof}
    Since $\hat{\sigma}$ is monotonic $(\alpha, \beta)$-unbounded, we have proven in \Cref{lem:sharpness-well-behaved} for monotonic $(\alpha, \beta)$-unbounded activations and distribution $\D_\x$ satisfying \Cref{assmpt:activation} to \Cref{assmpt:concentration}, $\Bar{\mathcal{L}}^{\D,\hat{\sigma}}_{\mathrm{sur}}$ is $\Bar{\mu}$ sharp with the parameter $\Bar{\mu} = \Omega(\lambda^2\gamma\beta\rho/B)$.
\end{proof}

\begin{corollary}\label{cor:l2-fast-rate-thm-for-non-monotone}
    Let $\sigma(t)$ be a non-monotonic $(\alpha,\beta)$-unbounded activation function, satisfying \Cref{def:non-monotone-activation}. Fix $\eps>0$ and $W>0$ and suppose \Cref{assmpt:margin,assmpt:concentration} hold.
    Let $\opt$ denote the minimum value of the $L_2^2$ loss i.e., 
    $$\opt = \min_{\w\in\ball(W)}\Ey{(\sigma(\w\cdot\x) - y)^2}.$$
   Let $\mu:=\mu(\lambda,\gamma,\beta,\rho,B)$ be a sufficiently small multiple of $\lambda^2\gamma\beta\rho/B$, and let $M = \alpha W H_2^{-1}\big(\frac{\eps}{4\alpha^2 W^2}\big)$.
Further, choose parameter $r_\eps$ large enough so that $H_4(r_\eps)$ is a sufficiently small multiple of $\eps$.
Then after
    \begin{equation*}
        T = \wt{\Theta}\left(\frac{B^2\alpha^2}{\rho^2\mu^2}\log\lp(\frac{W}{ \eps}\rp)\right)
\end{equation*}
    iterations with batch size $N = \wt{\Omega}(dT(r_\eps^2 + \alpha^2M^2)),$
    \Cref{alg:gd-4} converges to a point $\w^{(T)}$ such that $\Ltwo(\vec w^{(T)}) = O\lp(\frac{B^2\alpha^2}{\rho^2 \mu^2}\opt \rp)+\eps$ with probability at least $2/3$.
\end{corollary}
\begin{proof}
    First observe that since the $\alpha$-Lipschitz property remains for non-monotonic $(\alpha, \beta)$-unbounded functions, the following is still valid:
    \begin{align}
         \Ltwo(\w) & = \Ey{(\sigma(\w\cdot\x) - y)^2} \nonumber\\
         &\leq 2\E{(\sigma(\w\cdot\x) - \sigma(\w^*\cdot\x))^2} + 2\Ey{(\sigma(\w^*\cdot\x) - y)^2} \nonumber\\
         &\leq 2\alpha^2 \E{((\w - \w^*)\cdot\x)^2} + 2\opt \nonumber\\
         &\leq (10B\alpha^2/\rho) \norm{\w - \w^*}_2^2 + 2\opt. \label{eq:non-monotonic:L_2(w^T) <= 2opt + 20B alpha^2 ||w^T-w*||^2}
     \end{align}

     Now denote $\hat{\eps} = \Ey{(\hat{\sigma}(\w\cdot\x) - \hat{y})^2}$. If one can show that after $T = \wt{\Theta}\left(B^2\alpha^2/(\rho^2\mu^2)\log\lp(W/\eps\rp)\right)$ iterations with a large enough batch size $N = \wt{\Omega}(dT(r_\eps^2 + \alpha^2 M^2))$,
     \Cref{alg:gd-4} generates a point $\w^{(T)}$ such that it holds 
     \begin{equation}\label{eq:non-monotone-||w^T-w*||<=eps}
         \norm{\w^{(T)} - \w^*}_2^2\leq \frac{64B}{\rho\mu^2}\hat{\eps} +\epsilon,
     \end{equation}
     then, since we showed in \Cref{lem:truncate-sigma-and-y} that $\hat{\eps}\leq \opt$, combining with \Cref{eq:non-monotonic:L_2(w^T) <= 2opt + 20B alpha^2 ||w^T-w*||^2} we immediately get
     \begin{equation*}
         \Ltwo(\w)\leq \lp(2 + \frac{640B^2\alpha^2}{\rho^2\mu^2}\rp) \opt + (10B\alpha^2/\rho)\epsilon,
     \end{equation*}
     thus completing the corollary.

     In order to prove the claim above and \Cref{eq:non-monotone-||w^T-w*||<=eps}, one only needs to observe that $\hat{\sigma}(t)$ is monotonic $(\alpha,\beta)$-unobounded, and \Cref{assmpt:activation} to \Cref{assmpt:concentration} holds for $\hat{\sigma}(t)$ and the distribution $\D_\x$. Therefore, the exact same techniques for proving \Cref{thm:l2-fast-rate-thm-m} (see \Cref{app:full-version-sec-4.2}) can be applied. Results similar to \Cref{app:lem:y-bounded-by-M} and \Cref{app:lem:empirical-grad-approx-population} still hold for activation function $\hat{\sigma}(t)$ and data points $(\x,\hat{y})$, with the only difference being that we have $\Ey{(\hat{\sigma}(\w^*\cdot\x) - \hat{y})^2} = \hat{\eps}$ instead of $\opt$. Moreover, it has proven in \Cref{cor:sharp-non-mono-activation} that $\Bar{\mathcal{L}}^{\D,\hat{\sigma}}_{\mathrm{sur}}$ is $\Bar{\mu}$ sharp, therefore by \Cref{cor:true-sharpness-m} we know $\Lsur$ is $\Bar{\mu}/2$ sharp. Thus, \Cref{eq:non-monotone-||w^T-w*||<=eps} follows from the same steps and the same choice of parameters as in the proof of \Cref{thm:l2-fast-rate-thm-m} (see \Cref{app:full-version-sec-4.2}).
     
\end{proof}

\begin{algorithm}[tb]
   \caption{Stochastic Gradient Descent on Surrogate Loss For Non-Monotonic $(\alpha, \beta)$-Unbounded Activations}
   \label{alg:gd-4}
\begin{algorithmic}
   \STATE {\bfseries Input:} Iterations: $T$, sample access from $\D$, batch size $N$, step size $\eta$, bound $M$.
   \STATE Initialize $\vec w^{(0)} \gets \vec 0$.
\FOR{$t=1$ {\bfseries to} $T$}
\STATE Draw $N$ samples $\{(\x(j), y(j))\}_{j=1}^N\sim\D$.
   \STATE for each $j\in[N]$, $y(j)\gets\min([y(j)]_+,M)$ \STATE Let $\hat{\sigma}(t) = [\sigma(t)]_+$, calculate 
   $$\vec g^{(t)} \gets\frac{1}{N}\sum_{j=1}^N (\hat{\sigma}(\w^{(t)}\cdot\x(j)) - y(j))\x(j),$$ 

   \STATE $\vec w^{(t+1)} \gets \vec w^{(t)}-\eta \vec{g}^{(t)}$.
\ENDFOR
   \STATE {\bfseries Output:} The weight vector $\w^{(T)}$.
\end{algorithmic}
\end{algorithm}

\end{document}